\documentclass[times, review, 10pt]{elsarticle}
\usepackage{amssymb}
\usepackage{graphicx}
\usepackage{amsmath}
\usepackage{longtable}
\usepackage{a4wide}
\usepackage{mathrsfs}
\usepackage{subcaption}
\usepackage{mathtools}
\usepackage{pifont}
\usepackage{algorithmic}
\usepackage{algorithm}
\usepackage{xcolor}
\usepackage{color}
\usepackage{lineno}
\usepackage{makecell} 
\usepackage{pdflscape}
\usepackage{adjustbox}
\usepackage[utf8]{inputenc}
\usepackage{tabularx}
\usepackage{setspace}
\usepackage{blindtext}
\usepackage{multirow}
\usepackage{notoccite} 
\usepackage{lscape} 
\usepackage{caption} 
\usepackage{mwe}
\usepackage{tikz}
\usepackage{siunitx}
\usepackage{mathrsfs}
\usetikzlibrary{shapes,arrows}
\usepackage{xcolor}
\usepackage{booktabs}
\usepackage{hyperref}
\usepackage{amsthm}
\newtheorem{theorem}{Theorem}[section]

\theoremstyle{definition}
\newtheorem{definition}{Definition}
\newtheorem{lemma}[theorem]{Lemma}

\newcommand{\RNum}[1]{\lowercase\expandafter{\romannumeral #1\relax}}
\newcommand{\RNumU}[1]{\uppercase\expandafter{\romannumeral #1\relax}}
\usepackage{natbib}
\journal{Elsevier}

\begin{document}
\date{}
\begin{frontmatter}
\title{Multiview Random Vector Functional Link Network for Predicting DNA-Binding Proteins}
\author[inst1]{A. Quadir}
\ead{mscphd2207141002@iiti.ac.in}
\author[inst1]{M. Sajid}
\ead{phd2101241003@iiti.ac.in}
\author[inst1]{M. Tanveer\corref{Correspondingauthor}}
\ead{mtanveer@iiti.ac.in}

\affiliation[inst1]{organization={Department of Mathematics, Indian Institute of Technology Indore},
            addressline={Simrol}, 
            city={Indore},
            postcode={453552}, 
            state={Madhya Pradesh},
            country={India}}
            \cortext[Correspondingauthor]{Corresponding author}

\begin{abstract}
The identification of DNA-binding proteins (DBPs) is essential due to their significant impact on various biological activities. Understanding the mechanisms underlying protein-DNA interactions is essential for elucidating various life activities. In recent years, machine learning-based models have been prominently utilized for DBP prediction. In this paper, to predict DBPs, we propose a novel framework termed a multiview random vector functional link (MvRVFL) network, which fuses neural network architecture with multiview learning. The MvRVFL model integrates both late and early fusion advantages, enabling separate regularization parameters for each view, while utilizing a closed-form solution for efficiently determining unknown parameters. The primal objective function incorporates a coupling term aimed at minimizing a composite of errors stemming from all views. From each of the three protein views of the DBP datasets, we extract five features. These features are then fused together by incorporating a hidden feature during the model training process. The performance of the proposed MvRVFL model on the DBP dataset surpasses that of baseline models, demonstrating its superior effectiveness. We further validate the practicality of the proposed model across diverse benchmark datasets, and both theoretical analysis and empirical results consistently demonstrate its superior generalization performance over baseline models. The source code of the proposed MvRVFL model is available at \url{https://github.com/mtanveer1/MvRVFL}.
\end{abstract}
\end{frontmatter}
\section{Introduction and Motivation}
A DNA-binding protein (DBP) is a protein that has the ability to interact with DNA \cite{jones1987cellular}. The DBP plays a crucial role in numerous vital biological processes, including transcriptional regulation, DNA replication and repair, cellular development, and chromatin organization \cite{ohlendorf1982molecular}. Therefore, precise prediction of DBPs is of considerable significance for proteome annotation as well as endeavors in synthetic biology. To identify DBPs, extensive work has been carried out in wet lab settings, including techniques such as X-Ray crystallography \cite{chou2003crystal}, genetic analysis \cite{freeman1995molecular}, and chromatin immunoprecipitation on microarrays \cite{buck2004chip}. While results obtained from wet-lab methods are likely the most reliable, it's important to note that this approach is also associated with significantly high time and labor costs. In contrast to wet-lab methods, computational approaches can substantially decrease resource requirements and expedite the identification of DBPs. Furthermore, in the postgenomic era, there's an escalating demand for the development of efficient and swift computational methods for accurately identifying DBPs, thereby underscoring their increasing significance within the bioinformatics domain.

Computational techniques are widely used for identifying DBPs \cite{qu2019review}. These techniques are generally divided into two primary types: structure-based methods and sequence-based methods. Structure-based methods utilize information derived from the protein's three-dimensional structure. Initially, structure-based methods dominated the field of DBP prediction, with models such as LBi-DBP \cite{zeng2024lbi}, StackDPP \cite{ahmed2024stackdpp}, DBPboost \cite{sun2024dbpboost}, and ULDNA \cite{zhu2024uldna}. Effectively encoding a protein sequence into a vector format is considered both a fundamental and difficult aspect of sequence-based approaches \cite{zhang2019review}. DNABinder \cite{ranjan2021ensemble} employs evolutionary information encoded in a position-specific scoring matrix (PSSM), generated through the PSI-BLAST multiple sequence alignment tool \cite{sf1997gapped}. A support vector machine (SVM) is utilized for classification, taking these features as input, representing the initial application of this method in identifying DNA-binding proteins (DBPs). In iDNA-Prot \cite{chou2011some}, pseudo amino acid composition (PseAAC) features \cite{chowdhury2017idnaprot}, derived from protein sequences through the grey model, are combined with random forest (RF) for identifying DNA-binding proteins (DBPs). iDNAProt-ES \cite{chowdhury2017idnaprot} utilizes evolutionary data like PSSM alongside structural information predicted by SPIDER2 \cite{yang2017spider2} to characterize the features of a given protein. These features are trained to identify DBPs using SVM with a linear kernel. DBP datasets, derived from various sources or views, offer complementary information essential for improving task performance. Each view contributes unique perspectives and features that, when combined, enhance the overall understanding and efficiency of the model. Nevertheless, existing multiview learning approaches often necessitate complex architectures or extensive preprocessing to manage these diverse views effectively \cite{hu2023improving, quadir2024multiview}. The growing complexity of data from different sources poses significant challenges in extracting meaningful insights and achieving high performance. In recent years, the integration of multiview information has become increasingly popular in the domain of bioinformatics \cite{zou2019fkrr, zhou2022multivariate}. The FKRR-MVSF model \cite{zou2019fkrr}, which utilizes multi-view sequence features and fuzzy kernel ridge regression, is introduced for the identification of DNA-binding proteins. In FKRR-MVSF, the initial step entails deriving multi-view sequence features from protein sequences, followed by the application of a multiple kernel learning (MKL) approach to combine these varied features. In MSFBinder \cite{liu2018model}, a stacking framework is introduced to predict DNA-binding proteins (DBPs) by merging features from various perspectives. The multi-view hypergraph restricted kernel machine (MV-H-RKM) \cite{guan2022mv} model is introduced to derive various features from protein sequences. The multi-view features are linked through a shared latent representation, and multi-hypergraph regularization is employed to integrate them while preserving the structural coherence between the original features and the latent space.

Despite the extensive use of hyperplane-based classifiers like SVMs and their variants in predicting DNA-binding proteins, however, DBP prediction remains a challenging task due to the complex, heterogeneous, and nonlinear nature of protein sequences, which encapsulate diverse evolutionary, structural, and physicochemical properties \cite{glasscock2025computational}. SVM-based models and manual identification are both time-consuming and non-scalable, highlighting the growing need for efficient machine learning–based predictive frameworks \cite{sun2022mlapsvm}.

Among efficient learning paradigms, the Random Vector Functional Link (RVFL) \cite{pao1994learning, malik2023random} network has been widely recognized for its fast training speed and closed-form learning capability \cite{sajid2025wave, sajid2025gb}. Despite these advantages, conventional RVFL architectures fail to fully leverage complementary information from multiple heterogeneous feature representations, as evidenced by our experimental observations \cite{sajid2025rvfl}. This shortcoming arises because standard RVFLs perform shallow transformations of input data, limiting their capacity to capture intricate cross-feature dependencies, particularly when integrating PSSM-based, evolutionary, and physicochemical descriptors.

In the context of DBP prediction, there exists a pressing need for an RVFL-based model that can effectively exploit the complementary nature of multiple biological feature views while preserving computational efficiency. To address this gap, we introduce the Multiview RVFL (MvRVFL) network, an advanced RVFL formulation designed to integrate diverse biological features through a unified multiview learning framework. Unlike deep learning models that require extensive parameter tuning and computational resources, MvRVFL preserves the analytical training property of classical RVFL, yielding a closed-form solution with low memory and runtime overhead.

\noindent The main characteristics of the proposed MvRVFL model are as follows:
\begin{enumerate}
    \item To the best of our knowledge, this is the first time to integrate multiview learning principles into the RVFL network. While conventional RVFL models are known for their simplicity, scalability, and efficiency compared to traditional ANNs and ML methods, they have not been extended to a principled multiview optimization setting. The proposed MvRVFL formulation enables simultaneous exploitation of complementary information through both early and late fusion across distinct biological feature views.
    \item We introduce a novel optimization framework for MvRVFL that incorporates a coupling term into the objective function. This design allows the model to minimize a composite of errors originating from all views while maintaining a balance between shared and view-specific learning. The formulation thus provides a mathematically grounded mechanism for effective multiview unification.
    \item The proposed model unifies early fusion (which captures inter-view correlations at the input level) and late fusion (which aggregates decision-level outputs from each view). This hybrid fusion design enables MvRVFL to capture shared structures while preserving individual view information. The ablation studies empirically confirm that this combined fusion approach consistently outperforms models employing only early or late fusion, thereby validating its effectiveness (See section 6.4).
    \item The proposed MvRVFL framework is specifically designed with the goal of enhancing DBP prediction. By leveraging heterogeneous protein descriptors, such as PSSM-based, evolutionary, and physicochemical features, as distinct views, MvRVFL effectively captures complex discriminative patterns.
\end{enumerate}

\noindent The paper's main highlights are as follows:
\begin{itemize}
    \item We propose a novel MvRVFL network that leverages multiple distinct feature sets using the MVL and the simplicity cum effectiveness of RVFL to predict DBP.
    \item The mathematical foundations of the MvRVFL model are established by deriving its consistency and generalization error bounds using Rademacher complexity analysis.
    \item Following the method outlined in \cite{guan2022mv}, we partition each protein sequence into five segments of equal size, resulting in 14 distinct continuous and discontinuous regions. Each region is characterized using 63-dimensional features (7 + 21 + 35), leading to a comprehensive 882-dimensional (63 x 14) feature representation of the entire protein sequence. We extract five types of features from the sequences: NMBAC, PSSM-DWT, MCD, PSSM-AB, and PsePSSM. Experiments showcased that our method outperforms traditional models in effectively identifying DNA-binding proteins.
    \item Moreover, to assess the generalization capability of the proposed model, we perform evaluations using the UCI, KEEL, Corel5k, and AwA datasets, which span various domains. The results from the experiments show that the MvRVFL model surpasses various baseline models in performance.
\end{itemize}

The remainder of the paper is organized as follows: \ref{Background} offers a literature review on the baseline model, and \ref{Random Vector Functional Link (RVFL) Network} provides an overview of related work. Section \ref{Notations} introduces the notations and mathematical symbols that are consistently used throughout the paper. Section \ref{methods} covers the method for extracting features from DNA-binding protein sequences. We present the detailed formulation of the proposed MvRVFL model in Section \ref{MvRVFL}. Section \ref{Generalization Capability Analysis} explores the ability of the proposed MvRVFL model to generalize effectively. In Section \ref{Experiments and Results}, we detail the experimental findings and offer a comparative assessment between the proposed approach and established baseline models. Finally, Section \ref{Conclusion and Future Work} presents the conclusions and potential future research directions.

\section{Notations}
\label{Notations}
Let $\mathscr{T} = \mathscr{T}^A \times \mathscr{T}^B$ be the sample space defined by the Cartesian product of two separate feature views, $A$ and $B$, where $\mathscr{T}^A \subseteq \mathbb{R}^{m_1}$, $\mathscr{T}^B \subseteq \mathbb{R}^{m_2}$ and $\mathscr{Y} = \{-1, +1\}$ denotes the label space. Here $n$ represents the number of samples, $m_1$ and $m_2$ denote the number of features corresponding to view $A$ (Vw-$A$) and view $B$ (Vw-$B$), respectively. Suppose $\mathscr{H} = \{(x_i^A, x_i^B, y_i) | x_i^A \in \mathscr{T}^A, x_i^B \in \mathscr{T}^B, y_i \in \mathscr{Y}\}_{i=1}^n$ represent a two-view dataset. Let $X_1 \subseteq \mathbb{R}^{n \times m_1}$ and $X_2 \subseteq \mathbb{R}^{n \times m_2}$ be the input matrix of Vw-$A$ and Vw-$B$, respectively, and $Y$ denotes the label vector. $H_1 \subseteq \mathbb{R}^{n \times h_l}$ and $H_2 \subseteq \mathbb{R}^{n \times h_l}$ are hidden layer matrices, which are obtained by applying a nonlinear activation function, denoted as $\phi$, to the input matrices $X_1$ and $X_2$ after transforming them with randomly initialized weights and biases. Here, \( h_l \) denotes the number of nodes in the hidden layer, and \( (\cdot)^t \) indicates the transpose operation.

\section{Data Preprocessing and Feature Extractions}
\label{methods}
In this part, we first describe the DBP sequence from three different viewpoints. Next, five distinct algorithms are employed to derive features from these various viewpoints \cite{guan2022mv}. Finally, the MvRVFL model is leveraged to integrate the extracted features and construct the predictor for detecting DNA-binding proteins (DBPs). Fig \ref{Flowchart of the model construction} illustrates the flowchart of the model construction process.

\begin{figure}
    \centering
    \includegraphics[width=0.9\textwidth,height=7cm]{DNA.png}
    \caption{Flowchart of extracting features of Protein Sequences.}
    \label{Flowchart of the model construction}
\end{figure}

\subsection{Features of DBP}
In this subsection, we detail the DBP sequence from three different perspectives: physicochemical properties, evolutionary information, and amino acid composition. These views are transformed into feature matrices using various extraction algorithms: Multi-scale Continuous and Discontinuous (MCD) \cite{you2014prediction} for amino acid composition; Normalized Moreau-Broto Autocorrelation (NMBAC) \cite{ding2016predicting} for physicochemical properties; and PSSM-based Discrete Wavelet Transform (PSSM-DWT) \cite{nanni2012wavelet}, Pseudo Position-Specific Scoring Matrix (PsePSSM) \cite{liu2015pse}, and PSSM-based Average Blocks (PSSM-AB) \cite{cheol2010position}, for evolutionary information.

\subsection{Multi-Scale Continuous and Discontinuous for Protein Sequences}
The Multi-scale Continuous and Discontinuous (MCD) approach employs multi-scale decomposition to partition a protein sequence into equal-length segments, subsequently analyzing both continuous and discontinuous regions within each segment. Three descriptors are used for each region: distribution (D), composition (C), and transition (T). The following outlines the detailed process of feature extraction for each descriptor.

For distribution (D), the protein sequence is categorized into 7 different types of amino acids. The total number of occurrences of amino acids in each category is denoted by $M_i$, where $i$ belongs to the set $\{1, 2, \ldots, 7\}$. Subsequently, within each amino acid category, we determine the positions of the $1^{st}$, $25^{th}$ percentile, $50^{th}$ percentile, $75^{th}$ percentile, and last amino acids of that category in the entire protein sequence. These positions are then normalized by dividing by the total sequence length \( L \). Therefore, each amino acid category can be represented by a $5$-dimensional feature vector. For the entire protein sequence, this results in a $35$-dimensional feature vector. 

In the Composition (C) approach, the 20 standard amino acids are classified into seven groups according to their dipole moments and side-chain volumes, with each group exhibiting comparable characteristics. Table \ref{Amino Acid Categories} outlines the specific amino acids assigned to each respective group in detail. Let \( S = \{s_1, s_2, \ldots, s_L\} \) represent a protein sequence, where each \( s_i \) corresponds to the \( i^{\text{th}} \) amino acid residue, and \( L \) indicates the overall sequence length. Based on Table \ref{Amino Acid Categories}, we can represent \( S \) as \( S = \{s_1, s_2, \ldots, s_L\} \), where each residue \( s_i \) is classified into one of the categories \( \{1, 2, 3, \ldots, 7\} \). We calculate the proportion of amino acid residues in each category throughout the entire length of the protein sequence. In the case of composition (C), it is expressed as a 7-dimensional feature vector, with each component representing one of the seven predefined amino acid groups.

\begin{table*}[htp]
\centering
    \caption{Amino Acid Categories Based on Side Chain Dipoles and Volumes}
    \label{Amino Acid Categories}
    \resizebox{1\linewidth}{!}{
\begin{tabular}{ccccccc} 
\hline
Group 1 & Group 2 & Group 3 & Group 4 & Group 5 & Group 6 & Group 7 \\ \hline
A, G, V & D, E & F, P, I, L & H, Q, N, W & K, R & T, Y, M, S & C \\ \hline
\end{tabular}}
\end{table*}

For transition (T), the protein sequence is segmented into seven groups of amino acids according to the previously defined classification. The transition frequency between various amino acid categories can be computed to represent the protein’s transition (T) characteristics. This transition information is captured using a 21-dimensional vector, which encompasses all possible pairwise transitions among the seven amino acid groups.

In this study, each protein sequence is partitioned into $5$ uniform-length segments, creating a total of 14 continuous or discontinuous regions. Each region is characterized by the $63$-dimensional features mentioned earlier ($35$ for distribution, $21$ for transition, and $7$ for composition). Consequently, the entire protein sequence is described using an $882$-dimensional feature vector ($63 \times 14$).
\begin{table*}[ht!]
\centering
    \caption{Six physicochemical characteristics linked to each amino acid.}
    \label{Six Physicochemical}
    \resizebox{0.7\linewidth}{!}{
\begin{tabular}{lcccccc}
\hline Amino acid & Q1 & Q2 & SASA & H & NCISC & VSC \\
\hline 
D & 0.105 & 13 & 1.587 & -0.9 & -0.02382 & 40 \\
C & 0.128 & 5.5 & 1.461 & 0.29 & -0.03661 & 44.6 \\
A & 0.046 & 8.1 & 1.181 & 0.62 & 0.007187 & 27.5 \\
R & 0.291 & 10.5 & 2.56 & -2.53 & 0.043587 & 105 \\
G & 0 & 9 & 0.881 & 0.48 & 0.179052 & 0 \\
H & 0.23 & 10.4 & 2.025 & -0.4 & -0.01069 & 79 \\
P & 0.131 & 8 & 1.468 & 0.12 & 0.239531 & 41.9 \\
E & 0.151 & 12.3 & 1.862 & -0.74 & 0.006802 & 62 \\
I & 0.186 & 5.2 & 1.81 & 1.38 & 0.021631 & 93.5 \\
N & 0.134 & 11.6 & 1.655 & -0.78 & 0.005392 & 58.7 \\
Q & 0.18 & 10.5 & 1.932 & -0.85 & 0.049211 & 80.7 \\
F & 0.29 & 5.2 & 2.228 & 1.19 & 0.037552 & 115.5 \\
L & 0.186 & 4.9 & 1.931 & 1.06 & 0.051672 & 93.5 \\
T & 0.108 & 8.6 & 1.525 & -0.05 & 0.003352 & 51.3 \\
Y & 0.298 & 6.2 & 2.368 & 0.26 & 0.023599 & 117.3 \\
M & 0.221 & 5.7 & 2.034 & 0.64 & 0.002683 & 94.1 \\
W & 0.409 & 5.4 & 2.663 & 0.81 & 0.037977 & 145.5 \\
S & 0.062 & 9.2 & 1.298 & -0.18 & 0.004627 & 29.3 \\
K & 0.219 & 11.3 & 2.258 & -1.5 & 0.017708 & 100 \\
V & 0.14 & 5.9 & 1.645 & 1.08 & 0.057004 & 71.5 \\
\hline
\end{tabular}}
\end{table*}
\subsection{Normalized Moreau-Broto Autocorrelation for Protein Sequences}
The Normalized Moreau-Broto Autocorrelation (NMBAC) is a descriptor derived from the physicochemical attributes of amino acids. Following the methodology outlined in \cite{ding2016predicting}, we extract relevant features from protein sequences. Specifically, this work considers six physicochemical characteristics: hydrophobicity (O), solvent-accessible surface area (SASA), polarizability (P1), polarity (P2), side chain volume (VSC), and the net charge index of side chains (NCISC). Table \ref{Six Physicochemical} lists the values of these six physicochemical properties for each amino acid. The values for each property are standardized by the given equation:
\begin{align}
    \hat{M}_{i, j} = \frac{M_{i, j} - M_j}{S_j},
\end{align}
where \( M_{i,j} \) denotes the value assigned to the \( i^{th} \) amino acid corresponding to the \( j^{th} \) physicochemical property, while \( M_j \) refers to the average of that property calculated across all 20 standard amino acids. \( S_j \) represents the standard deviation of property \( j \) among these 20 amino acids. Hence, NMBAC can be computed using the following formula:
\begin{align}
    P(lg, j) = \frac{1}{L-lg} \sum_{i=1}^{L-lg}(\hat{M}_{i, j}\times \hat{M}_{{i+lg}, j}),
\end{align}
where \( j \in \{1, 2, \ldots, 6\} \) corresponds to one of the six selected physicochemical properties, \( i \in \{1, 2, \ldots, 20\} \) indicates the \( i^{th} \) amino acid in the sequence, \( L \) stands for the total length of the protein sequence, and \( lg \) refers to the lag or interval between amino acid residues. We use the same \( lg \) values as those in \cite{ding2016predicting}, ranging from $1$ to $30$. Consequently, this results in a $180$-dimensional vector $(30 \times 6)$. Moreover, the NMBAC feature incorporates the occurrence frequencies of the 20 amino acids in the protein sequence. As a result, the NMBAC feature vector has a total dimension of $200(180 + 20)$. 
\begin{figure}
    \centering
    \includegraphics[width=0.8\textwidth,height=6cm]{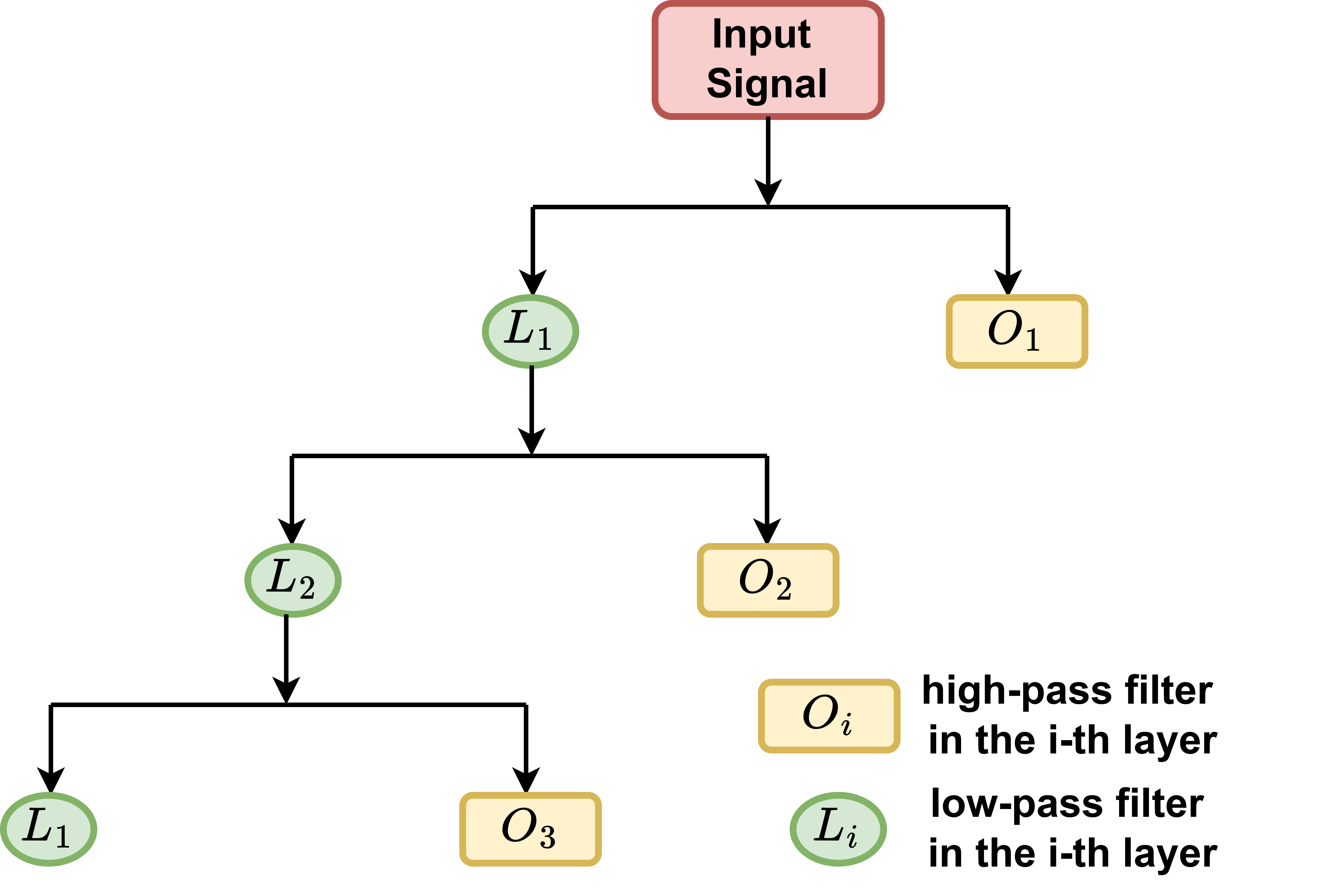}
    \caption{Four-level discrete wavelet transform for PSSM analysis.}
    \label{Four-level discrete wavelet transform for PSSM analysis}
\end{figure}

\subsection{PSSM-Based Average Blocks for Protein Sequences}
Due to its ability to encapsulate evolutionary insights, the Position-Specific Scoring Matrix (PSSM) has gained widespread use in recent years for predicting protein functions. Three different feature types are taken from the PSSM: Pseudo-PSSM (PsePSSM), PSSM-derived Average Blocks (PSSM-AB), and Discrete Wavelet Transform based on PSSM (PSSM-DWT). The original PSSM matrix, denoted as \( Q_{PSSM} \), is defined as follows:
\begin{align}
Q_{PSSM} = 
    \begin{bmatrix}
         Q_{1, 1} & Q_{1, 2} & \cdots & Q_{1, 20} \\
         \vdots   & \vdots  & \ddots & \vdots  \\
         Q_{L, 1} & Q_{L, 2} & \cdots & Q_{L, 20} \\
    \end{bmatrix}_{L \times 20},
\end{align} 
where \( Q_{i,j} \) represents the score for the transition from the $i^{th}$ residue to the $j^{th}$ residue class throughout the evolutionary process. For every protein, the PSSM feature's dimensionality is \( L \times 20 \).

The PSSM is first partitioned into segments by PSSM-AB, with each segment representing $5\%$ of the total protein sequence. As a result, the PSSM is split into $20$ core blocks, each containing $20$ columns, regardless of the protein's length. The extraction formula is given by:
\begin{align}
    E_j = \frac{1}{A_j}\sum_{i=1}^{A_j}Q_i^{(j)},
\end{align}
where \( E_j \) represents the feature vector of the \(j^{th}\) block, and its dimension is \( 1 \times 20 \). The sequence length in the \(j^{th}\) block is denoted by \( A_j \), and \( Q_{i}^{(j)} \) refers to the PSSM value for the \(i^{th}\) residue within the \(j^{th}\) block, also with a dimension of \( 1 \times 20 \). The feature vector derived from each protein sequence has a dimension of $400$ \((20 \times 20)\), with a total of $20$ blocks.

\subsection{PSSM-Based Discrete Wavelet Transform for Protein Sequences}
To derive features using the Discrete Wavelet Transform (DWT), we treat each column of the protein’s PSSM profile as an input signal. This approach, known as PSSM-based DWT (PSSM-DWT), enables the extraction of informative patterns from the PSSM. As shown in Fig. \ref{Four-level discrete wavelet transform for PSSM analysis}, a 4-level discrete wavelet transform is applied to the PSSM for analysis \cite{ding2016predicting}.

High-pass and low-pass filters are used to partition the approximation function at each stage of the procedure, followed by downsampling. As a result, high-frequency coefficients \( O_i \) and low-frequency coefficients \( L_i \) are produced, where \( i \) indicates the decomposition level. For each set of coefficients at every layer, we compute statistical measures including the mean, median, maximum, and minimum. Furthermore, we extract more information from each layer of the low-frequency component by obtaining the first five discrete cosine transform coefficients. Thus, for a protein's PSSM profile, we obtain a $1040$-dimensional feature vector \(((4 + 4 + 5) \times 4 \times 20)\).

\subsection{Pseudo PSSM for Protein Sequences}
The dimensions of the PSSM features vary depending on the length of the protein sequence. When extracting features from the PSSM profile, the PsePSSM is often used to ensure that PSSM features have uniform dimensions and include sequence order information. The features derived from PsePSSM are organized as follows:
\begin{align}
    Q_{PsePSSM} = [Q_1 ~ Q_2 ~ \ldots Q_{20} ~ J_1^{\zeta_1} ~ J_2^{\zeta_1} ~ \ldots ~ J_{20}^{\zeta_1} ~ \ldots ~ J_1^{\zeta_N} ~ J_2^{\zeta_N} ~ \ldots ~ J_{20}^{\zeta_N}]^T,
\end{align}
where
\begin{align}
    Q_j = \frac{1}{L}\sum_{i=1}^L Q_{i, j}, ~~ j = 1, 2, \ldots, 20,
\end{align}
and
\begin{align}
    J_j^{\zeta_l} = \frac{1}{L - \zeta_l}\sum_{i=1}^{L - \zeta_l} [Q_{i, j} - Q_{(i+\zeta_l),j}]^2, ~~~ \zeta_l < L;~ 1 \leq l \leq N,
\end{align}
where \( l \) indicates the gap between residues, \( L \) denotes the total length of the protein sequence, and \( Q_j \) represents the average score for residues in the sequence transitioning to the \(j^{th}\) residue, and \( J_j^{\zeta_l} \) denotes the PSSM values for two adjacent residues separated by a distance of \( l \). We adopt the same range of \( l \) values as those used in \cite{chou2007memtype}, where \( l \) varies from 1 to $15$. As a result, we form a vector of size 320 for each protein sequence, which consists of \(20\) plus \(20 \times 15\) components. We can extract PSSM features with a fixed length using this technique while preserving sequence order information.

\subsection{Feature Selection for Protein Sequences}
In this section, we discuss feature selection and our use of the Minimum Redundancy–Maximum Relevancy (mRMR) method \cite{wei2017local}. This method prioritizes input features by reducing redundancy while enhancing their relevance to the target. In mRMR, mutual information helps evaluate the relevance and redundancy of the features. The calculations are performed as follows:
\begin{align}
    J(g, h) = \iint p(g, h)\log \frac{p(g, h)}{p(g) p(h)}dg~dh,
\end{align}
where \( g \) and \( h \) be two vectors, with \( p(g) \) and \( p(h) \) denoting their respective marginal probability distributions, and \( p(g, h) \) indicating their joint probability distribution. The complete feature set is represented by \( T \), while \( T_t \) refers to an ordered subset containing \( m \) features, and \( T_n \) denotes an unordered subset comprising \( n \) features. Therefore, the relevance \( L \) and redundancy \( D \) of features \( k \) in the subset \( T_n \) are determined using the following formulas:
\begin{align}
    &L(k) = J(k, s), \\
    &D(k) = \frac{1}{m} \sum_{k_i \in T_t} J(k, k_i),
\end{align}
where \( s \) represents the target associated with feature \( k \). To maximize the relevance \( L \) and minimize the redundancy \( D \), we derive:
\begin{align}
    \underset{k_i \in T_n}{\min}[D(k_i) - L(k_i)], ~~~i = 1, 2, 3, \ldots, n. 
\end{align}

Using mRMR, we create a reorganized feature set, with each feature ranked based on its significance. For each type of feature considered in this study, we began by evaluating the importance of individual features using the mRMR algorithm. Based on the resulting ranking, we constructed several new subsets by selecting the top-ranked features, specifically the top \(1/2\), \(3/4\), \(7/8\), and \(15/16\) portions of the full ranked feature list, which were then used for further experimentation.

\section{Proposed Multiview Random Vector Functional Link (MvRVFL) Network}
\label{MvRVFL}
This section offers an in-depth explanation of the proposed MvRVFL model. Initially, we outline the generic mathematical framework of the proposed MvRVFL model, specifically tailored to handle data originating from two distinct views. While training on a single view, the impact of additional views is taken into account by adding a coupling component to the primal formulation of the optimization problem being proposed. An intuitive illustration of the MvRVFL model is shown in Fig. \ref{Geometrical structure of MvRVFL model}. Let $Z_1$ and $Z_2$ represent the nonlinear projection of the class samples corresponding to Vw-$A$ and Vw-$B$, as defined by: $Z_1=[X_1 \hspace{0.2cm} H_1]$ and $Z_2=[X_2 \hspace{0.2cm} H_2]$. Here, $H_1$ and $H_2$ denote the hidden layer matrix corresponding to Vw-$A$ and Vw-$B$, respectively. These matrices are obtained by transforming $X_1$ and $X_2$ using randomly initialized weights and biases and then applying a nonlinear activation function. The optimization formulation for the proposed MvRVFL model is presented as follows:

\begin{align}
\label{eq:6}
    \underset{\beta_1, \beta_2}{\min} \hspace{0.2cm} & \frac{1}{2}\|\beta_1\|^2 + \frac{\theta}{2}\|\beta_2\|^2 + \frac{\mathcal{C}_1}{2}\|\xi_1\|^2 + \frac{\mathcal{C}_2}{2}\|\xi_2\|^2 + \rho\xi_1^t\xi_2 \nonumber \\
    \text{s.t.} \hspace{0.2cm} & Z_1\beta_1 - Y = \xi_1, \nonumber \\
    & Z_2\beta_2 - Y =\xi_2.
\end{align}

\begin{figure}
    \centering
    \includegraphics[width=0.8\textwidth,height=7cm]{MvRVFL.png}
    \caption{An intuitive illustration of  MvRVFL model in two-view setting.}
    \label{Geometrical structure of MvRVFL model}
\end{figure}

Here, $\beta_1$ and $\beta_2$ are the output weight matrix corresponding to Vw-$A$ and Vw-$B$, $\xi_1$ and $\xi_2$ represent the error corresponding to Vw-$A$ and Vw-$B$, respectively. $\mathcal{C}_1$, $\mathcal{C}_2$, $\theta$ and $\rho$ are the regularization parameters.

Each component of the optimization problem of MvRVFL has the following significance.
\begin{enumerate}
\item The terms $\|\beta_1\|$ and $\|\beta_2\|$ are regularization components for Vw-$A$ and Vw-$B$, respectively. These elements help prevent overfitting by limiting the complexity of the classifier sets associated with each specific view.
\item The error variables \(\xi_1\) and \(\xi_2\) are pertinent to both views, enabling tolerance for misclassifications in situations of overlapping distributions.
\item The primal optimization function consists of two separate classification goals for each view, connected through the coupling term \(\rho\xi_1^t \xi_2\). In this context, \(\rho\) is an extra regularization constant, known as the coupling parameter. This term aims to reduce the product of the error variables from both views while balancing the trade-off parameters between them.
\end{enumerate}

The Lagrangian corresponding to the problem Eq. (\ref{eq:6}) is given by
\begin{align}
    L = & \frac{1}{2}\|\beta_1\|^2 + \frac{\theta}{2}\|\beta_2\|^2 + \frac{\mathcal{C}_1}{2}\|\xi_1\|^2 + \frac{\mathcal{C}_2}{2}\|\xi_2\|^2 + \rho\xi_1^t\xi_2  -\alpha_1^t(Z_1\beta_1 - Y - \xi_1) -\alpha_2^t(Z_2\beta_2 - Y -\xi_2),
\end{align}
where $\alpha_1 \in \mathbb{R}^{n \times 1}$ and $\alpha_2 \in \mathbb{R}^{n \times 1}$ are the vectors of Lagrangian multipliers.\\
By utilizing the Karush-Kuhn-Tucker (KKT) conditions, we derive the following:
\begin{align}
    \beta_1 -Z_1^t\alpha_1 = 0,  \label{eq:8} \\
    \theta\beta_2 - Z_2^t\alpha_2 = 0,  \label{eq:9} \\
    \mathcal{C}_1\xi_1 + \rho\xi_2 + \alpha_1 = 0,   \label{eq:10} \\
    \mathcal{C}_2\xi_2 + \rho\xi_1 + \alpha_2 = 0,   \label{eq:11} \\
    Z_1\beta_1 - Y - \xi_1 = 0,  \label{eq:12} \\
    Z_2\beta_2 - Y - \xi_2 = 0.  \label{eq:13}
\end{align}

Using Eqs. (\ref{eq:10}), (\ref{eq:12}) and ( \ref{eq:13}) in Eq. (\ref{eq:8}), we get
\begin{align}
    \beta_1 = Z_1^t\alpha_1, \nonumber \\
    \beta_1 = -Z_1^t(\mathcal{C}_1\xi_1 + \rho\xi_2 ),   \nonumber \\
    \beta_1 = -Z_1^t(\mathcal{C}_1(Z_1\beta_1 - Y) + \rho(Z_2\beta_2 - Y)),  \nonumber \\
    (I + \mathcal{C}_1 Z_1^tZ_1)\beta_1  + \rho Z_1^tZ_2\beta_2  = Z_1^t(\mathcal{C}_1+\rho)Y.
       \label{eq:14}
\end{align}

Using Eqs. (\ref{eq:11}), (\ref{eq:12}) and (\ref{eq:13}) in Eq. (\ref{eq:9}), we get
\begin{align}
 \label{eq:15}
    \rho Z_2^tZ_1\beta_1  + (\theta I + \mathcal{C}_2 Z_2^t Z_2)\beta_2   = Z_2^t(\mathcal{C}_2+\rho)Y.  
\end{align}

Using Eqs. (\ref{eq:14}) and (\ref{eq:15}), the solution of Eq. (\ref{eq:6}) is given by

\begin{align}
\label{eq:16}
    \begin{bmatrix}
            \beta_1 \\ \beta_2
        \end{bmatrix} = \begin{bmatrix}
            (I + \mathcal{C}_1 Z_1^tZ_1)  &  \rho Z_1^tZ_2 \\
            \rho Z_2^tZ_1  &     (\theta I + \mathcal{C}_2 Z_2^t Z_2)  \\
        \end{bmatrix}^{-1} \begin{bmatrix}
            Z_1^t(\mathcal{C}_1+\rho) \\
            Z_2^t(\mathcal{C}_2+\rho)
        \end{bmatrix} Y.
\end{align}
Once the optimal values of \(\beta_1\) and \(\beta_2\) have been determined, a new data point \(x\) can be classified by using the following equation:
\begin{enumerate}
    \item Firstly, the decision function of Vw-$A$ and Vw-$B$ can be articulated as follows: 
   \begin{align}
       & class^A(x^A) = sign( y_{A}), ~\text{where} \hspace{0.2cm} y_A = [x^A \hspace{0.2cm} \phi(x^AW^A + b^A)]\beta_1, ~\text{and} \label{eq:17} \\
       & class^B(x^B) = sign(y_{B}), ~\text{where} \hspace{0.2cm} y_B = [x^B \hspace{0.2cm} \phi(x^BW^B + b^B)]\beta_2. \label{eq:18}
   \end{align}
   \item The decision function, which combines two views, can be articulated in the following manner:
   \begin{align}
       & class(x) = sign(y_{c}), ~\text{where}~~  y_{c} = \frac{1}{2}\left( y_A  +  y_B \right).\label{eq:19}
   \end{align}
\end{enumerate}
Here, $W^A$ ($W^B$) and $b^A$ ($b^B$) are randomly generated weights and biases for Vw-$A$ (Vw-$B$), respectively. \\

\textbf{\textit{Remark:}} The proposed MvRVFL model is referred to as MvRVFL-1 if the classification of the test sample is determined by the functions outlined in Eq. (\ref{eq:19}). Also, we employ majority voting to determine the final anticipated output of the proposed MvRVFL model by aggregating predictions from Eqs. (\ref{eq:17}), (\ref{eq:18}) and (\ref{eq:19}), and we refer to it as MvRVFL-2 model. In the majority voting approach, the class label with the highest number of votes is selected as the final prediction.

\subsection{Complexity Analysis of the Proposed MvRVFL Model}
\label{Computational Complexity}
Let $(X_1, X_2, Y)$ represent the training dataset, where $X_1 \in \mathbb{R}^{n \times m_1}$, $X_2 \in \mathbb{R}^{n \times m_2}$, and $Y \in \mathbb{R}^{n \times 1}$. Here, $m_1$ and $m_2$ refer to the number of features for Vw-$A$ and Vw-$B$, respectively, while $n$ indicates the total number of samples. In RVFL-based models, the computational complexity arises from the need to compute matrix inverses to optimize the weights of the output layer. Thus, the size of the matrices requiring inversion dictates the model's complexity. Therefore, the time complexity of the RVFL model is $\mathcal{O}(n^3)$ or $\mathcal{O}((m+h_l)^3)$, where $h_l$ represents the number of hidden nodes. Hence, the proposed MvRVFL model is required to inverse the matrix of dimension $(m_1 + m_2 + 2h_l)\times (m_1 + m_2 + 2h_l)$. As a result, the time complexity of the MvRVFL model is $\mathcal{O}((m_1 + m_2 + 2h_l)^3)$. The procedure for the MvTPSVM model is outlined concisely in Algorithm \ref{MvRVFL classifier}.

\begin{algorithm}
\caption{MvRVFL classifier}
\label{MvRVFL classifier}
\textbf{Require:} $X_1 \in \mathbb{R}^{n \times m_1}$ and $X_2 \in \mathbb{R}^{n \times m_2}$ be the input matrix of Vw-$A$ and Vw-$B$, respectively and $Y \in \mathbb{R}^{n \times 1}$ be the target vector. Here, \(n\) represents the number of samples, while \(m_1\) and \(m_2\) represent the number of features of the input sample corresponding to Vw-$A$ and Vw-$B$, respectively. 
\begin{algorithmic}[1]
\STATE $W^A \in \mathbb{R}^{m_1 \times h_l}$, $b^A \in \mathbb{R}^{n \times h_l}$, $W^B \in \mathbb{R}^{m_2 \times h_l}$, and $b^B \in \mathbb{R}^{n \times h_l}$ are randomly initialized corresponding to Vw-$A$ and Vw-$B$, with all columns of $b^A$ and $b^B$ being identical.
\STATE Calculate $H_1 = \phi(X_1W^A + b^A) \in \mathbb{R}^{m \times h_1}$ and $H_2 = \phi(X_2W^B + b^B) \in \mathbb{R}^{m \times h_l}$ corresponding to Vw-$A$ and Vw-$B$, where $\phi$ is an activation function.
\STATE Calculate $Z_1=[X_1 \hspace{0.2cm} H_1]$ and $Z_2=[X_2 \hspace{0.2cm} H_2]$.
\STATE Find the unknown matrices, $\beta_1$ and $\beta_2$, representing the output layer weights, using Eq. (\ref{eq:16}). 
\STATE The test sample is assigned to class $+1$ or $-1$ based on the decision rule given in Eq. (\ref{eq:19}).
\end{algorithmic}
\end{algorithm}

\section{Generalization Capability Analysis}
\label{Generalization Capability Analysis}
In this section, we examine the generalization error bound for the MvRVFL model. For this analysis, we adopt the label $y_i \in \{-1, +1\}$. The optimization formulation of the MvRVFL model is then redefined as follows: 

\begin{align}
\label{eq:TT1}
    \underset{\beta_1, \beta_2}{\min} \hspace{0.2cm} & \frac{1}{2}\|\beta_1\|^2 + \frac{\theta}{2}\|\beta_2\|^2 + \frac{\mathcal{C}_1}{2}\sum_{i=1}^n\xi_{1_i}^T\xi_{1_i} + \frac{\mathcal{C}_2}{2}\sum_{i=1}^n\xi_{2_i}^T\xi_{2_i} + \rho\sum_{i=1}^n\xi_{1_i}^t\xi_{2_i} \nonumber \\
    \text{s.t.} \hspace{0.2cm} & Z_1^{(x_i)}\beta_1 - y_i = \xi_{1_i}, \nonumber \\
    & Z_2^{(x_i)}\beta_2 - y_i =\xi_{2_i}, \hspace{0.2cm} i=1, 2, \ldots, n.
\end{align}
If we multiply by $y_i$ on both sides of the constraints of Eq. (\ref{eq:TT1}), it results in
\begin{align}
    & y_iZ_1^{(x_i)}\beta_1 - 1 = y_i\xi_{1_i}, \label{eq:TT2} \\
    & y_iZ_2^{(x_i)}\beta_2 - 1 =y_i\xi_{2_i}, \hspace{0.2cm} i=1, 2, \ldots, n, \label{eq:TT3}
\end{align}
Eqs. (\ref{eq:TT2}) and (\ref{eq:TT3}) are used in the theorem \ref{Theorem2}.

To start, we define the Rademacher complexity \cite{bartlett2002rademacher} as follows:
\begin{definition}
\label{def1}
For a sample set \( S = \{x_1, \ldots, x_n\} \), consisting of \( n \) independent samples drawn from the distribution \( \mathcal{D} \), and a function class \( \mathscr{G} \) defined on \( S \), the empirical Rademacher complexity of \( \mathscr{G} \) is expressed as:
\begin{align}
    \hat{R}_n({\mathscr{G}}) & = \mathbb{E}_{\sigma}\left[ \underset{g \in \mathscr{G}}{\sup} | \frac{2}{n} \sum_{i=1}^n \sigma_i g(x_i) |: x_1, x_2,\ldots, x_n \right],
\end{align}
where \( \sigma = (\sigma_1, \ldots, \sigma_n) \) represents a sequence of independent random variables, each taking values from the set \(\{+1, -1\}\) with equal probability (Rademacher variables). We can express the Rademacher complexity for the function class \( \mathscr{G} \) as:
\begin{align}
    R_n(\mathscr{G}) & = \hat{\mathbb{E}}_S[\hat{R}_n(\mathscr{G})] =\hat{\mathbb{E}}_{S\sigma}\left[ \underset{g \in \mathscr{G}}{\sup} | \frac{2}{n} \sum_{i=1}^n \sigma_i g(x_i) |\right].
\end{align}
\end{definition}
\begin{lemma}
\label{lemma1}
    Choose \( \theta \) within the interval \( (0, 1) \), and let \( \mathscr{G} \) denote a set of functions mapping the input space \( S \) to the range \( [0, 1] \). Suppose \( \{x_i\}_{i=1}^n \) are independently sampled from a probability distribution \( \mathcal{D} \). Then, with a probability of at least \( 1 - \theta \) over random samples of size \( n \), every function \( g \in \mathscr{G} \) meets the specified condition.
    \begin{align}
        \mathbb{E}_{\mathcal{D}}[g(x)] \leq \hat{\mathbb{E}}_{\mathcal{D}}[g(x)] + \hat{R}_n(\mathscr{G}) + 3\sqrt{\frac{\ln(2/\theta)}{2n}}.
    \end{align}
\end{lemma}

\begin{lemma}
\label{lemma2}
   Let $S = \{(x_i, y_i)\}_{i=1}^n$  be a sample set and $Z^{(x)}$ be the enhanced feature matrix. For the function class $\mathscr{G}_B = \{g|g: x \rightarrow \lvert Z^{(x)}\beta \rvert, \|\beta\| \le B\}$ then, the empirical Rademacher complexity of $\mathscr{G}_B$ satisfies:
   \begin{align}
       \hat{R}_n (\mathscr{G}_B) = \frac{2B}{n} \sqrt{\sum_{i=1}^n Z^{(x_i)^T}Z^{(x_i)}}
   \end{align}
\end{lemma}

\begin{lemma}
\label{lemma3}
Suppose \( \mathscr{A} \) is a Lipschitz continuous function with a Lipschitz constant \( \mathcal{L} \), mapping real numbers to real numbers and satisfying \( \mathscr{A}(0) = 0 \). The Rademacher complexity for the function class \( \mathscr{A} \circ \mathscr{G} \) is given by the following expression:
\begin{align}
    \hat{R}_n (\mathscr{A} \circ \mathscr{G}) \leq 2\mathcal{L}\hat{R}_n (\mathscr{G}).
\end{align}
\end{lemma}

The discrepancy between the final outputs from the two views is represented by \( g_m(x) = \lvert Z_1^{(x)}\beta_1 - Z_2^{(x)}\beta_2 \rvert \). Thus, the theoretical expectation of \( g_m(x) \) can be derived from the following theorem.

\begin{theorem}
Let $N \in \mathbb{R}^+$, $\theta \in (0, 1)$, and a training set $T = \{(x_i, y_i)\}_{i=1}^n$ drawn independently and identically from probability distribution $\mathcal{D}$, where $y_i \in \{-1, +1\}$ and $x_i = (x_i^A; x_i^B)$. Define the function class $\mathscr{G}_N = \{g|g: x \rightarrow \lvert Z^{(x)}\beta \rvert, \|\beta\| \le N\}$ and $\hat{\mathscr{G}}_N = \{\hat{g}|\hat{g}: x \rightarrow  Z^{(x)}\beta, \|\beta\| \le N\}$, where $\beta = (\beta_1; \beta_2)$, $Z^{(x)} = \left( Z_1^{(x)}; -Z_2^{(x)} \right) = \left( [x^A \hspace{0.2cm} \phi(x^AW^A + b^A)]; -[x^B \hspace{0.2cm} \phi(x^BW^B + b^B)]\right)$ and $g_m(x) = \lvert Z_1^{(x)}\beta_1 -  Z_2^{(x)}\beta_2 \rvert = \lvert Z^{(x)}\beta \rvert \in \mathscr{G}_N$. Then, every $g_m(x) \in \mathscr{G}_N$ satisfies with a probability of at least $1 - \theta$ over $T$:
\begin{align}
\label{Th1}
    \mathbb{E}_{\mathcal{D}}[g_m(x)] \le 2N + 3N \mathcal{K}_m\sqrt{\frac{\ln(2/\theta)}{2n}}
    + \frac{4N}{n}\sqrt{\sum_{i=1}^n(\|Z_1^{(x_i)}\|^2 + \|Z_2^{(x_i)}\|^2)},
\end{align}
where
\begin{align}
\label{Th2}
    \mathcal{K}_m = \underset{x_i \in \sup(\mathcal{D})}{\max}\sqrt{\|Z_1^{(x_i)}\|^2 + \|Z_2^{(x_i)}\|^2}.
\end{align}
\end{theorem}

\begin{proof}
Let \( \Omega: \mathbb{R} \to [0, 1] \) be a loss function defined as follows:
\begin{align}
    \Omega=\left\{\begin{array}{lll}{-\frac{x}{N\mathcal{K}_m}}, & \text{if}~~ -N\mathcal{K}_m \leq x < 0,  \vspace{3mm} \\ {\frac{x}{N\mathcal{K}_m}} , & \text{if}~~ 0 \leq x \leq N\mathcal{K}_m, \\
    {1}, & \text{otherwise}.
 \end{array}\right. 
\end{align}
For an independently drawn sample \((x_i, y_i)\) from the probability distribution \(\mathcal{D}\), combining \(\|\beta\| \leq N\) with Eq. (\ref{Th2}), we obtain:
\begin{align}
    g_m(x_i) &= \lvert Z^{(x_i)}\beta\rvert \leq N\|Z^{(x_i)}\| \nonumber \\
    &= N \sqrt{\|Z_1^{(x_i)}\|^2 + \|Z_2^{(x_i)}\|^2} \leq N\mathcal{K}_m. 
\end{align}
In that case, \( g_m(x) \) ranges from \( 0 \) to \( N \mathcal{K}_m \), while \( \hat{g}_m(x) \) ranges from \( -N \mathcal{K}_m \) to \( N \mathcal{K}_m \). According to Lemma \ref{lemma1} and with \(\Omega(\hat{g}_m(x))\) ranging from \( 0 \) to \( 1 \), with a probability of at least \( 1 - \theta \) over \( T \), the following inequality is true:
\begin{align}
    \mathbb{E}_{\mathcal{D}}[\Omega(\hat{g}_m(x))] \leq \hat{\mathbb{E}}_T [\Omega(\hat{g}_m(x))] + \hat{R}_n(\Omega \circ \hat{\mathscr{G}}_N) + 3\sqrt{\frac{\ln(2/\theta)}{2n}}.
\end{align}
Assuming \( \Omega(x) \) is a Lipschitz continuous function with a constant of \( \frac{1}{N \mathcal{K}_m} \), passes through the origin, and is uniformly bounded, the following inequality can be established using the findings from Lemmas \ref{lemma2} and \ref{lemma3}:
\begin{align}
    \hat{R}_n(\Omega \circ \hat{\mathscr{G}}_N) &\leq \frac{4}{n\mathcal{K}_m} \sqrt{\sum_{i=1}^n \|Z^{(x_i)}\|^2} \nonumber \\
    &= \frac{4}{n\mathcal{K}_m} \sqrt{\sum_{i=1}^n(\|Z_1^{(x_i)}\|^2 + \|Z_2^{(x_i)}\|^2)}.
\end{align}
Hence,
\begin{align}
    \mathbb{E}_{\mathcal{D}}[\Omega(\hat{g}_m(x))] \leq \hat{\mathbb{E}}_T [\Omega(\hat{g}_m(x))] + 3\sqrt{\frac{\ln(2/\theta)}{2n}} + \frac{4}{n\mathcal{K}_m}\sqrt{\sum_{i=1}^n(\|Z_1^{(x_i)}\|^2 + \|Z_2^{(x_i)}\|^2)}.
\end{align}
Since $g_m(x_i) = N \mathcal{K}_m \Omega(\hat{g}_m(x))$, we have
\begin{align}
\label{THH1-1}
    \mathbb{E}_{\mathcal{D}}[g_m(x)] & = N\mathcal{K}_m  \mathbb{E}_{\mathcal{D}}[\Omega(\hat{g}_m)] \nonumber \\
    & \leq \hat{\mathbb{E}}_T [g_m(x)] + 3N\mathcal{K}_m \sqrt{\frac{\ln{(2/\theta)}}{2n}} + \frac{4N}{n}\sqrt{\sum_{i=1}^n(\|Z_1^{(x_i)}\|^2 + \|Z_2^{(x_i)}\|^2)}. 
\end{align}
Also,
\begin{align}
    \lvert Z^{(x_1)}\beta \rvert \leq N\lvert Z^{(x_i)} \rvert =  N\lvert Z_1^{(x_i)} - Z_2^{(x_i)} \rvert \leq N(\lvert Z_1^{(x_i)}\rvert + \lvert Z_2^{(x_i)} \rvert) \leq 2N 
\end{align}
We can obtain,
\begin{align}
\label{THH1-2}
    \hat{\mathbb{E}}_T [g_m(x)] = \hat{\mathbb{E}}_T [Z^{(x)}\beta] \leq 2N.   
\end{align}
From Eqs. (\ref{THH1-1}) and (\ref{THH1-2}), we can conclude the theorem.
\end{proof}
Consistency plays a crucial role in MVL, aiming to reduce the discrepancy in predictions across different perspectives or views. A lower consistency error typically results in improved generalization performance. The function \( g_m(x) \) is introduced to quantify the consistency discrepancy between the outputs of the two views. By leveraging the empirical Rademacher complexity of the function classes \( \hat{\mathscr{G}}_m \) and the empirical expectation of \( g_m \), a margin-based estimate of the expected consistency gap is developed. As the sample size \( n \) increases, MvRVFL achieves a tighter bound on the consensus error. As the consistency error during training decreases, the generalization error likewise decreases. This theoretical assurance underscores MvRVFL's robust generalization performance concerning consistency.

Next, we examine the generalization error bound. According to (\ref{eq:19}), we use the weighted combination of predictions from the two views to define the prediction function for MvRVFL. Consequently, the generalization error bound for the MvRVFL model can be established as outlined in the subsequent theorem.
\begin{theorem}
\label{Theorem2}
Given $N \in \mathbb{R}^+$, $\theta \in (0, 1)$, and a training set $T = \{(x_i, y_i)\}_{i=1}^n$ drawn  independently and identically from probability distribution $\mathcal{D}$, where $y_i \in \{-1, +1\}$ and $x_i = (x_i^A; \delta x_i^B)$. Define the function classes $\mathscr{G} = \{g|g: x \rightarrow Z^{(x)}\beta, \|\beta\| \le N\}$ and $\hat{\mathscr{G}} = \{\hat{g}|\hat{g}: (x, y) \rightarrow yg(x), g(x) \in \mathscr{G}\}$, where $\beta = (\beta_1; \beta_2)$, $Z^{(x)} = \left( Z_1^{(x)}; Z_2^{(x)} \right) = \left( [x^A \hspace{0.2cm} \phi(x^AW^A + b^A)]; [x^B \hspace{0.2cm} \phi(x^BW^B + b^B)]\right)$ and $g(x) = \left( Z_1^{(x)}\beta_1 +  \delta Z_2^{(x)}\beta_2  \right) = Z^{(x)}\beta \in \mathscr{G}_N$. Then, every $g(x) \in \mathscr{G}$ fulfils with a probability of at least $1 - \theta$ over $T$ satisfies
\begin{align}
\label{Th11}
    \mathbb{P}_{\mathcal{D}}[yg(x)\leq0] \le \frac{1}{n(1+\delta)} \sum_{i=1}^n(\xi_{1_i} + \delta\xi_{2_i}) + 3\sqrt{\frac{\ln(2/\theta)}{2n}} + \frac{4N}{n(1+\delta)}\sqrt{\sum_{i=1}^n(\|Z_1^{(x_i)}\|^2 + \delta^2 \|Z_2^{(x_i)}\|^2)}.
\end{align}
\end{theorem}

\begin{proof}
Let \( \Omega: \mathbb{R} \to [0, 1] \) be a loss function defined as follows:
\begin{align}
    \Omega(x)=\left\{\begin{array}{lll}{1}, & \text{if}~~ x < 0,  \vspace{3mm} \\ {1-\frac{x}{1+\delta}} , & \text{if}~~ 0 \leq x \leq 1+\delta, \\
    {0}, & \text{otherwise}.
 \end{array}\right. 
\end{align}
Following that, we obtained
\begin{align}
\label{TH2:3}
    \mathbb{P}_{\mathcal{D}}(yg(x)\leq0) \leq \mathbb{E}_{\mathcal{D}}[\Omega(\hat{g}(x, y))].
\end{align}
Using Lemma \ref{lemma1}, we have
\begin{align}
    \mathbb{E}_{\mathcal{D}}[\Omega(\hat{g}(x, y))-1] &\leq \hat{\mathbb{E}}_{T}[\Omega(\hat{g}(x, y))-1] + 3\sqrt{\frac{\ln(2/\theta)}{2n}} + \hat{R}_n((\Omega-1)\circ \mathscr{G}).
\end{align}
Therefore,
\begin{align}
\label{TH2:4}
    \mathbb{E}_{\mathcal{D}}[\Omega(\hat{g}(x, y))] \leq \hat{\mathbb{E}}_{T}[\Omega(\hat{g}(x, y))] + 3\sqrt{\frac{\ln(2/\theta)}{2n}} + \hat{R}_n((\Omega-1)\circ \mathscr{G}). 
\end{align}

By using Eq. \ref{eq:TT2} and Eq. \ref{eq:TT3}, we deduce:
\begin{align}
\label{TH2:5}
    \hat{\mathbb{E}}_{T}[\Omega(\hat{g}(x, y))] & \leq \frac{1}{n(1+\delta)} \sum_{i=1}^n [1 + \delta - y_ig(x_i)]_+ \nonumber \\
    & = \frac{1}{n(1+\delta)} \sum_{i=1}^n [1-y_ig_A(x_i^A) + \delta(1 - y_if_B(x_i^B)) ]_+  \nonumber \\
    & \leq \frac{1}{n(1+\delta)} \sum_{i=1}^n \{[1-y_ig_A(x_i^A)]_+  + \delta[1 - y_if_B(x_i^B))]_+ \} \nonumber \\
    & \leq \frac{1}{n(1+\delta)} \sum_{i=1}^n [y_i(\xi_{1_i} + \delta\xi_{2_i})]_+ \nonumber \\
    & \leq \frac{1}{n(1+\delta)} \sum_{i=1}^n (\xi_{1_i} + \delta\xi_{2_i}).
\end{align}

Given that \((\Omega - 1)(x)\) exhibits Lipschitz continuity with a Lipschitz constant of \(\frac{1}{1 + \delta}\), is bounded, and passes through the origin, we can establish the following inequality using Lemma \ref{lemma3}:
\begin{align}
    \hat{R}_n((\Omega-1) \circ \hat{\mathscr{G}}) \leq \frac{2}{1+\delta} \hat{R}_n(\hat{\mathscr{G}}).
\end{align}
Next, we obtain (using the Definition \ref{def1}):
\begin{align}
    \hat{R}_n(\hat{\mathscr{G}}) & = \mathbb{E}_{\sigma}\left[ \underset{\hat{g} \in \hat{\mathscr{G}}}{\sup} \left| \frac{2}{n} \sum_{i=1}^n \sigma_i \hat{g}(x_i, y_i) \right|  \right] \nonumber \\
    & = \mathbb{E}_{\sigma}\left[ \underset{g \in \mathscr{G}}{\sup} \left| \frac{2}{n} \sum_{i=1}^n \sigma_i y_ig(x_i) \right| \right] \nonumber \\
    & = \mathbb{E}_{\sigma}\left[ \underset{g \in \mathscr{G}}{\sup} \left| \frac{2}{n} \sum_{i=1}^n \sigma_i g(x_i) \right| \right] \nonumber \\
    & = \hat{R}_n(\mathscr{G}).
\end{align}
Combining this with Lemma \ref{lemma2}, we have:
\begin{align}
\label{TH2:6}
    \hat{R}_n((\Omega-1) \circ \hat{\mathscr{G}}) & \leq \frac{2}{1+\delta} \hat{R}_n(\mathscr{G})  \nonumber \\
    & = \frac{4N}{n(1+\delta)}\sqrt{\sum_{i=1}^n Z^{(x_i)^T}Z^{(x_i)}} \nonumber \\
    & = \frac{4N}{n(1+\delta)}\sqrt{\sum_{i=1}^n (Z_1^{(x_i)^T}Z_1^{(x_i)} + \delta^2 Z_2^{(x_i)^T}Z_2^{(x_i)})} \nonumber \\
    & = \frac{4N}{n(1+\delta)}\sqrt{\sum_{i=1}^n(\|Z_1^{(x_i)}\|^2 + \delta^2 \|Z_2^{(x_i)}\|^2)}.
\end{align}
Moreover, by combining equations Eqs. (\ref{TH2:3}), (\ref{TH2:4}), (\ref{TH2:5}), and (\ref{TH2:6}), we can obtain inequality, which demonstrates the generalization error bound of MvRVFL.
\end{proof}
We define error function \( g(x) \) by employing the integrated decision function as specified in MvRVFL (\ref{eq:19}). Integrating the empirical Rademacher complexity of \( \mathscr{G} \) with the empirical mean of \( \hat{\mathscr{G}} \) allows us to derive a margin-based estimation of the classification error rate. As the number of samples \( n \) grows, the MvRVFL model yields a tight generalization error bound, enhancing its classification capability. A reduction in training error typically leads to a simultaneous decline in the generalization error. This theoretical result ensures that MvRVFL exhibits improved generalization performance.

\section{Experiments and Results}
\label{Experiments and Results}
The effectiveness of the proposed MvRVFL model is evaluated by conducting a comparative analysis with baseline models on the DNA-binding proteins dataset \cite{liu2014idna}. Furthermore, we evaluate our proposed model using publicly available AwA\footnote{\url{http://attributes.kyb.tuebingen.mpg.de}} and Corel5k\footnote{\url{https://wang.ist.psu.edu/docs/related/}} benchmark datasets. We compare our proposed MvRVFL models with SVM2K \cite{farquhar2005two}, MvTSVM \cite{xie2015multi}, ELM (Extreme learning machine, also known as RVFL without direct link (RVFLwoDL)) \cite{huang2006extreme}, RVFL \cite{pao1994learning}, MVLDM \cite{hu2024multiview}, CDC (A simple framework for complex data clustering) \cite{kang2024cdc}, and MCMC (Multilevel contrastive multiview clustering with dual self-supervised learning) \cite{bian2025multilevel}. We denote the ELM model as ELM-Vw-$A$ and ELM-Vw-$B$ if it is trained over Vw-$A$ and Vw-$B$ of the datasets, respectively. Similar nomenclature is followed for RVFL-Vw-$A$ and RVFL-Vw-$B$.

\subsection{Experimental Setup}
The experiments are performed on a personal computer equipped with an Intel(R) Xeon(R) Gold 6226R processor running at 2.90 GHz. The system uses Windows 11 with 128 GB of RAM and executes Python 3.11. The dataset is randomly divided into training and testing subsets, with $30\%$ allocated for testing and $70\%$ for training. A five-fold cross-validation approach combined with grid search is utilized to fine-tune the model hyperparameters, examining the specified ranges: $\mathcal{C}_i = \theta = \rho = \{10^{-5}, 10^{-4}, \ldots, 10^5\}$ for $i=1,2.$ The range $3:20:203$ is used to choose the number of hidden nodes.

\subsection{Evaluation on DNA-binding Protein Dataset}
We conduct a comparison of the proposed MvRVFL model by utilizing two benchmark datasets, namely PDB186 and PDB1075. The training dataset, consisting of $1075$ protein samples, is derived from the PDB1075 \cite{liu2014idna} dataset. Within this dataset, $550$ sequences are labeled as negative (non-DBPs), while $525$ sequences are categorized as positive (DBPs). The test set, comprising $186$ protein samples, is derived from the PDB186 dataset \cite{lou2014sequence}). It consists of an equal number of negative and positive sequences. The features partially dictate the upper-performance limit of the model. To evaluate the impact of various features and their combination on DBP prediction, we evaluate each individual feature utilized in the proposed MvRVFL model. The DBP sequence is represented through three distinct views: physicochemical property, evolutionary information, and amino acid composition. These views are converted into feature matrices through various extraction techniques. Specifically, the amino acid composition is processed using Multi-scale Continuous and Discontinuous (MCD) \cite{you2014prediction}, the physicochemical properties are analyzed with Normalized Moreau-Broto Autocorrelation (NMBAC), and evolutionary information is transformed through methods like PSSM-based Average Blocks (PSSM-AB), PSSM-based Discrete Wavelet Transform (PSSM-DWT), and Pseudo Position-Specific Scoring Matrix (PsePSSM) \cite{liu2015pse}. Five distinct types of features extracted from the sequences are utilized, encompassing PsePSSM, PSSM-DWT, NMBAC, MCD, and PSSM-AB.

\begin{table*}[ht!]
\centering
    \caption{The performance of the proposed MvRVFL-1 and MvRVFL-2 models is compared with baseline models based on classification Acc. for the DNA-binding protein datasets.}
    \label{Average ACC and average rank for proteins datasets}
    \resizebox{1\linewidth}{!}{
\begin{tabular}{llccccccccc}
\hline
Dataset & Dataset & SVM2K \cite{farquhar2005two} & MvTSVM \cite{xie2015multi} & ELM-Vw-$A$ \cite{huang2006extreme} & ELM-Vw-$B$ \cite{huang2006extreme} & RVFL-Vw-$A$ \cite{pao1994learning} & RVFL-Vw-$B$ \cite{pao1994learning} & MVLDM \cite{hu2024multiview} & MvRVFL-1$^{\dagger}$ & MvRVFL-2$^{\dagger}$ \\ 
index& name & (Acc., Seny.) & (Acc., Seny.) & (Acc., Seny.) & (Acc., Seny.) & (Acc., Seny.) & (Acc., Seny.) & (Acc., Seny.) & (Acc., Seny.) & (Acc., Seny.) \\  \hline
1. & MCD \& NMBAC & $(65, 66.13)$ & $(60, 50)$ & $(70.43, 72.04)$ & $(66.13, 65.59)$ & $(69.89, 77.42)$ & $(65.59, 60.22)$ & $(61.83, 79.46)$ & $(75.27, 89.25)$ & $(68.92, 79.89)$ \\
2. & MCD \& PSSM-AB & $(60.89, 68.13)$ & $(65.45, 64.85)$ & $(61.83, 58.06)$ & $(68.82, 78.49)$ & $(74.52, 75.27)$ & $(73.66, 83.87)$ & $(62.9, 80)$ & $(75.81, 74.19)$ & $(74.04, 80.65)$ \\
3. & MCD \& PSSM-DWT & $(69.56, 67.34)$ & $(66.87, 74.23)$ & $(65.59, 63.44)$ & $(68.28, 79.57)$ & $(70.43, 78.49)$ & $(69.89, 87.1)$ & $(65.59, 81.25)$ & $(69.42, 91.4)$ & $(66.92, 83.23)$ \\
4. & MCD \& PsePSSM & $(74.89, 67.51)$ & $(69.85, 60)$ & $(70.97, 77.42)$ & $(68.28, 78.49)$ & $(79.52, 72.04)$ & $(68.28, 75.27)$ & $(59.68, 83.26)$ & $(80.11, 78.49)$ & $(76.45, 75.48)$ \\
5. & NMBAC \& PSSM-DWT & $(72.72, 67.34)$ & $(75.85, 60)$ & $(64.52, 64.52)$ & $(68.28, 79.57)$ & $(74.52, 61.29)$ & $(68.82, 76.34)$ & $(66.67, 74.52)$ & $(76.88, 81.06)$ & $(76.98, 80)$ \\
6. & NMBAC \& PSSM-AB & $(72.45, 70.76)$ & $(69.62, 72.49)$ & $(67.2, 64.52)$ & $(70.97, 76.34)$ & $(76.13, 64.52)$ & $(70.97, 80.65)$ & $(66.67, 67.74)$ & $(78.49, 81.98)$ & $(75.85, 78.25)$ \\
7. & NMBAC \& PsePSSM & $(84.51, 65.61)$ & $(75, 62.38)$ & $(64.52, 62.37)$ & $(66.67, 78.49)$ & $(76.13, 65.59)$ & $(70.43, 80.65)$ & $(61.29, 59.14)$ & $(80.8, 85.05)$ & $(80.54, 76.13)$ \\
8. & PSSM-AB \& PSSM-DWT & $(58.23, 67.38)$ & $(68.74, 63.25)$ & $(74.19, 87.1)$ & $(68.82, 82.8)$ & $(71.51, 81.72)$ & $(66.13, 78.49)$ & $(73.66, 79.57)$ & $(76.34, 79.82)$ & $(72.89, 77.2)$ \\
9. & PsePSSM \& PSSM-AB & $(66.84, 65.96)$ & $(68.23, 72.58)$ & $(74.19, 81.72)$ & $(66.67, 76.34)$ & $(70.43, 82.8)$ & $(74.74, 77.42)$ & $(69.89, 77.42)$ & $(75.65, 83.49)$ & $(74.54, 80.43)$ \\
10. & PsePSSM \& PSSM-DWT & $(59.44, 67.8)$ & $(70.76, 63.8)$ & $(67.2, 78.49)$ & $(68.28, 76.34)$ & $(71.82, 83.87)$ & $(72.04, 81.72)$ & $(50, 72.89)$ & $(74.41, 86.38)$ & $(72.91, 84.73)$ \\ \hline
& Average Acc. & $68.45$ & $69.04$ & $68.06$ & $68.12$ & $73.49$ & $70.06$ & $63.82$ & $\textbf{76.32}$ & $74$ \\ \hline
\multicolumn{11}{l}{\begin{tabular}[c]{@{}l@{}} Seny. denote the Sensitivity. \end{tabular}} \\
\multicolumn{11}{l}{$^{\dagger}$ represents the proposed models.}
\end{tabular}}
\end{table*}
\subsubsection{Comparison of the proposed model with the existing state-of-the-art models (non-DNA binding prediction models)}
Table \ref{Average ACC and average rank for proteins datasets} provides the results (in terms of Accuracy and Sensitivity) among the baseline and proposed MvRVFL models for DBP prediction. The MvRVFL-1 and MvRVFL-2 models ranked first and second, respectively, with average accuracies of 76.32\% and 74\%. In contrast, the average Acc. of the baseline SVM2K, MvTSVM, ELM-Vw-$A$, ELM-Vw-$B$, RVFL-Vw-$A$, RVFL-Vw-$B$, and MVLDM models are $68.45\%$, $69.04\%$, $68.06\%$,  $68.12\%$, $73.49\%$, $ 70.06\%$ and $63.82\%$, respectively. Compared to the third-top model, RVFL-Vw-$A$, the proposed models (MvRVFL-1 and MvRVFL-2) exhibit average Acc. of approximately $2.83\%$ and $0.51\%$ higher, respectively. For the MCD \& PsePSSM, NMBAC \& PsePSSM, PsePSSM \& PSSM-AB and PsePSSM \& PSSM-DWT cases, the proposed MvRVFL-1 model archives the Acc. of $80.11\%$, $80.80\%$, $75.65\%$, and $74.41\%$, respectively, emerging as the top performers. The results suggest the significance of PsePSSM as an important feature for predicting DBPs. Thus, the proposed MvRVFL consistently demonstrates exceptional performance, achieving high Acc. across diverse scenarios, solidifying its distinction among other models. From Table \ref{Average ACC and average rank for proteins datasets}, we can see that our proposed MvRVFL-1 model has Seny. $89.25$ which is the second highest among all the datasets. The MvRVFL-1 model is highly effective at correctly detecting true positive cases in the MCD \& NMBAC dataset, thus demonstrating its superior performance in identifying relevant instances compared to other models. Also, the Seny. for the MCD \& PSSM-DWT dataset reaches a peak value of $91.40$, indicating that our proposed model excels in predicting MCD features, achieving the best performance among all compared models. Furthermore, the MvRVFL-2 model employs a majority voting mechanism, combining predictions from multiple classifiers to enhance its efficiency and robustness. This approach leads to more reliable and accurate predictions, underscoring the superior performance of the MvRVFL-2 model in predicting MCD features compared to baseline models.

\begin{figure*}
\begin{minipage}{.31\linewidth}
\centering
\subfloat[PsePSSM \& PSSM-AB]{\label{fig:7a}\includegraphics[scale=0.26]{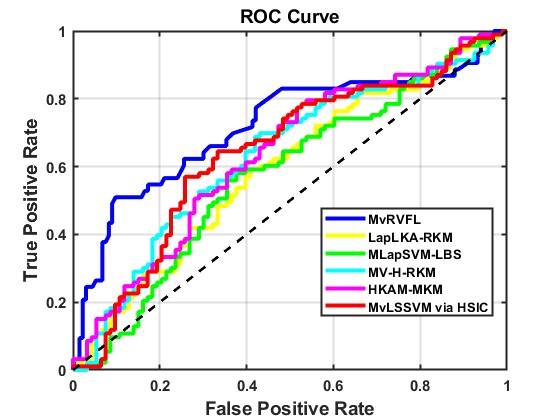}}
\end{minipage}
\begin{minipage}{.31\linewidth}
\centering
\subfloat[NMBAC \& PSSM-DWT]{\label{fig:7b}\includegraphics[scale=0.26]{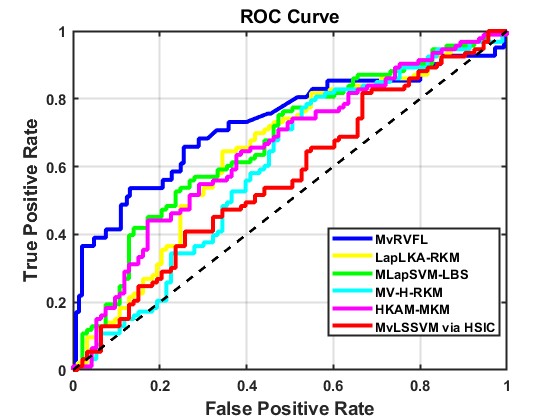}}
\end{minipage}
\begin{minipage}{.31\linewidth}
\centering
\subfloat[MCD \& PsePSSM]{\label{fig:7c}\includegraphics[scale=0.26]{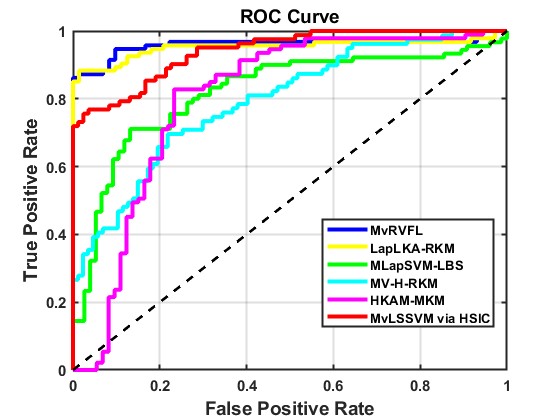}}
\end{minipage}
\caption{ROC curves comparing the performance of the proposed MvRVFL model with various DNA-binding protein prediction models.}
\label{ROC Cuve}
\end{figure*}

\subsubsection{Comparison of the proposed model with the existing DNA binding prediction models}
We compare our proposed MvRVFL model with existing DNA binding protein prediction models, including MvLSSVM via HSIC \cite{zhao2022identification}, HKAM-MKM \cite{zhao2022hkam}, MV-H-RKM \cite{guan2022mv}, MLapSVM-LBS \cite{sun2022mlapsvm}, and LapLKA-RKM \cite{qian2023identification}. Table \ref{Average ACC and average rank for proteins datasets with oh} demonstrates that all five features (MCD, NMBAC, PSSM-DWT, PsePSSM, and PSSM-AB) are beneficial for predicting DBPs using MvRVFL. Among these, MCD\_vs\_PsePSSM achieves the best result (Acc. = $80.11\%$). This suggests that PsePSSM is a crucial feature for predicting DBPs. As illustrated in Table \ref{Average ACC and average rank for proteins datasets with oh}, the proposed MvRVFL-1 and MvRVFL-2 models achieve average accuracies of 76.32\% and 74\%, respectively, both outperforming the baseline models. The proposed MvRVFL-1 and MvRVFL-2 exhibit exceptional generalization performance, marked by consistently higher Acc., indicating a high level of confidence in their learning process. It is found that the MvRVFL model is comparable to the baseline models in most cases. By incorporating the coupling term, MvRVFL can integrate information from both views. This approach allows for a larger error variable in one view if it is compensated by the other view by minimizing the product of the error variables. The MvRVFL-1 model demonstrates strong performance in all these areas, showing its robustness and reliability in comparison with existing models. 

Figure \ref{ROC Cuve} displays the ROC curve, highlighting the enhanced performance of the proposed MvRVFL model over baseline methods on DNA-binding protein datasets. The curve offers a thorough assessment of the model's classification performance by plotting sensitivity (true positive rate) versus the false positive rate across different decision thresholds. The MvRVFL model achieves a significantly higher area under the ROC curve (AUC), indicating a better balance between sensitivity and specificity. An elevated AUC indicates that the MvRVFL model possesses a stronger capability to differentiate positive cases from negative ones, resulting in improved prediction accuracy. The increase in AUC demonstrates the model's improved ability to minimize false negatives while accurately identifying true positives, leading to more reliable detection.

The superior performance of the MvRVFL model can be attributed to the inclusion of crucial features, particularly PsePSSM. PsePSSM captures essential evolutionary information from the protein sequences, providing a more comprehensive representation that significantly contributes to the model's predictive accuracy. By leveraging PsePSSM, the model benefits from detailed sequence order information, which enhances its ability to correctly classify DBPs.

These results demonstrate the MvRVFL model's superiority over baseline models in classification tests, demonstrating its superiority and effectiveness. The importance of PsePSSM as a feature is evident, as it plays a pivotal role in improving the model's overall accuracy and reliability in identifying true positives and reducing false negatives.

\begin{figure}[ht!]
\centering
{\label{cong2}\includegraphics[scale=0.45]{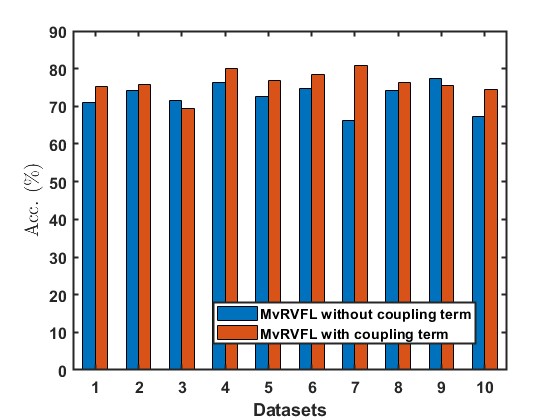}}
\caption{Ablation study of the coupling term for MvRVFL: The x-axis represents the datasets indexed in Table \ref{Average ACC and average rank for proteins datasets with oh}, and the y-axis represents the Acc. ($\%$).}
\label{Ablation study}
\end{figure}

\subsubsection{Ablation study on DNA binding datasets}
To validate our claim that the coupling term acts as a bridge among multiple views, enhancing the coordination between different features and thereby improving the training efficiency of the proposed MvRVFL model, we conducted an ablation study. In this study, we set the coupling term to zero (\(\rho = 0\)) in the optimization (\ref{eq:6}) to assess its true impact. The results, depicted in Fig. \ref{Ablation study}, demonstrate that the MvRVFL model consistently outperforms the model without coupling term across most datasets, confirming the effectiveness of the coupling term.

This result underscores the importance of the coupling term \(\rho\) in the model's design. It enables the MvRVFL model to seamlessly combine information from both the image and caption views. By reducing the product of the error terms from both views, the model can tolerate a higher error in one view, as long as it is compensated by a lower error in the other view. This balancing approach improves the model's overall robustness and reliability.

The ablation study clearly demonstrates that incorporating the coupling term enables the MvRVFL model to better manage the complementary information from different views, resulting in improved performance and more reliable predictions.

\begin{table*}[ht!]
\centering
    \caption{Performance comparison of the proposed MvRVFL-1 and MvRVFL-2 along with DNA binding protein prediction models.}
    \label{Average ACC and average rank for proteins datasets with oh}
    \resizebox{1\linewidth}{!}{
\begin{tabular}{llccccccc}
\hline
Dataset & Dataset  & MvLSSVM via HSIC \cite{zhao2022identification} & HKAM-MKM \cite{zhao2022hkam} & MV-H-RKM \cite{guan2022mv} & MLapSVM-LBS \cite{sun2022mlapsvm} & LapLKA-RKM \cite{qian2023identification} & MvRVFL-1$^{\dagger}$ & MvRVFL-2$^{\dagger}$ \\
index & name & (Acc., Seny.) & (Acc., Seny.) & (Acc., Seny.) & (Acc., Seny.) & (Acc., Seny.) & (Acc., Seny.) & (Acc., Seny.)  \\ \hline
1. & MCD\_vs\_NMBAC & $(74.47, 77.78)$ & $(72.96, 73.04)$ & $(65.7, 74.06)$ & $(61.7, 81.03)$ & $(68.85, 86.54)$ & $(75.27, 89.25)$ & $(68.92, 79.89)$ \\
2. & MCD\_vs\_PSSM-AB & $(74.58, 82.35)$ & $(71.72, 71.79)$ & $(69.46, 72.82)$ & $(74.82, 73.85)$ & $(63.64, 77.89)$ & $(75.81, 74.19)$ & $(74.04, 80.65)$ \\
3. & MCD\_vs\_PSSM-DWT & $(66.57, 89.44)$ & $(67.74, 86.92)$ & $(60, 84.58)$ & $(65.69, 89.34)$ & $(66.21, 81.84)$ & $(69.42, 91.4)$ & $(66.92, 83.23)$ \\
4. & MCD\_vs\_PsePSSM & $(66.99, 72.14)$ & $(73.11, 74.8)$ & $(79.46, 74.63)$ & $(76.47, 76.56)$ & $(78.66, 78.41)$ & $(80.11, 78.49)$ & $(76.45, 75.48)$ \\
5. & NMBAC\_PSSM-DWT & $(75.34, 73.61)$ & $(76.34, 78.52)$ & $(67.03, 74.89)$ & $(72.88, 80)$ & $(74.13, 79.59)$ & $(76.88, 81.06)$ & $(76.98, 80)$ \\
6. & NMBAC\_vs\_PSSM-AB & $(76.47, 80.96)$ & $(71.39, 78.23)$ & $(77.78, 74.7)$ & $(76.57, 73.65)$ & $(73.75, 75.38)$ & $(78.49, 81.98)$ & $(75.85, 78.25)$ \\
7. & NMBAC\_vs\_PsePSSM & $(74.57, 74.19)$ & $(76.32, 82.27)$ & $(77.27, 79.35)$ & $(76.99, 80.33)$ & $(79.35, 81.6)$ & $(80.8, 85.05)$ & $(80.54, 76.13)$ \\
8. & PSSM-AB\_vs\_PSSM-DWT & $(77.53, 74.55)$ & $(90.86, 67.77)$ & $(91.3, 67.83)$ & $(55.34, 76.12)$ & $(62.12, 76.62)$ & $(76.34, 79.82)$ & $(72.89, 77.2)$ \\
9. & PsePSSM\_vs\_PSSM-AB & $(74.4, 84.29)$ & $(74.37, 76.6)$ & $(73.48, 79.44)$ & $(76.47, 82.59)$ & $(74.74, 80.97)$ & $(75.65, 83.49)$ & $(74.54, 80.43)$ \\
10. & PsePSSM\_vs\_PSSM-DWT & $(75.82, 76.21)$ & $(75.84, 72.26)$ & $(73.04, 75.65)$ & $(72.31, 74.78)$ & $(71.93, 82.47)$ & $(74.41, 86.38)$ & $(72.91, 84.73)$ \\ \hline 
& Average Acc. & $73.67$ & $75.07$ & $73.45$ & $70.92$ & $71.34$ & $\textbf{76.32}$ & $74$  \\ \hline
\multicolumn{9}{l}{\begin{tabular}[c]{@{}l@{}} Seny. denote the Sensitivity. \end{tabular}} \\
\multicolumn{9}{l}{$^{\dagger}$ represents the proposed models.}
\end{tabular}}
\end{table*}

\subsubsection{Sensitivity of \texorpdfstring{$\mathcal{C}_1$}{C1} and \texorpdfstring{$\rho$}{rho} combinations}
In this subsection, we evaluate the impact of the regularization parameters $\mathcal{C}_1$ and $\rho$ of MvRVFL. In Fig. \ref{effect of parameter c and rho}, the values of $\mathcal{C}_1$ and $\rho$ are varied from $10^{-5}$ to $10^{5}$, keeping the other parameters is fixed at their optimal values and the corresponding Acc. values are recorded. The performance of MvRVFL is depicted under different parameter settings $\mathcal{C}_1$ and $\rho$. The parameter $\rho$ serves as a coupling term designed to minimize the product of error variables between Vw-$A$ and Vw-$B$. When $\mathcal{C}_1$ lies between $10^{-1}$ and $10^{5}$ with the same value of $\rho$, there is a corresponding improvement observed in the Acc. values. These findings suggest that when evaluating parameters $\mathcal{C}_1$ and $\rho$, the model's performance is primarily influenced by $\mathcal{C}_1$ rather than $\rho$. Therefore, precise adjustment of the hyperparameters in the proposed MvRVFL model is essential for achieving optimal generalization performance.

\subsubsection{Influence of the number of hidden nodes \texorpdfstring{$h_l$}{hl}}
To thoroughly understand the robustness of the proposed MvRVFL model, it is essential to examine how it responds to variations in the number of hidden nodes \(h_l\).
The influence of the hyperparameter $h_l$ is depicted in Fig. \ref{effect of parameter act}. Fig. \ref{fig:5a} and \ref{fig:5c} show that the performance peaks at $h_l=23$ and then gradually declines as $h_l$ increases further. Therefore, to achieve the best performance from the MvRVFL, we recommend using $h_l=23$. The performance peaks at $h_l=23$ and $h_l=83$ and then gradually declines as $h_l$ increases further depicted in Fig. \ref{fig:5d}. From Fig. \ref{fig:5b}, the performance consistently improves with an increase in the number of hidden nodes until reaching a plateau. Optimal performance is typically achieved with higher values of $h_l$. We recommend fine-tuning the hyperparameters to attain the best performance from the proposed models for specific tasks.

\begin{figure*}
\begin{minipage}{.246\linewidth}
\centering
\subfloat[MCD \& NMBAC]{\includegraphics[scale=0.20]{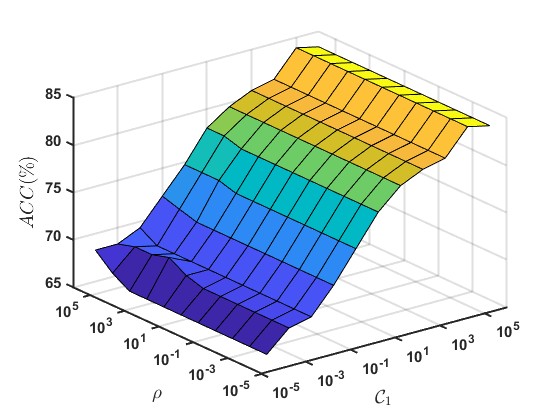}}
\end{minipage}
\begin{minipage}{.246\linewidth}
\centering
\subfloat[MCD \& PSSM-DWT]{\includegraphics[scale=0.20]{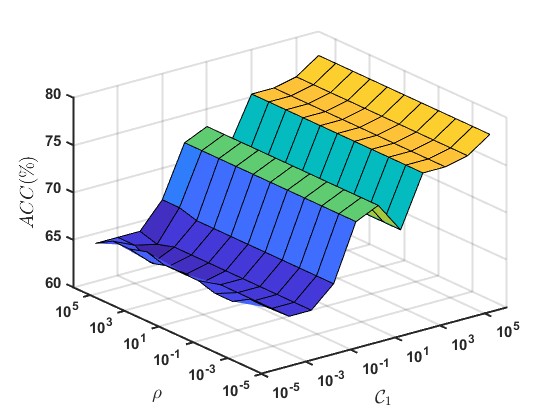}}
\end{minipage}
\begin{minipage}{.246\linewidth}
\centering
\subfloat[NMBAC \& PSSM-DWT]{\includegraphics[scale=0.20]{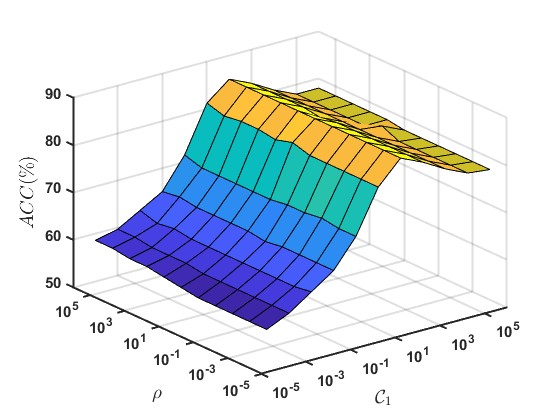}}
\end{minipage}
\begin{minipage}{.246\linewidth}
\centering
\subfloat[NMBAC \& PSSM-AB]{\includegraphics[scale=0.20]{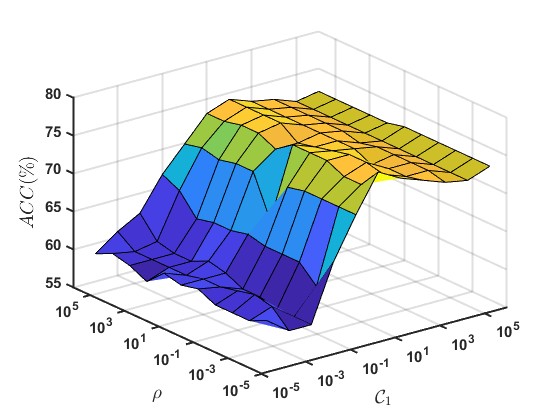}}
\end{minipage}
\caption{The impact of the parameters $\mathcal{C}_1$ and $\rho$ on the performance of the proposed MvRVFL model.}
\label{effect of parameter c and rho}
\end{figure*}

\begin{figure*}
\begin{minipage}{.246\linewidth}
\centering
\subfloat[MCD \& PsePSSM]{\label{fig:5a}\includegraphics[scale=0.20]{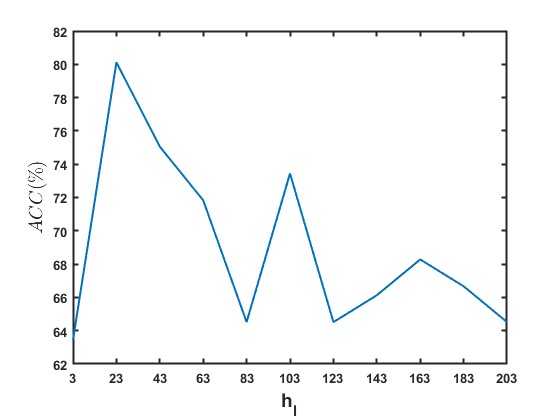}}
\end{minipage}
\begin{minipage}{.246\linewidth}
\centering
\subfloat[MCD \& PSSM-AB]{\label{fig:5b}\includegraphics[scale=0.20]{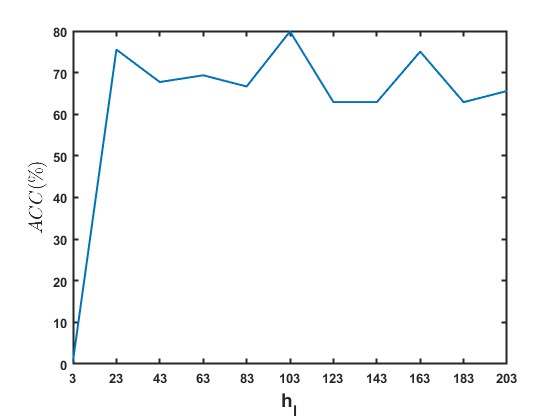}}
\end{minipage}
\begin{minipage}{.246\linewidth}
\centering
\subfloat[MCD \& PSSM-DWT]{\label{fig:5c}\includegraphics[scale=0.20]{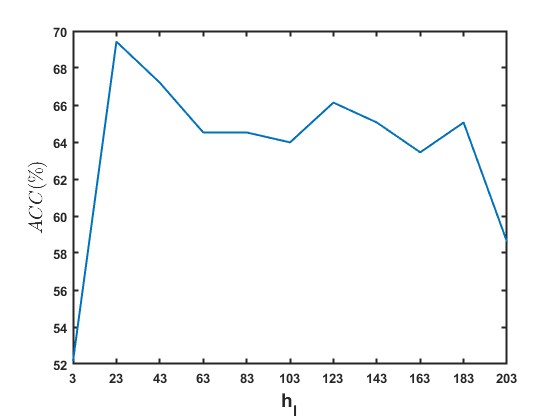}}
\end{minipage}
\begin{minipage}{.246\linewidth}
\centering
\subfloat[PsePSSM \& PSSM-DWT]{\label{fig:5d}\includegraphics[scale=0.20]{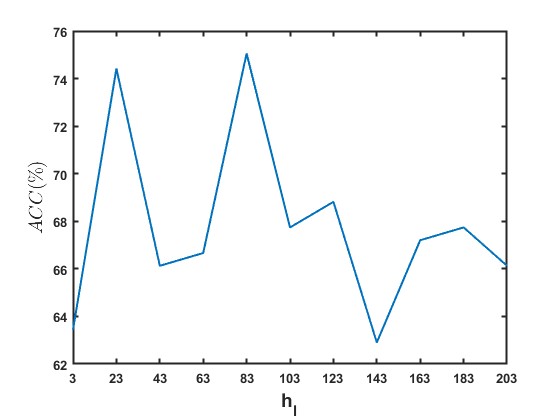}}
\end{minipage}
\caption{The impact of the parameter $h_l$ on the performance of the proposed MvRVFL model.}
\label{effect of parameter act}
\end{figure*}

\subsection{Evaluation on UCI and KEEL Datasets}
\label{Evaluation on UCI and KEEL Datasets}
This section provides an in-depth analysis, featuring a comparison of the proposed MvRVFL model with baseline models on 27 benchmark datasets from UCI \cite{dua2017uci} and KEEL \cite{derrac2015keel}. Since the UCI and KEEL datasets lack inherent multiview characteristics, we generate an additional view by applying principal component analysis (PCA), retaining 95\% of the variance. The resulting reduced feature set is denoted as Vw-$B$, while the original dataset is referred to as Vw-$A$ \cite{wang2023safe}. The performance of the proposed MvRVFL model along with the baseline models is evaluated using Acc. metrics along with the corresponding optimal hyperparameters, as depicted in Table S.1 in the Supplementary material. Table \ref{Average ACC and average rank for UCI and KEEL datasets} shows the average Acc. for the proposed MvRVFL-1 and MvRVFL-2 models along with the baseline SVM2K, MvTSVM, ELM-Vw-$A$, ELM-Vw-$B$, RVFL-Vw-$A$, RVFL-Vw-$B$, MVLDM, CDC, and MCMC models are $85.15\%$, $83.18\%$, $76.65\%$, $67.24\%$, $82.66\%$, $81.7\%$, $83.14\%$, $81.26\%$, $79.9\%$, $81.10\%$, and $81.98\%$, respectively. The proposed MvRVFL-1 achieved the top position, and MvRVFL-2 ranked second in terms of average Acc. This demonstrates that the proposed MvRVFL-1 and MvRVFL-2 models exhibit a significant level of confidence in their predictive capabilities. Average Acc. can be deceptive, as it may mask a model's strong performance on one dataset by compensating for weaker results on another. To overcome the drawbacks of using average accuracy and determine the relevance of the results, we employed a suite of statistical tests recommended by \citet{demvsar2006statistical}. These tests are tailor-made for comparing classifiers across multiple datasets, especially when the conditions required for parametric tests are not satisfied. We utilized the following tests: ranking test, Friedman test, Nemenyi post hoc test, and win-tie loss sign test. Through the use of statistical tests, our goal is to thoroughly assess the performance of the models, allowing us to make well-rounded and impartial conclusions about their effectiveness. In the ranking scheme, each model receives a rank according to its performance on individual datasets, allowing for an evaluation of its overall performance. Higher ranks are attributed to the worst-performing models, while lower ranks are assigned to the best-performing models. To assess the performance of $q$ models across $N$ datasets, the rank of the $j^{th}$ model on the $i^{th}$ dataset can be denoted as $\mathfrak{R}_{j}^i$. Then the $j^{th}$ model's average rank is given by $\mathfrak{R}_j = \frac{1}{N} \sum_{i=1}^{N}\mathfrak{R}_j^i.$ The rank of the proposed MvRVFL-1 and MvRVFL-2 models along with the baseline SVM2K, MvTSVM, ELM-Vw-$A$, ELM-Vw-$B$, RVFL-Vw-$A$, RVFL-Vw-$B$, MVLDM, CDC, and MCMC models are $2.15$, $3.61$, $7.72$, $10.5$, $ 5.54$, $ 6.26$, $ 4.94$, $6.06$, $6.52$, $6.70$, and $6.00$, respectively. The average ranks of the models are presented in Table \ref{Average ACC and average rank for UCI and KEEL datasets}. The MvRVFL-1 model achieved an average rank of $2.15$, which is the lowest among all the models. While the proposed MvRVFL-2 attained the second position with an average rank of $3.61$. Since a lower rank indicates superior performance, the MvRVFL-1 and MvRVFL-2 models were identified as the highest-performing models. The Friedman test \cite{friedman1937use} compares whether significant differences exist among the models by comparing their average ranks. The Friedman test, a nonparametric statistical analysis, is utilized to compare the effectiveness of multiple models across diverse datasets. Under the null hypothesis, the models' average rank is equal, implying that they give equal performance. The Friedman test relies on the chi-squared distribution (\(\chi_F^2\)) with \((q - 1)\) degrees of freedom, and is computed using the following expression:
\[
\chi^2_F = \frac{12N}{q(q+1)}\left[\sum_{j}\mathfrak{R}_j^2 - \frac{q(q+1)^2}{4}\right].
\]
To further analyze the statistical significance, the corresponding \(F_F\) statistic is calculated as:
\[
F_F = \frac{(N - 1)\chi_F^2}{N(q - 1) - \chi_F^2},
\]
where the degrees of freedom for the F-distribution are \((q - 1)\) and \((N - 1)(q - 1)\). For the given values \(N = 27\) and \(q = 11\), we obtain \(\chi_F^2 = 112.68\) and \(F_F = 18.62\). From the $F$-distribution table, the value of $F_F(10, 260) = 1.87$ at a significance level of $5\%$. As $F_F > 1.87$, the null hypothesis is rejected. Hence, notable discrepancies are evident among the models. As a result, we assess the pairwise comparisons across the models using the Nemenyi post hoc test \cite{demvsar2006statistical}. C.D. = $q_\alpha \times \sqrt{\frac{q(q+1)}{6N}}$ is the critical difference (C.D.). Here, \(q_{\alpha}\) represents the critical value for the two-tailed Nemenyi test that is obtained from the distribution table. From the statistical F-distribution table, where $q_{\alpha} = 3.219$, the C.D. is computed as $2.9056$ at a $5\%$ significance level. The average rank disparities between the MvRVFL-1 and MvRVFL-2 models and the baseline SVM2K, MvTSVM, ELM-Vw-$A$, ELM-Vw-$B$, RVFL-Vw-$A$, RVFL-Vw-$B$, MVLDM, CDC, and MCMC models are $(5.57, 4.11)$, $(8.35, 6.89)$, $(3.39, 1.93)$, $(4.11, 2.65)$, $(2.79, 1.33)$, $(3.91, 2.45)$, $(4.37, 2.91)$, $(4.55, 3.09)$, and $(3.85, 2.39)$. The Nemenyi post hoc test confirms that the proposed MvRVFL-1 model demonstrates statistically significant superiority over SVM2K, MvTSVM, ELM-Vw-$A$, ELM-Vw-$B$, RVFL-Vw-$B$, MVLDM, CDC, and MCMC. The MvRVFL-2 model shows a statistically significant difference when compared to both the SVM2K, MvTSVM, and CDC models.

\begin{table*}[htp]
\centering
    \caption{Average Acc. and rank of the proposed MvRVFL-1 and MvRVFL-2 and the baseline models based on classification Acc. for UCI, KEEL, AwA, and Corel5K datasets.}
    \label{Average ACC and average rank for UCI and KEEL datasets}
   \resizebox{1\linewidth}{!}{
\begin{tabular}{llccccccccccc}
\hline
& Dataset & SVM2K \cite{farquhar2005two} & MvTSVM \cite{xie2015multi} & ELM-Vw-$A$ \cite{huang2006extreme} & ELM-Vw-$B$ \cite{huang2006extreme} & RVFL-Vw-$A$ \cite{pao1994learning}  & RVFL-Vw-$B$ \cite{pao1994learning}  & MVLDM \cite{hu2024multiview} & CDC \cite{kang2024cdc} & MCMC \cite{bian2025multilevel} & MvRVFL-1$^{\dagger}$ & MvRVFL-2$^{\dagger}$ \\ \hline
{UCI and KEEL}& Average Acc. & $76.65$ & $67.24$ & $82.66$ & $81.7$ & $83.14$ & $81.26$ & $79.9$ & $81.1$ & $81.98$ & \textbf{$85.15$} & $83.18$ \\ 
& Average Rank & $7.72$ & $10.5$ & $5.54$ & $6.26$ & $4.94$ & $6.06$ & $6.52$ & $6.7$ & $6$ & $2.15$ & $3.61$ \\ \hline
AwA & Average Acc. & $77.46$ & $64.31$ & $71.74$ & $75.99$ & $72.87$ & $77.46$ & $73.33$ & $77.66$ & $77.69$ & \textbf{$82.5$} & $77.97$ \\ 
& Average Rank & $5.1$ & $9.51$ & $8.36$ & $6.08$ & $7.59$ & $5.26$ & $7.1$ & $5.41$ & $5.34$ & $1.77$ & $4.49$ \\ \hline
Corel5k & Average Acc. & $74.87$ & $49.93$ & $74.98$ & $74.83$ & $76.33$ & $75.43$ & $69.87$ & $73.98$ & $73.47$ & \textbf{$78.01$} & $75.63$ \\  
& Average Rank & $5.57$ & $10.93$ & $5.4$ & $5.44$ & $4.47$ & $5.48$ & $7.55$ & $6.49$ & $6.82$ & $3.17$ & $4.68$ \\ \hline
\multicolumn{9}{l}{$^{\dagger}$ denotes the proposed models.}
\end{tabular}}
\end{table*}

\begin{table*}[htp]
    \centering
    \caption{The pairwise win-tie-loss comparison of the proposed MvRVFL-1 and MvRVFL-2 models, in comparison to baseline models, across the UCI and KEEL datasets.}
   \label{Pairwise Win-tie using linear kernel}
    \resizebox{1\textwidth}{!}{
    \begin{tabular}{lcccccccccc}
\hline
 & SVM2K \cite{farquhar2005two} & MvTSVM \cite{xie2015multi} & ELM-Vw-$A$ \cite{huang2006extreme} & ELM-Vw-$B$ \cite{huang2006extreme} & RVFL-Vw-$A$ \cite{pao1994learning} & RVFL-Vw-$B$ \cite{pao1994learning} & MVLDM \cite{hu2024multiview} & CDC \cite{kang2024cdc} & MCMC \cite{bian2025multilevel} & MvRVFL-1$^{\dagger}$ \\ 
\hline
MvTSVM \cite{xie2015multi} & {[}$2, 2, 23${]} &  &  &  &  &  &  &  &  &  \\
ELM-Vw-$A$ \cite{huang2006extreme} & {[}$18, 1, 8${]} & {[}$27, 0, 0${]} &  &  &  &  &  &  &  &  \\
ELM-Vw-$B$  \cite{huang2006extreme} & {[}$18, 2, 7${]} & {[}$26, 0, 1${]} & {[}$9, 4, 14${]} &  &  &  &  &  &  &  \\ 
RVFL-Vw-$A$ \cite{pao1994learning} & {[}$21, 1, 5${]} & {[}$27, 0, 0${]} & {[}$11, 10, 6${]} & {[}$15, 4, 8${]} &  &  &  &  &  &  \\
RVFL-Vw-$B$ \cite{pao1994learning} & {[}$18, 1, 8${]} & {[}$25, 1, 1${]} & {[}$9, 5, 13${]} & {[}$11, 9, 7${]} & {[}$7, 7, 13${]} &  &  &  &  \\ 
MVLDM \cite{hu2024multiview} & {[}$16, 1, 10${]} & {[}$24, 0, 3${]} & {[}$11, 0, 16${]} & {[}$12, 0, 15${]} & {[}$9, 0, 18${]} & {[}$13, 0, 14${]} &  &  &  &  \\
CDC \cite{kang2024cdc} &  {[}$19, 0, 8${]} & {[}$25, 0, 2${]} & {[}$10, 1, 16${]} & {[}$11, 0, 16${]} & {[}$10, 1, 16${]} & {[}$11, 0, 16${]} & {[}$13, 0, 14${]} &  &  &  \\
MCMC \cite{bian2025multilevel} &  {[}$18, 0, 9${]} & {[}$25, 0, 2${]} & {[}$12, 0, 15${]} & {[}$16, 0, 11${]} & {[}$9, 0, 18${]} & {[}$15, 0, 12${]} & {[}$15, 0, 12${]} & {[}$17, 0, 10${]} &  &  \\
MvRVFL-1$^{\dagger}$ & {[}$24, 2, 1${]} & {[}$27, 0, 0${]} & {[}$22, 1, 4${]} & {[}$24, 2, 1${]} & {[}$20, 4, 3${]} & {[}$24, 2, 1${]} & {[}$24, 0, 3${]} & {[}$26, 0, 1${]} & {[}$25, 0, 2${]} &  \\
MvRVFL-2$^{\dagger}$ & {[}$22, 1, 4${]} & {[}$25, 2, 0${]} & {[}$19, 1, 7${]} & {[}$20, 1, 6${]} & {[}$18, 2, 7${]} & {[}$18, 2, 7${]} & {[}$20, 1, 6${]} & {[}$22, 0, 5${]} & {[}$21, 0, 6${]} & {[}$8, 3, 16${]} \\ \hline
\multicolumn{9}{l}{$^{\dagger}$ represents the proposed models.}
\end{tabular}}
\end{table*}

We evaluate the models using the pairwise win-tie-loss sign test. This test implies that the two models perform similarly and are predicted to win in $N/2$ datasets, where $N$ is the total number of datasets, in accordance with the null hypothesis. A model is considered significantly better if it wins on approximately \(\frac{N}{2} + 1.96\frac{\sqrt{N}}{2}\) datasets. If the number of ties between the two models is even, they are split evenly between both models. If there is an odd number of ties, one tie is discarded, and the remaining ties are distributed equally among the classifiers. In this scenario, with \(N = 27\), there is a statistically significant difference between the two models if one of them receives at least 18.59 wins. Table \ref{Pairwise Win-tie using linear kernel} compares the performance of MvRVFL-1 and MvRVFL-2 with that of the baseline models. In Table \ref{Pairwise Win-tie using linear kernel}, the entry $[x, y, z]$ indicates the number of instances where the model in the row loses $z$ times, ties $y$ times, and wins $x$ times when evaluated against the model in the corresponding column. The proposed MvRVFL-2 model attains a statistically significant difference from the baseline models. The results show that the proposed MvRVFL-1 and MvRVFL-2 models significantly outperform the baseline models.

\subsection{Ablation Study on Fusion Strategies using UCI Datasets}
To rigorously evaluate the necessity and effectiveness of combining early and late fusion within the proposed MvRVFL framework, we conducted ablation studies on five benchmark UCI datasets: aus, breast\_cancer\_wisc, brwisconsin, cmc, and Ripley. Three distinct configurations were examined to isolate the impact of each fusion mechanism. In the Early Fusion Only variant, features from all available views were concatenated into a single composite feature vector prior to inputting them into the MvRVFL model. This configuration allows the network to jointly learn inter-view dependencies at the feature representation level. In contrast, the Late Fusion Only variant processes each view independently through a dedicated RVFL subnetwork. The individual prediction scores from these subnetworks are then aggregated at the decision level using a weighted averaging scheme, enhancing robustness against view-specific noise. Finally, the Combined Early + Late Fusion (MvRVFL) configuration integrates both mechanisms: multi-view features are first concatenated for early fusion and then processed by separate view-specific subnetworks, whose outputs are subsequently combined through late fusion. This dual strategy ensures that both cross-view correlations and view-specific discriminative information are preserved during training.

\begin{table}[ht]
\centering
\caption{Ablation study results (Acc. \%) of different fusion strategies on UCI datasets.}
\label{tab:ablation}
\begin{tabular}{lccc}
\hline
Dataset & Early Fusion Only & Late Fusion Only & Combined (MvRVFL) \\ \hline
aus                & 85.32 & 86.10 & \textbf{87.50} \\
breast\_cancer\_wisc         & 97.42 & 97.86 & \textbf{98.57} \\
brwisconsin               & 94.72 & 95.13 & \textbf{96.10} \\
cmc                      & 72.45 & \textbf{73.02} & 72.17 \\
Ripley                          & 88.11 & 88.45 & \textbf{89.87} \\ \hline
\end{tabular}
\end{table}

The ablation study results in Table~\ref{tab:ablation} highlight the role of fusion strategies in the proposed MvRVFL framework. On four out of the five UCI datasets, the Combined (MvRVFL) variant achieves the highest accuracy, outperforming both the early fusion and late fusion variants. This demonstrates the strength of integrating both fusion mechanisms, as early fusion captures inter-view correlations at the feature level, while late fusion enhances robustness by aggregating view-specific predictions. An exception is observed on the cmc dataset, where the Late Fusion Only strategy (73.02\%) slightly outperforms the combined variant (72.17\%). This suggests that in datasets with high variability and potential inter-view noise, late fusion may provide an advantage by focusing on more stable view-specific outputs. Overall, the combined strategy proves to be the most effective across the majority of datasets. These findings highlight that integrating early and late fusion offers a clear advantage in leveraging complementary multi-view information and enhances the robustness of the model across datasets with diverse characteristics.

\subsection{Experimental Scalability Results on UCI Datasets}
\begin{table}[ht]
\centering
\caption{Comparison of training time (in seconds) and memory usage (in MB) across different models. The best results are highlighted in bold.}
\label{tab:scalability_results}
\begin{tabular}{lcccccccc}
\hline
\multirow{2}{*}{Dataset} & \multicolumn{2}{c}{MVLDM \cite{hu2024multiview}} & \multicolumn{2}{c}{CDC \cite{kang2024cdc}} & \multicolumn{2}{c}{MCMC \cite{bian2025multilevel}} & \multicolumn{2}{c}{Proposed MvRVFL} \\ \cline{2-9}
 & Time & Mem. & Time & Mem. & Time & Mem. & Time & Mem. \\ \hline
aus       & 2.31 & 185 & 3.05 & 210 & 2.78 & 192 & \textbf{1.12} & \textbf{95} \\
breast\_cancer\_wisc   & 1.87 & 165 & 2.76 & 198 & 2.15 & 170 & \textbf{0.94} & \textbf{82} \\
brwisconsin             & 1.95 & 172 & 2.83 & 205 & 2.34 & 181 & \textbf{1.01} & \textbf{88} \\
cmc                    & 2.74 & 220 & 3.54 & 240 & 3.05 & 227 & \textbf{1.36} & \textbf{110} \\
Ripley                 & 2.12 & 178 & 2.97 & 201 & 2.43 & 184 & \textbf{1.08} & \textbf{92} \\ \hline
\end{tabular}
\end{table}
Table~\ref{tab:scalability_results} clearly demonstrates the computational advantages of the proposed MvRVFL framework. On all datasets, MvRVFL consistently achieves the lowest training time and memory usage, in some cases reducing runtime by nearly half compared to MVLDM and CDC. For example, on the cmc dataset, MvRVFL requires only 1.36 seconds and 110 MB of memory, compared to 3.54 seconds and 240 MB for CDC. Similarly, on the brwisconsin dataset, the proposed model is more than twice as fast as MVLDM while consuming nearly 50\% less memory. These results confirm that the closed-form output layer solution in MvRVFL effectively eliminates the computational overhead of iterative optimization methods used in CDC and MCMC. Moreover, the reduced memory footprint highlights its suitability for moderately large multi-view datasets. 
Overall, the proposed MvRVFL not only provides superior classification Acc. (as shown in Section \ref{Evaluation on UCI and KEEL Datasets}) but also scales efficiently in terms of time and memory, making it a practical choice for real-world multi-view learning tasks.

\subsection{Evaluation on AwA and Corel5k Dataset}
We evaluate the performance of the proposed MvRVFL model in comparison with baseline approaches on the AwA dataset. This dataset contains 30,475 images categorized into 50 animal classes, each described using six distinct pre-computed feature sets. For the purpose of our experiments, we focus on a subset comprising ten target classes, raccoon, leopard, Persian cat, chimpanzee, rat, giant panda, humpback whale, pig, hippopotamus, and seal, resulting in a total of 6,180 images. Among the available features, the 252-dimensional Histogram of Oriented Gradients (HOG) representation is labeled as Vw-$B$, while the 2000-dimensional $L_1$-normalized Speeded-Up Robust Features (SURF) are identified as Vw-$A$.
 For every pair of classes, we use the one-against-one approach to train $45$ binary classifiers. The Corel5k dataset contains $50$ categories, with each category including $100$ images that cover a range of semantic topics, such as bus, dinosaur, beach, and others. The $512$-dimensional GIST features are denoted as Vw-$A$, while the $100$-dimensional DenseHue features are denoted as Vw-$B$. In the experiments, we adopt a one-versus-rest approach for each category, training $50$ binary classifiers for the task. We randomly choose $100$ images from the other classes and include $100$ images from the target class in each binary dataset.

The effectiveness of the proposed MvRVFL models is evaluated against baseline methods by measuring classification accuracy under their respective optimal hyperparameter settings. These results are detailed in Tables S.2 and S.3 for the AwA and Corel5K datasets in the supplementary material, with the average accuracy summarized in Table \ref{Average ACC and average rank for UCI and KEEL datasets}. We can infer the following conclusions: Firstly, MvRVFL models achieve the highest mean accuracy, the lowest average ranking, and the most victories, demonstrating their exceptional performance. Secondly, although the MvRVFL model's performance is slightly inferior in some instances, it remains competitive, with results closely approaching the best outcomes. Thirdly, across most datasets, MvRVFL demonstrates higher accuracies compared to SVM-2K. This highlights the capability of MvRVFL models to effectively utilize two views by adhering to the coupling term that minimizes the product of the error variables for both views, leading to enhanced classification performance. In the supplementary material, we perform several sensitivity analyses of different aspects of the proposed models. This involves examining how the parameters $\mathcal{C}_1$ and $\mathcal{C}_2$ affect the proposed models in subsection S.I.A. We conduct experiments with varying numbers of training samples on the AwA dataset, as discussed in subsection S.I.B. Finally, we conduct the sensitivity of $\theta$ and $\rho$ discussed in subsection S.I.C.

\section{Conclusion and Future Work}
\label{Conclusion and Future Work}
This work proposes a novel multiview random vector functional link (MvRVFL) framework designed for the prediction of DBPs. The MvRVFL models proposed here effectively capture complex feature representations from the hidden layers of multiple views and function as a weighting network. It assigns weights to the features from all hidden layers, as well as to the original features obtained through direct connections. The coupling of different views in the MvRVFL models is achieved by incorporating the coupling term in the primal formulation of the model. 
The exceptional performance of MvRVFL in DBP prediction, when compared to baseline models, can largely be credited to the integration of features derived from protein sequences, such as PsePSSM, PSSM-DWT, NMBAC, MCD, and PSSM-AB. 
Furthermore, experiments were carried out on the UCI, KEEL, AwA, and Corel5K datasets. Statistical analyses confirm that the proposed MvRVFL models surpass baseline models in generalization performance. In future work, we plan to adapt the proposed model for addressing class-imbalance issues in multi-view settings (involving more than two views). Furthermore, we plan to enhance the methods for feature representation and create predictive models that more effectively integrate various features for improved synergy.

\subsubsection*{Acknowledgments}
This research is funded by the Science and Engineering Research Board (SERB) under the Mathematical Research Impact-Centric Support (MATRICS) scheme, Grant No. MTR/2021/000787. M. Sajid also expresses gratitude to the Council of Scientific and Industrial Research (CSIR), New Delhi, for awarding a fellowship through grants 09/1022(13847)/2022-EMR-I.

\bibliography{refs.bib}
\bibliographystyle{plainnat}

\appendix

\section{Background}
\label{Background}
In this section, we discuss the background of DNA-binding proteins, artificial neural networks, and multi-view learning in detail.

We begin by examining the significance and role of DNA-binding proteins, which are essential in numerous biological functions through their interactions with DNA sequences. Next, we explore the core concepts of artificial neural networks (ANNs), a class of machine learning models designed to mimic the functioning of the human brain, and examine their architecture, learning mechanisms, and applications. Finally, we cover the concept of multi-view learning, a technique that integrates multiple sets of features or perspectives to improve the performance of the predictive model.
\subsection{DNA binding protein}
A DNA-binding protein (DBP) is a protein that directly engages with DNA via its specific binding domain. These proteins are among the most prevalent intracellular proteins and are vital for regulating genomic functions, including transcription, DNA replication, and repair processes \cite{zimmermann2020evaluation}. The significant involvement of DBPs in various cellular activities highlights their critical importance. With the ongoing discovery of new proteins in the post-genomic era, identifying DNA-binding proteins (DBPs) within the vast array of newly uncovered sequences has become a key focus. How can DBPs be accurately recognized amidst the large volume of novel proteins? Experimental detection methods, which are resource-intensive, are not as cost-effective as computational approaches. Consequently, there has been a surge in the creation of various computational models in recent years aimed at predicting DNA-binding proteins (DBPs). \citet{lin2011idna} presented iDNA-Prot, a model that extracts features using pseudo amino acid composition (PseACC) \cite{fu2019full} and employs a random forest (RF) classifier for prediction. Subsequently, Liu et al. developed three successive predictors: PseDNA-Pro \cite{liu2015psedna}, iDNAPro-PseAAC \cite{liu2015dna}, and iDNA-Prot|dis \cite{liu2014idna}. These predictors utilized features extracted through different algorithms, integrating them into a single feature set that was then used as input for an SVM model to make predictions. Similarly, StackPDB predicts DNA-binding proteins through a three-step process: feature extraction, feature selection, and model construction are the key stages. StackPDB extracts features from protein sequences based on amino acid composition and evolutionary information. The position-specific scoring matrix (PSSM), produced by the PSI-BLAST tool \cite{altschul1997gapped}, is used to encapsulate evolutionary information. In the StackPDB approach, features such as PsePSSM, PSSM-TPC, EDT, and RPT are extracted from the PSSM. These features are then processed using extreme gradient boosting combined with recursive feature elimination to identify the most relevant features. The optimal feature set is subsequently fed into a stacked ensemble classifier composed of XGBoost, LightGBM, and SVM. Prior research \cite{ding2017identification, ding2019protein} has shown that protein sequences can be depicted through various methods, such as amino acid composition and the position-specific scoring matrix (PSSM). Fusion techniques, which combine multiple representations, are increasingly used in DBP identification to enhance model performance.

Multiple kernel learning (MKL) is a commonly applied early fusion method that seeks to fine-tune the weights of different kernels. By linearly combining multiple base kernels with their corresponding weights, MKL constructs an optimal kernel. The CKA-MKL method \cite{qian2021sequence} seeks to enhance the cosine similarity between the ideal and optimal kernels. Furthermore, a Laplacian term is added to the objective function to mitigate extreme scenarios. However, CKA-MKL primarily focuses on global kernel alignment and does not consider the differences in local samples. To address this, HKAM-MKL \cite{zhao2022hkam} improves upon CKA-MKL by optimizing both global and local kernel alignment scores. Both HKAM-MKL and CKA-MKL employ SVM classifiers. In contrast, HSIC-MKL \cite{qian2022identification} aims to enhance the independence between samples and their corresponding labels within the Reproducing Kernel Hilbert Space (RKHS). Unlike CKA-MKL, which focuses on global alignment, and HKAM-MKL, which integrates both global and local considerations, HSIC-MKL emphasizes sample-label independence. In DNA-binding protein prediction, HKAM-MKL has shown superior performance compared to CKA-MKL. MLapSVM-LBS \cite{sun2022mlapsvm}, a different MKL method, integrates multiple types of information during the training phase using a local behavior similarity graph as a regularization term. Given its non-convex objective function, an alternating algorithm is used. MLapSVM-LBS stands out by offering flexibility in combining various information sources during training while allowing different views to be modeled independently.

\subsection{Artificial neural networks}
Artificial neural networks (ANNs) are machine learning systems designed to emulate the structure and operations of the human brain's neural network. These networks consist of interconnected units, or neurons, which perform mathematical computations to process and relay information. ANNs are engineered to identify patterns and correlations within data, leveraging this acquired knowledge to make predictions. ANNs have showcased efficacy across diverse domains, including brain age prediction \cite{tanveer2023deep}, fault diagnosis of drilling pumps \cite{guo2024parallel}, detection of sickle cell disease \cite{goswami2024detection}, rainfall forecasting \cite{luk2001application}, diagnosis of Alzheimer’s disease \cite{tanveer2024ensemble, ganaie2024graph}, and so on. 
Conventional ANNs rely on gradient descent (GD) based iterative methods, which present several inherent challenges in parameter calculation. These include a tendency to converge to local rather than global optima, heightened sensitivity to the choice of learning rate and initial parameters, and a sluggish convergence rate.

To circumvent the limitations of GD-based neural networks, randomized neural network (RNN) \cite{schmidt1992feed} is proposed. Only the output layer's parameters are determined during training using a closed-form solution in RNNs, while other network parameters remain constant \cite{suganthan2021origins}. The parameters of the hidden layer of the random vector functional link (RVFL) neural network \cite{pao1994learning, malik2023random} are randomly initialised and do not change throughout the training process, making it a shallow feed-forward model. The direct connections between the input and output layers make RVFL differ from other RNNs. The direct connections serve as an implicit regularization mechanism \cite{zhang2016comprehensive, zhang2020new}, improving the model's capacity to learn. To optimize output parameters, RVFL uses approaches like pseudo-inverse or least squares to obtain a closed-form solution, leading to a more efficient learning process with a reduced number of parameters to fine-tune. For a comprehensive understanding, readers can refer to the detailed review on RVFL \cite{malik2023random}.

\subsection{Multiview learning}
Multiview learning (MVL) is a key area of research that significantly enhances generalization performance in various tasks by combining multiple feature sets, each providing complementary insights \cite{zhao2017multi, xu2017re, tang2020cgd}. MVL emerges in response to the common occurrence of diverse types of data in practical scenarios. Consider an image, which can be characterized by its color or texture properties, and an individual, whose identity can be identified through facial features or fingerprints. In real-world situations, samples from different viewpoints may belong to distinct spaces or display considerable differences in their distributions because of the large variations between perspectives. Traditional approaches tackle such data using a cascade method \cite{xu2017re}, transforming multiview data into a unified single-view format by merging various feature spaces into a common one. However, this method overlooks the unique statistical characteristics of each view and is hindered by the challenges of high dimensionality. Multiview learning (MVL) techniques \cite{tang2019cross, huang2020partially} are applied to a variety of tasks, such as transfer learning \cite{zhao2017consistent}, clustering \cite{wen2018incomplete, wen2020adaptive}, dimensionality reduction \cite{wang2013new, hu2019multi}, and classification \cite{sun2018multiview}. The SVM-2K model, a dual-view SVM approach integrating support vector machines with the distance minimization form of kernel canonical correlation analysis (KCCA), was initially proposed by \cite{farquhar2005two}. Using the consensus principle, this method makes use of two points of view. Multiview twin SVM (MvTSVM) \cite{xie2015multi} represents the initial endeavor to integrate a best-fitting hyperplane classifier with MVL. Recently, various Models of MvTSVM have been introduced, including multiview restricted kernel machine (MVRKM) \cite{houthuys2021tensor}, multiview learning with twin parametric margin SVM (MvTPMSVM) \cite{quadir2024multiview}, and multiview large margin distribution machine (MVLDM) \cite{hu2024multiview}.

\section{Random Vector Functional Link (RVFL) Network}
\label{Random Vector Functional Link (RVFL) Network}
The RVFL \cite{pao1994learning} is a type of feed-forward neural network with three layers: the input layer, a hidden layer, and an output layer. The links between the input and hidden layers, along with the hidden layer biases, are randomly set at the start and stay fixed throughout the training process. The original features of input samples are also directly connected to the output layer. Analytical methods such as the Moore-Penrose inverse or the least squares approach are used to determine the output layer weights. The structure of the RVFL model is illustrated in Fig. \ref{Geometrical structure of RVFL model}.
\begin{figure}[ht!]
    \centering
    \includegraphics[width=0.8\textwidth,height=7cm]{RVFL.png}
    \caption{Geometrical structure of RVFL model}
    \label{Geometrical structure of RVFL model}
\end{figure}

Let $T=\{(x_i, y_i), i=1, 2, \dots, n\}$ denote the training dataset, where $y_i \in \{+1, -1\}$ indicates the label of $x_i \in \mathbb{R}^{1 \times m}$. Let $X=(x_1^t, x_2^t, \ldots, x_n^t)^t \in \mathbb{R}^{n \times m}$ and $Y=(y_1^t, y_2^t, \ldots, y_n^t)^t \in \mathbb{R}^{n \times 2}$ be the collection of all input and target vectors, respectively. Randomly initializing weights and biases on the input matrix yields the hidden layer matrix \( H_1 \), followed by a nonlinear activation function \( \phi \). It is expressed as:  
\begin{align}
    H_1  = \phi(XW_1 + b_1) \in \mathbb{R}^{n \times h_l},
\end{align}
where $W_1 \in \mathbb{R}^{m \times h_l}$ represents the weights vector which is initialized randomly, and drawn from a uniform distribution spanning $[-1, 1]$, and $b_1 \in \mathbb{R}^{n \times h_l}$ is the bias matrix. Thus, $H_1$ is given as:
\begin{align}
   H_1= \begin{bmatrix}
      \phi(x_1w_1+b^{(1)}) & \ldots & \phi(x_1w_{h_l}+b^{(h_l)}) \\
     \vdots & \vdots & \vdots \\
      \phi(x_nw_n+b^{(1)}) & \ldots & \phi(x_nw_{h_l}+b^{(h_l)})
    \end{bmatrix},
\end{align}
here $w_k \in \mathbb{R}^{m \times 1}$ represents the $k^{th}$ column vector of the weights matrix of $W_1$, $x_i \in \mathbb{R}^{1 \times m}$ denotes the $i^{th}$ sample of matrix $X$ and $b^{(j)}$ signifies the bias term of the $j^{th}$ hidden node. The weights of the output layer are determined through the following matrix equation:
\begin{align}
\label{eq:3}
    \begin{bmatrix}
            X & H_1
        \end{bmatrix}W_2 = \hat{Y}.
\end{align}
Here, \( W_2 \in \mathbb{R}^{(m + h_l) \times 2} \) denotes the weight matrix that connects both the input and the concatenated hidden nodes to the output layer, while \( \hat{Y} \) denotes the predicted output. The optimization problem derived from Eq. (\ref{eq:3}) can be formulated as:

\begin{align}
\label{eq:4}
    (W_2)_{min} = \underset{W_2}{\arg\min} \frac{\mathcal{C}}{2}\|H_2W_2 - Y\|^2 + \frac{1}{2}\|W_2\|^2,
\end{align}
where $H_2 = \begin{bmatrix}
            X & H_1
        \end{bmatrix}$. The optimal solution of Eq. (\ref{eq:4}) is defined as follows:
\begin{align}
    (W_2)_{min}=\left\{\begin{array}{ll}\left({H_2}^t {H_2}+\frac{1}{\mathcal{C}} I\right)^{-1} {H_2}^{t} {Y}, & (m+h_l) \leq n,\vspace{3mm} \\ 
{H_2}^t\left({H_2} {H_2}^t+\frac{1}{\mathcal{C}} I\right)^{-1} {Y}, & n<(m+h_l), \end{array}\right.
\end{align}
where $\mathcal{C}>0$ is a tunable parameter and $I$ represents the identity matrix of conformal dimensions.

\clearpage
\section*{Supplementary Material}

\renewcommand{\thesection}{S.I}
\section{Sensitivity Analysis}
\label{Sensitivity Apendix}
In this section, we perform the sensitivity analysis of several key hyperparameters of the proposed MvRVFL model. These analyses covered various factors, including hyperparameters $\mathcal{C}_1$ and $\mathcal{C}_2$ discussed in subsection \ref{Effect of parameter C1 and C2}. Performance with different numbers of training samples on the AwA dataset is discussed in subsection \ref{Performance with different number of training samples on AwA dataset}. Finally, we conduct the sensitivity of $\theta$ and $\rho$ discussed in subsection \ref{Sensitivity of theta and rho}.  

\subsection{Effect of parameter \texorpdfstring{$\mathcal{C}_1$}{C1} and \texorpdfstring{$\mathcal{C}_2$}{C2} on the performance of the proposed MvRVFL model on AwA dataset}
\label{Effect of parameter C1 and C2}
The performance of the proposed MvRVFL model is assessed by adjusting the values of $\mathcal{C}_1$ and $\mathcal{C}_2$. This thorough analysis helps us pinpoint the configuration that enhances predictive accuracy and improves the model's robustness against new data samples. Fig. \ref{effect of parameter C1 and C2} illustrates significant variations in the model's accuracy across different values of $\mathcal{C}_1$ and $\mathcal{C}_2$, underscoring the model's sensitivity to these specific hyperparameters.

According to the findings presented in Fig. \ref{effect of parameter C1 and C2}, optimal performance of the proposed model is observed within the $\mathcal{C}_1$ and $\mathcal{C}_2$ ranges of $10^{-4}$ to $10^{4}$. These findings indicate that both $\mathcal{C}_1$ and $\mathcal{C}_2$ significantly impact the model's performance. Hence, it is advisable to meticulously select the hyperparameters $\mathcal{C}_1$ and $\mathcal{C}_2$ in the MvRVFL model to achieve superior generalization performance.

\renewcommand{\thefigure}{S.1}
\begin{figure*}[htp]
\begin{minipage}{.246\linewidth}
\centering
\subfloat[Chimpanzee vs Giant panda]{\label{fig:2a}\includegraphics[scale=0.20]{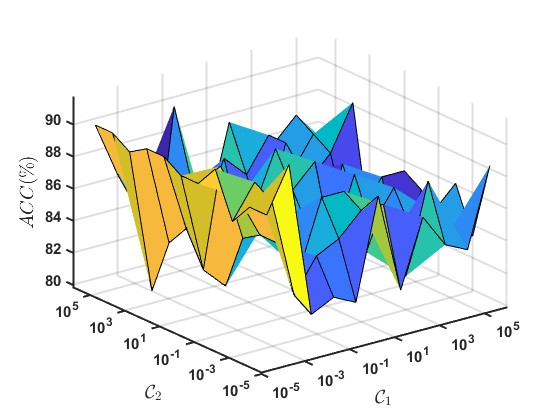}}
\end{minipage}
\begin{minipage}{.246\linewidth}
\centering
\subfloat[Giant panda vs Leopard]{\label{fig:2b}\includegraphics[scale=0.20]{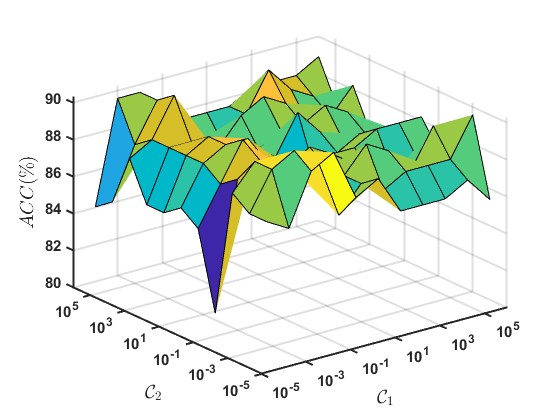}}
\end{minipage}
\begin{minipage}{.246\linewidth}
\centering
\subfloat[Persian cat vs Rat
]{\label{fig:2c}\includegraphics[scale=0.20]{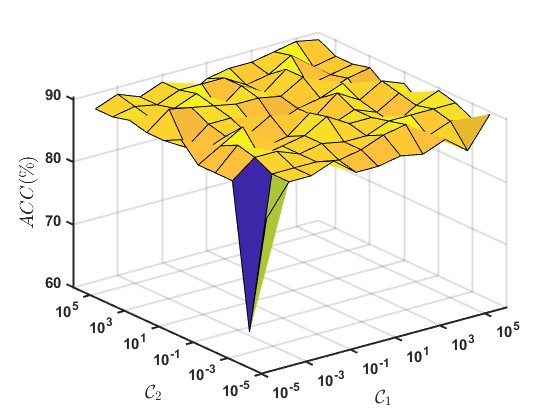}}
\end{minipage}
\begin{minipage}{.246\linewidth}
\centering
\subfloat[Giant panda vs Hippopotamus]{\label{fig:2d}\includegraphics[scale=0.20]{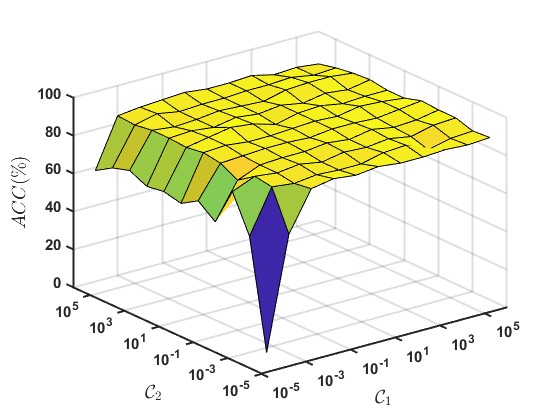}}
\end{minipage}
\caption{Effect of parameter $\mathcal{C}_1$ and $\mathcal{C}_2$ on the performance of the proposed MvRVFL model on AwA dataset.}
\label{effect of parameter C1 and C2}
\end{figure*}

\subsection{Performance with different number of training samples on AwA dataset}
\label{Performance with different number of training samples on AwA dataset}
We assess how the proposed MvRVFL model's performance varies with different numbers of training samples. Fig. \ref{Performance with different number of training samples} shows how the Acc. changes as the number of training samples ranges from $86$ to $336$. The x-axis displays the number of training samples, and the y-axis shows the corresponding Acc. values. It is observed that the Acc. value generally increases with the rise in the number of training samples. This is because an increase in training samples provides more data for the model to learn from, leading to improved accuracy in the classification results.

\renewcommand{\thefigure}{S.2}
\begin{figure*}[htp]
\begin{minipage}{.246\linewidth}
\centering
\subfloat[Chimpanzee vs \\ Giant panda]{\label{fig:3a}\includegraphics[scale=0.24]{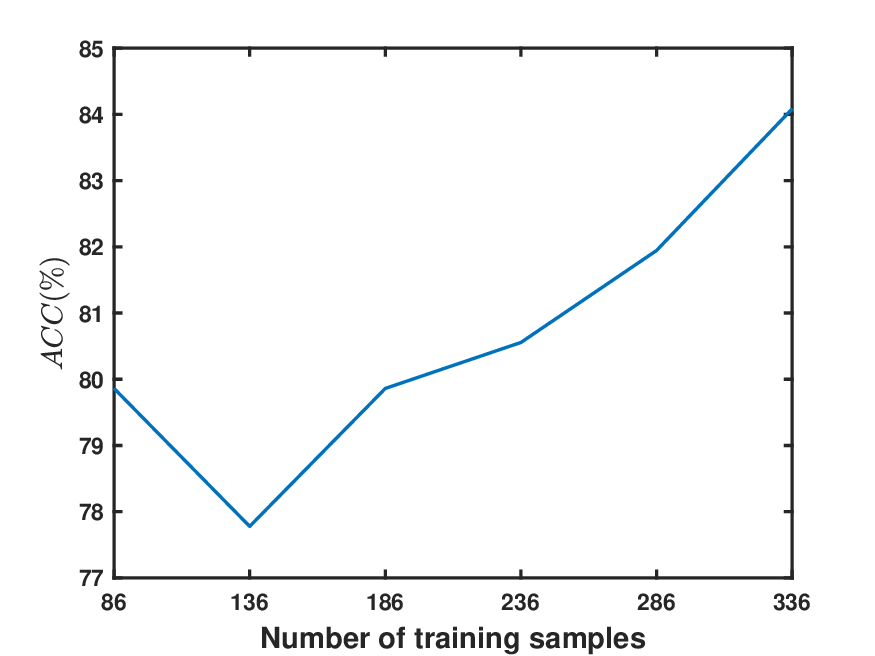}}
\end{minipage}
\begin{minipage}{.246\linewidth}
\centering
\subfloat[Chimpanzee panda vs \\ Hippopotamus]{\label{fig:3b}\includegraphics[scale=0.24]{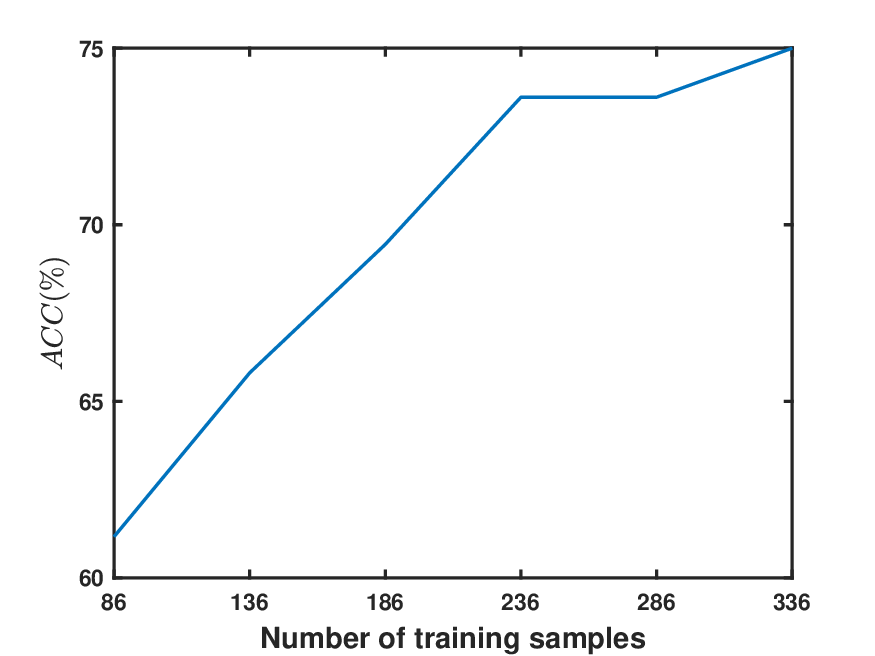}}
\end{minipage}
\begin{minipage}{.246\linewidth}
\centering
\subfloat[Chimpanzee vs \\ Seal
]{\label{fig:3c}\includegraphics[scale=0.24]{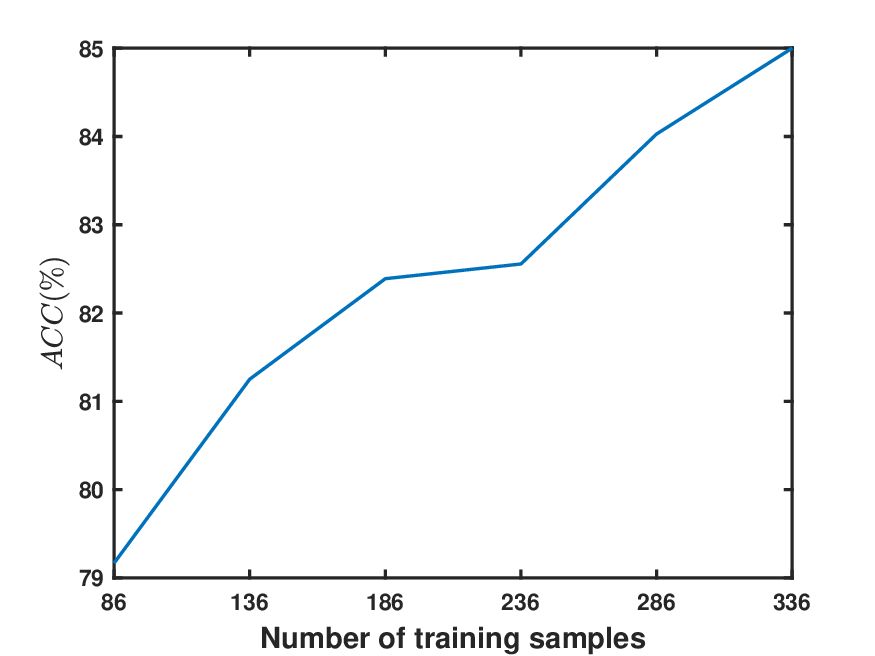}}
\end{minipage}
\begin{minipage}{.246\linewidth}
\centering
\subfloat[Giant panda vs \\ Leopard]{\label{fig:3d}\includegraphics[scale=0.24]{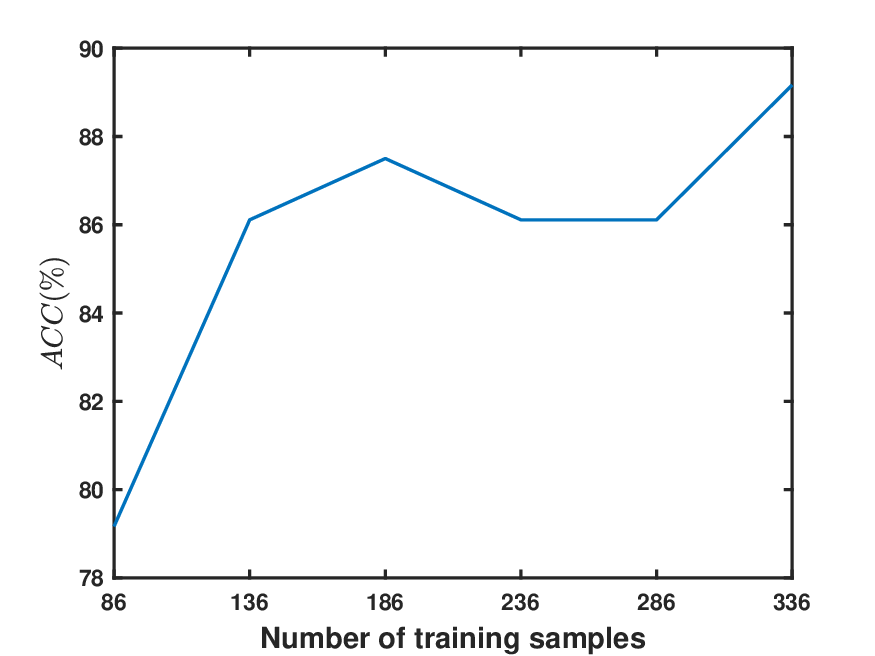}}
\end{minipage}
\caption{Performance with different number of training samples of the proposed MvRVFL model on AwA dataset.}
\label{Performance with different number of training samples}
\end{figure*}

\subsection{Sensitivity of \texorpdfstring{$\theta$}{theta} and \texorpdfstring{$\rho$}{rho} combinations}
\label{Sensitivity of theta and rho}
We explore the sensitivity of the MvRVFL model to various values of parameters $\theta$ and $\rho$. The parameter $\theta$ regulates the gap between two views, while $\rho$ is linked to the coupling term $\xi_1^t\xi_2$. We vary the values of $\theta$ and $\rho$ from $10^{-5}$ to $10^{5}$ and record the corresponding Acc. results. With other parameters held constant at their optimal settings, Fig. \ref{Performance variation with different values of parameters} illustrates how the performance of MvRVFL changes as the values of parameters $\theta$ and $\rho$ are varied.

From the perspective of hyperparameters, $\theta$ and $\rho$, the parameter $\theta$ regulates the gap between view $A$ and view $B$. With the same value of $\rho$, when $\theta > 10^{-1}$, it indicates that view $B$ plays a more significant role than view $A$ in learning the overall model. Otherwise, view $A$ is more important. For instance, on the 1000 and 10000 sub-datasets, the Acc. reaches its highest value when the parameter $\theta$ is relatively large (e.g., $10^{3}$ or $10^{5}$), suggesting that view $B$ holds greater importance than view $A$. Furthermore, on the 103000, and 143000 sub-datasets, the optimal performance is achieved when the parameter $\theta$ is small (e.g., $10^{-5}$ or $10^{-3}$), indicating that view $A$ holds more significance than view $B$.

\renewcommand{\thefigure}{S.3}
\begin{figure*}[htp]
\begin{minipage}{.246\linewidth}
\centering
\subfloat[1000]{\label{fig:4a}\includegraphics[scale=0.22]{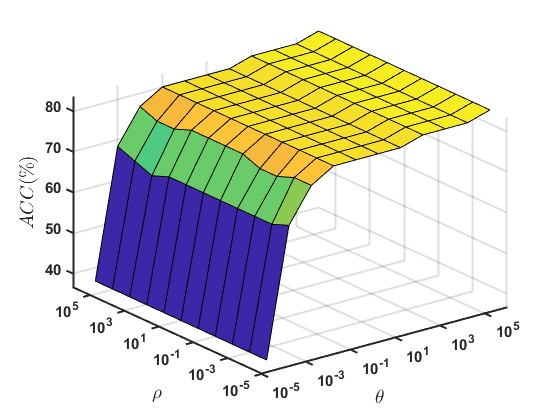}}
\end{minipage}
\begin{minipage}{.246\linewidth}
\centering
\subfloat[10000]{\label{fig:4b}\includegraphics[scale=0.22]{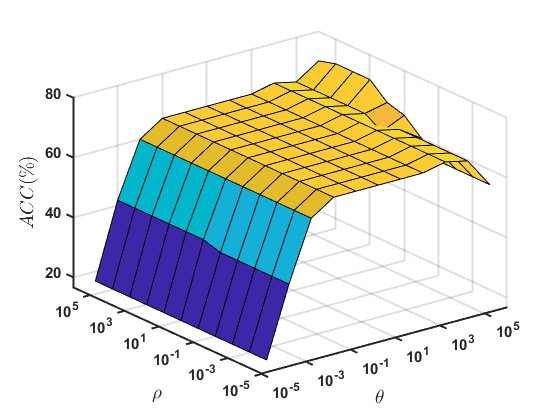}}
\end{minipage}
\begin{minipage}{.246\linewidth}
\centering
\subfloat[103000]{\label{fig:4c}\includegraphics[scale=0.22]{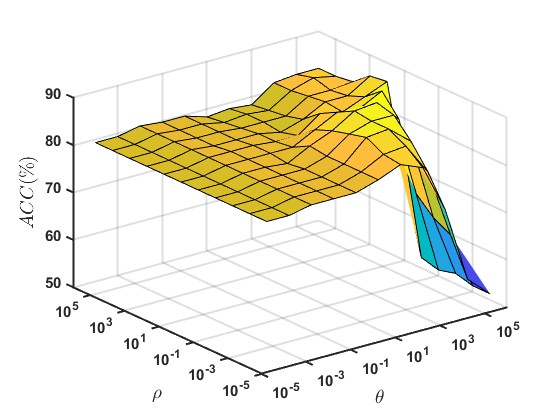}}
\end{minipage}
\begin{minipage}{.246\linewidth}
\centering
\subfloat[143000]{\label{fig:4d}\includegraphics[scale=0.22]{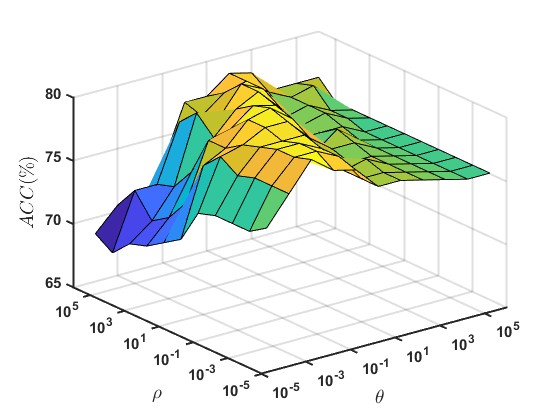}}
\end{minipage}
\caption{Performance variation with different values of parameters $\theta$ and $\rho$ of the proposed MvRVFL model on Corel5k dataset.}
\label{Performance variation with different values of parameters}
\end{figure*}

\renewcommand{\thesection}{S.II}
\section{Classification accuracy tables of the proposed MvRVFL models along with baseline models on UCI, KEEL, AwA and Corel5k datasets}
In this section, we present the performance of the proposed MvRVFL models, along with the baseline models.

\renewcommand{\thetable}{S.I}
\begin{table*}[htp]
\centering
    \caption{Performance comparison of the proposed MvRVFL along with the baseline models based on classification Acc. for UCI and KEEL datasets.}
    \label{Classification performance in Linear Case.}
    \resizebox{1.1\linewidth}{!}{
\begin{tabular}{lccccccccccc}
\hline
Dataset & SVM2K \cite{farquhar2005two} & MvTSVM \cite{xie2015multi}   & ELM-Vw-$A$ \cite{huang2006extreme} & ELM-Vw-$B$ \cite{huang2006extreme} & RVFL-Vw-$A$ \cite{pao1994learning} & RVFL-Vw-$B$ \cite{pao1994learning} & MVLDM \cite{hu2024multiview} & CDC \cite{kang2024cdc} & MCMC \cite{bian2025multilevel} & MvRVFL-1$^{\dagger}$ & MvRVFL-2$^{\dagger}$ \\ 
& $(C_1)$ & $(C_1, C_2, D)$ & $(C, N)$ & $(C, N)$ & $(C, N)$ & $(C, N)$ & $(C_1, \nu_1, \nu_2, \sigma)$ & $(\alpha, \beta, m)$ & $(\lambda_1, \lambda_2, \lambda_3, \zeta)$ & $(C_1, C_2, \rho, N)$ & $(C_1, C_2, \rho, N)$ \\
\hline
aus & $87.02$ & $71.15$ & $85.98$ & $86.06$ & $85.98$ & $85.98$ & $71.98$ & $73.62$ & $85.65$ & $87.5$ & $87.98$ \\
 & $(0.001)$ & $(0.00001, 0.00001, 0.00001)$ & $(0.01, 123)$ & $(0.001, 163)$ & $(0.001, 163)$ & $(10, 23)$ & $(0.25, 0.25, 0.25, 0.25)$ & $(0.0001, 10, 100)$ & $(0.01, 1, 0.0001, 0.001)$ & $(1, 0.0001, 0.00001, 63)$ & $(100, 100000, 1, 43)$ \\
bank & $80.74$ & $71.86$ & $85.54$ & $85.39$ & $85.83$ & $85.61$ & $73.67$ & $87.19$ & $88.53$ & $90.05$ & $89.24$ \\
 & $(0.001)$ & $(0.00001, 0.00001, 0.00001)$ & $(0.00001, 3)$ & $(10000, 23)$ & $(1000, 163)$ & $(10, 123)$ & $(0.25, 0.25, 0.25, 4)$ & $(1, 10000, 30)$ & $(2.5, 0.1, 0.00001, 0.0001)$ & $(100, 0.1, 0.0001, 203)$ & $(100, 1000, 0.0001, 83)$ \\
breast\_cancer & $62.45$ & $55.58$ & $69.77$ & $65.12$ & $67.44$ & $66.28$ & $70$ & $73.15$ & $71.4$ & $70.93$ & $75.58$ \\
 & $(0.001)$ & $(0.00001, 0.00001, 0.00001)$ & $(1, 83)$ & $(0.01, 43)$ & $(0.0001, 3)$ & $(1, 83)$ & $(0.25, 0.5, 0.5, 4)$ & $(10, 1000, 50)$ & $(5, 1, 0.000001, 0.00001)$ & $(0.1, 0.0001, 0.00001, 43)$ & $(0.0001, 0.01, 0.0001, 43)$ \\
breast\_cancer\_wisc & $90.04$ & $81.43$ & $92.1$ & $92.1$ & $92.14$ & $95.1$ & $75$ & $93.65$ & $95.22$ & $98.57$ & $96.19$ \\
 & $(0.00001)$ & $(0.00001, 0.00001, 0.00001)$ & $(0.0001, 203)$ & $(0.01, 63)$ & $(0.00001, 23)$ & $(0.0001, 203)$ & $(0.5, 0.25, 0.25, 0.25)$ & $(1000, 10, 10)$ & $(0.1, 0.01, 0.001, 0.001)$ & $(0.00001, 0.1, 0.00001, 163)$ & $(100, 0.001, 10, 23)$ \\
breast\_cancer\_wisc\_diag & $95.49$ & $88.6$ & $92.42$ & $97.08$ & $95.83$ & $96.49$ & $93.15$ & $95.82$ & $94.44$ & $98.25$ & $97.08$ \\
 & $(0.001)$ & $(0.00001, 0.00001, 0.00001)$ & $(0.01, 163)$ & $(0.001, 163)$ & $(0.001, 163)$ & $(100, 3)$ & $(0.25, 0.25, 1, 4)$ & $(10000, 1, 100)$ & $(1, 2.5, 0.0001, 0.0001)$ & $(0.001, 0.0001, 0.0001, 123)$ & $(0.1, 10, 0.0001, 83)$ \\
breast\_cancer\_wisc\_prog & $58.33$ & $58.33$ & $68.33$ & $68.33$ & $73.33$ & $68.33$ & $71.17$ & $70.63$ & $72.09$ & $73.33$ & $73.33$ \\
 & $(0.001)$ & $(0.00001, 0.00001, 0.00001)$ & $(0.01, 203)$ & $(0.01, 83)$ & $(100000, 3)$ & $(0.01, 203)$ & $(0.25, 0.25, 4, 2)$ & $(0.0001, 0.0001, 70)$ & $(10, 5, 0.000001, 0.000001)$ & $(0.00001, 1000, 0.01, 23)$ & $(10000, 100, 100, 3)$ \\
brwisconsin & $97.56$ & $61.95$ & $97.07$ & $96.1$ & $97.07$ & $96.59$ & $95.59$ & $95.34$ & $92.74$ & $96.1$ & $97.56$ \\
 & $(0.00001)$ & $(0.00001, 0.00001, 0.00001)$ & $(0.0001, 143)$ & $(0.001, 63)$ & $(0.001, 63)$ & $(100000, 23)$ & $(0.25, 0.25, 2, 4)$ & $(1, 10, 100)$ & $(0.01, 10, 0.00001, 0.00001)$ & $(0.00001, 1, 0.001, 163)$ & $(0.0001, 1, 0.00001, 203)$ \\
bupa or liver-disorders & $54.8$ & $42.31$ & $63.46$ & $65.38$ & $63.46$ & $66.35$ & $55.34$ & $65.07$ & $66.14$ & $70.19$ & $64.42$ \\
 & $(0.00001)$ & $(0.00001, 0.00001, 0.00001)$ & $(1000, 23)$ & $(0.1, 163)$ & $(0.1, 163)$ & $(0.1, 163)$ & $(0.25, 0.25, 0.25, 4)$ & $(10, 0.0001, 30)$ & $(2.5, 2.5, 0.0001, 0.001)$ & $(10, 0.01, 0.001, 23)$ & $(0.1, 1, 0.0001, 83)$ \\
checkerboard\_Data & $87.02$ & $43.75$ & $86.98$ & $86.06$ & $86.98$ & $86.98$ & $84.06$ & $82.83$ & $85.03$ & $87.5$ & $87.98$ \\
 & $(0.00001)$ & $(0.00001, 0.00001, 0.00001)$ & $(0.01, 123)$ & $(0.001, 163)$ & $(0.001, 163)$ & $(10, 23)$ & $(0.25, 0.25, 0.25, 1)$ & $(1000, 10000, 50)$ & $(5, 0.1, 0.001, 0.0001)$ & $(1, 0.0001, 0.00001, 63)$ & $(100, 100000, 1, 43)$ \\
chess\_krvkp & $80.45$ & $82.35$ & $95.62$ & $93.33$ & $95.62$ & $94.68$ & $97.7$ & $90.46$ & $94.85$ & $96.77$ & $94.58$ \\
 & $(0.001)$ & $(0.00001, 0.00001, 0.00001)$ & $(100, 203)$ & $(100000, 203)$ & $(10000, 203)$ & $(100, 203)$ & $(0.25, 0.25, 0.25, 4)$ & $(10000, 10, 10)$ & $(0.1, 5, 0.0001, 0.000001)$ & $(0.1, 0.0001, 0.1, 203)$ & $(100000, 0.00001, 10, 203)$ \\
cleve & $80$ & $75.56$ & $80$ & $85.56$ & $81.11$ & $81.11$ & $84.27$ & $82.73$ & $80.59$ & $84.44$ & $81.11$ \\
 & $(0.1)$ & $(0.00001, 0.00001, 0.00001)$ & $(0.001, 103)$ & $(0.001, 123)$ & $(0.01, 3)$ & $(0.01, 23)$ & $(0.25, 0.25, 0.25, 1)$ & $(0.0001, 1000, 100)$ & $(1, 0.1, 0.000001, 0.00001)$ & $(0.00001, 1, 0.00001, 203)$ & $(0.0001, 0.1, 0.00001, 123)$ \\
cmc & $64.25$ & $55.88$ & $69.91$ & $70.14$ & $68.1$ & $71.72$ & $74.38$ & $69.87$ & $67.91$ & $72.17$ & $70.36$ \\
 & $(0.00001)$ & $(0.00001, 0.00001, 0.00001)$ & $(0.01, 143)$ & $(0.01, 183)$ & $(100, 63)$ & $(1000, 43)$ & $(0.25, 0.25, 0.25, 4)$ & $(1, 0.0001, 50)$ & $(10, 1, 0.00001, 0.001)$ & $(1000, 1, 0.001, 23)$ & $(10000, 0.00001, 0.0001, 23)$ \\
conn\_bench\_sonar\_mines\_rocks & $80.95$ & $46.03$ & $80.54$ & $74.6$ & $88.54$ & $73.02$ & $75.81$ & $78.84$ & $75.77$ & $80.95$ & $66.67$ \\
 & $(1000)$ & $(0.00001, 0.00001, 0.00001)$ & $(10, 3)$ & $(0.01, 203)$ & $(0.01, 203)$ & $(0.1, 183)$ & $(0.25, 0.25, 4, 0.5)$ & $(10, 10, 70)$ & $(0.01, 0.01, 0.000001, 0.0001)$ & $(10, 1000, 0.1, 183)$ & $(10, 1000, 10, 163)$ \\
cylinder\_bands & $68.18$ & $60.39$ & $76.62$ & $74.03$ & $72.08$ & $74.68$ & $71.9$ & $72.45$ & $75.18$ & $74.68$ & $78.57$ \\
 & $(0.00001)$ & $(0.00001, 0.00001, 0.00001)$ & $(0.001, 203)$ & $(100, 143)$ & $(0.001, 183)$ & $(1, 23)$ & $(0.25, 4, 0.25, 4)$ & $(1000, 0.0001, 30)$ & $(2.5, 10, 0.000001, 0.00001)$ & $(100, 0.01, 0.0001, 43)$ & $(10, 1000, 0.01, 43)$ \\
fertility & $75.25$ & $75$ & $90$ & $90$ & $90$ & $90$ & $86.67$ & $86$ & $92$ & $93.33$ & $86.67$ \\
 & $(0.01)$ & $(0.00001, 0.00001, 0.00001)$ & $(0.01, 83)$ & $(0.01, 63)$ & $(0.01, 63)$ & $(0.01, 83)$ & $(0.25, 0.25, 2, 2)$ & $(10000, 1000, 50)$ & $(5, 2.5, 0.00001, 0.0001)$ & $(0.01, 1, 0.01, 183)$ & $(0.0001, 10000, 1, 63)$ \\
hepatitis & $80.85$ & $78.72$ & $78.85$ & $80.85$ & $80.85$ & $78.72$ & $78.26$ & $75.82$ & $74.19$ & $85.11$ & $78.72$ \\
 & $(0.00001)$ & $(0.00001, 0.00001, 0.00001)$ & $(1, 23)$ & $(0.01, 183)$ & $(0.1, 123)$ & $(10, 3)$ & $(0.25, 0.25, 0.25, 4)$ & $(0.0001, 1, 10)$ & $(0.1, 1, 0.0001, 0.001)$ & $(0.01, 100, 0.00001, 203)$ & $(0.01, 100, 0.00001, 203)$ \\
hill\_valley & $60.98$ & $53.3$ & $68.05$ & $51.92$ & $68.78$ & $51.92$ & $56.2$ & $65.33$ & $67.82$ & $68.96$ & $75.89$ \\
 & $(0.00001)$ & $(0.00001, 0.00001, 0.00001)$ & $(10000, 103)$ & $(1, 103)$ & $(10, 103)$ & $(10000, 103)$ & $(0.25, 0.25, 0.5, 1)$ & $(1, 1000, 100)$ & $(1, 0.01, 0.001, 0.000001)$ & $(10, 0.1, 0.0001, 123)$ & $(100, 1, 100, 143)$ \\
mammographic & $80.27$ & $77.06$ & $80.28$ & $82.01$ & $80.28$ & $81.66$ & $83.33$ & $81.61$ & $83.26$ & $84.08$ & $83.74$ \\
 & $(0.01)$ & $(0.00001, 0.00001, 0.00001)$ & $(1, 83)$ & $(100000, 63)$ & $(100000, 23)$ & $(100, 23)$ & $(0.25, 0.25, 1, 4)$ & $(10, 1, 100)$ & $(10, 0.1, 0.000001, 0.0001)$ & $(1000, 1, 0.00001, 23)$ & $(1000, 0.00001, 0.0001, 23)$ \\
monks\_3 & $80.24$ & $76.11$ & $95.21$ & $95.41$ & $95.21$ & $95.81$ & $96.39$ & $95.21$ & $92.18$ & $96.41$ & $96.41$ \\
 & $(0.01)$ & $(0.00001, 0.00001, 0.00001)$ & $(0.1, 143)$ & $(0.1, 163)$ & $(0.1, 163)$ & $(0.1, 123)$ & $(0.25, 0.25, 2, 4)$ & $(1000, 1, 70)$ & $(0.01, 0.1, 0.00001, 0.001)$ & $(0.1, 0.00001, 0.001, 163)$ & $(0.1, 1, 0.00001, 183)$ \\
new-thyroid1 & $78.46$ & $82.31$ & $98.46$ & $96.92$ & $98.46$ & $96.92$ & $95.31$ & $98.16$ & $98.67$ & $100$ & $100$ \\
 & $(0.1)$ & $(0.00001, 0.00001, 0.00001)$ & $(10, 103)$ & $(1, 43)$ & $(1, 23)$ & $(10, 103)$ & $(0.25, 2, 0.5, 0.25)$ & $(10000, 0.0001, 10)$ & $(2.5, 0.01, 0.0001, 0.000001)$ & $(0.001, 0.0001, 0.00001, 183)$ & $(1, 0.001, 0.01, 123)$ \\
oocytes\_merluccius\_nucleus\_4d & $74.27$ & $64.82$ & $82.41$ & $81.11$ & $83.71$ & $80.78$ & $75.16$ & $80.54$ & $81.38$ & $82.08$ & $81.76$ \\
 & $(0.00001)$ & $(0.00001, 0.00001, 0.00001)$ & $(0.1, 83)$ & $(1, 143)$ & $(1, 143)$ & $(0.1, 83)$ & $(0.5, 4, 0.5, 4)$ & $(0.0001, 10000, 30)$ & $(5, 5, 0.000001, 0.00001)$ & $(10, 100, 0.1, 143)$ & $(1, 10000, 0.001, 203)$ \\
oocytes\_trisopterus\_nucleus\_2f & $78.83$ & $58.39$ & $86.13$ & $78.83$ & $85.04$ & $78.83$ & $82.05$ & $82.41$ & $83.88$ & $85.04$ & $86.5$ \\
 & $(0.0001)$ & $(0.00001, 0.00001, 0.00001)$ & $(0.01, 103)$ & $(0.1, 143)$ & $(100000, 43)$ & $(0.01, 103)$ & $(0.25, 0.25, 1, 4)$ & $(1, 1, 10)$ & $(0.1, 2.5, 0.001, 0.0001)$ & $(100000, 0.0001, 100, 83)$ & $(0.1, 100, 0.1, 203)$ \\
parkinsons & $71.19$ & $71.19$ & $81.36$ & $77.97$ & $89.83$ & $77.97$ & $93.1$ & $87.4$ & $88.21$ & $89.83$ & $93.22$ \\
 & $(0.01)$ & $(0.00001, 0.00001, 0.00001)$ & $(0.1, 163)$ & $(1, 63)$ & $(0.1, 203)$ & $(0.1, 163)$ & $(0.25, 0.25, 1, 2)$ & $(10, 1000, 100)$ & $(1, 1, 0.0001, 0.000001)$ & $(1000, 100000, 0.1, 163)$ & $(1, 1000, 0.01, 203)$ \\
pima & $76.19$ & $73.33$ & $74.89$ & $74.03$ & $74.46$ & $74.03$ & $69.13$ & $74.9$ & $72.91$ & $76.19$ & $74.89$ \\
 & $(0.01)$ & $(0.00001, 0.00001, 0.00001)$ & $(0.1, 63)$ & $(0.01, 143)$ & $(0.01, 143)$ & $(0.1, 63)$ & $(0.25, 2, 2, 4)$ & $(1000, 10, 50)$ & $(10, 2.5, 0.0001, 0.001)$ & $(0.001, 1, 0.01, 3)$ & $(10000, 0.01, 0.001, 3)$ \\
pittsburg\_bridges\_T\_OR\_D & $70.85$ & $65.85$ & $89.32$ & $93.55$ & $80.65$ & $80.65$ & $90$ & $82.52$ & $80.98$ & $90.32$ & $74.19$ \\
 & $(100)$ & $(0.00001, 0.00001, 0.00001)$ & $(0.01, 83)$ & $(0.01, 43)$ & $(1000, 3)$ & $(0.01, 3)$ & $(0.25, 0.25, 1, 4)$ & $(10000, 1, 70)$ & $(0.01, 5, 0.000001, 0.00001)$ & $(0.00001, 1, 0.01, 163)$ & $(10000, 10, 0.0001, 3)$ \\
planning & $65.85$ & $63.64$ & $76.36$ & $76.36$ & $76.36$ & $76.36$ & $68.52$ & $62.87$ & $65.65$ & $76.36$ & $63.64$ \\
 & $(10)$ & $(0.00001, 0.00001, 0.00001)$ & $(0.01, 23)$ & $(0.00001, 3)$ & $(0.00001, 23)$ & $(0.01, 23)$ & $(0.25, 0.25, 1, 1)$ & $(1, 0.0001, 100)$ & $(2.5, 10, 0.000001, 0.00001)$ & $(0.01, 0.00001, 0.0001, 23)$ & $(0.0001, 0.01, 100000, 43)$ \\
ripley & $89.07$ & $80.67$ & $86.13$ & $87.53$ & $87.53$ & $87.53$ & $89.07$ & $85.17$ & $86.84$ & $89.87$ & $89.6$ \\
 & $(100000)$ & $(0.00001, 0.00001, 0.00001)$ & $(10000, 103)$ & $(100, 203)$ & $(10000, 123)$ & $(10000, 103)$ & $(0.25, 0.25, 0.25, 4)$ & $(10, 1, 10)$ & $(5, 10, 0.0001, 0.001)$ & $(1, 0.00001, 0.0001, 123)$ & $(10000, 0.01, 10, 83)$ \\ \hline
Average Acc. & $76.65$ & $67.24$ & $82.66$ & $81.7$ & $83.14$ & $81.26$ & $79.9$ & $81.1$ & $81.98$ & $85.15$ & $83.18$ \\ \hline
Average Rank & $7.72$ & $10.5$ & $5.54$ & $6.26$ & $4.94$ & $6.06$ & $6.52$ & $6.7$ & $6$ & $2.15$ & $3.61$ \\ \hline
\multicolumn{10}{l}{$^{\dagger}$ represents the proposed models.}
\end{tabular}}
\end{table*}

\renewcommand{\thetable}{S.II}
\begin{table*}[htp]
\centering
    \caption{Performance comparison of the proposed MvRVFL along with the baseline models based on classification Acc. for AwA datasets.}
    \label{Classification performance AwA in Linear Case.}
    \resizebox{1.1\linewidth}{!}{
\begin{tabular}{lccccccccccc}
\hline
Dataset & SVM2K \cite{farquhar2005two} & MvTSVM \cite{xie2015multi}  & ELM-Vw-$A$ \cite{huang2006extreme} & ELM-Vw-$B$ \cite{huang2006extreme} & RVFL-Vw-$A$ \cite{pao1994learning} & RVFL-Vw-$B$ \cite{pao1994learning} & MVLDM \cite{hu2024multiview} & CDC \cite{kang2024cdc} & MCMC \cite{bian2025multilevel} & MvRVFL-1$^{\dagger}$ & MvRVFL-2$^{\dagger}$ \\ 
 & $(C_1)$ & $(C_1, C_2, D)$ & $(C, N)$ & $(C, N)$ & $(C, N)$ & $(C, N)$ & $(C_1, \nu_1, \nu_2, \sigma)$ & $(\alpha, \beta, m)$ & $(\lambda_1, \lambda_2, \lambda_3, \zeta)$ & $(C_1, C_2, \rho, N)$ & $(C_1, C_2, \rho, N)$ \\ \hline
Chimpanzee vs Giant   panda & $84.03$ & $47.22$ & $71.53$ & $72.92$ & $71.53$ & $80.89$ & $72.22$ & $75.21$ & $76.67$ & $88.19$ & $84.72$ \\
 & $(0.00001)$ & $(0.00001, 0.00001, 0.00001)$ & $(0.0001, 123)$ & $(0.00001, 163)$ & $(0.001, 3)$ & $(0.001, 3)$ & $(1000, 0.0001, 0.01, 4)$ & $(1000, 0.0001, 50)$ & $(0.1, 5, 0.0001, 0.00001)$ & $(100, 100000, 0.01, 3)$ & $(0.001, 10, 0.0001, 3)$ \\
Chimpanzee vs Leopard & $80.11$ & $46.53$ & $63.89$ & $83.33$ & $72.92$ & $80.42$ & $68.75$ & $85.63$ & $81.04$ & $89.58$ & $84.72$ \\
 & $(0.00001)$ & $(0.00001, 0.00001, 0.00001)$ & $(1000, 43)$ & $(0.00001, 203)$ & $(0.01, 23)$ & $(0.0001, 3)$ & $(0.001, 0.00001, 0.00001, 4)$ & $(0.0001, 10, 30)$ & $(5, 0.01, 0.001, 0.000001)$ & $(0.001, 1, 0.00001, 3)$ & $(0.001, 1, 0.00001, 3)$ \\
Chimpanzee vs Persian cat & $70.86$ & $50$ & $79.86$ & $69.44$ & $79.17$ & $80.56$ & $86.11$ & $75.25$ & $71.25$ & $83.33$ & $74.31$ \\
 & $(0.0001)$ & $(0.00001, 0.00001, 0.00001)$ & $(0.0001, 183)$ & $(0.00001, 183)$ & $(0.0001, 183)$ & $(0.1, 3)$ & $(100, 0.0001, 0.00001, 4)$ & $(1, 10000, 100)$ & $(1, 2.5, 0.00001, 0.001)$ & $(0.001, 1, 0.00001, 3)$ & $(0.001, 1, 0.001, 23)$ \\
Chimpanzee vs Pig & $50.42$ & $51.39$ & $68.75$ & $81.25$ & $69.44$ & $79.17$ & $66.67$ & $80.21$ & $79.04$ & $61.81$ & $80.56$ \\
 & $(0.01)$ & $(0.00001, 0.00001, 0.00001)$ & $(0.00001, 163)$ & $(0.00001, 163)$ & $(0.00001, 163)$ & $(10, 3)$ & $(100000, 0.001, 0.001, 4)$ & $(10, 1, 70)$ & $(10, 0.1, 0.0001, 0.0001)$ & $(0.1, 10000, 100000, 23)$ & $(0.001, 1, 0.00001, 3)$ \\
Chimpanzee vs Hippopotamus & $70.94$ & $54.86$ & $71.53$ & $70.14$ & $72.92$ & $78.47$ & $78.47$ & $70.63$ & $72.33$ & $79.17$ & $75$ \\
 & $(0.00001)$ & $(0.00001, 0.00001, 0.00001)$ & $(0.00001, 143)$ & $(0.00001, 183)$ & $(0.00001, 143)$ & $(0.001, 3)$ & $(1000, 0.0001, 1, 4)$ & $(10000, 1000, 10)$ & $(0.01, 1, 0.000001, 0.00001)$ & $(0.0001, 0.1, 0.00001, 3)$ & $(0.001, 1, 0.00001, 3)$ \\
Chimpanzee vs Humpback whale & $92.36$ & $81.39$ & $86.81$ & $92.36$ & $88.89$ & $91.14$ & $81.25$ & $91.96$ & $90.79$ & $95.14$ & $93.06$ \\
 & $(0.1)$ & $(0.00001, 0.00001, 0.00001)$ & $(0.0001, 83)$ & $(0.00001, 183)$ & $(0.0001, 83)$ & $(0.01, 3)$ & $(10000, 0.001, 0.0001, 4)$ & $(0.0001, 0.0001, 100)$ & $(2.5, 10, 0.001, 0.0001)$ & $(0.001, 0.01, 0.00001, 23)$ & $(0.001, 10, 0.00001, 3)$ \\
Chimpanzee vs Raccoon & $80.33$ & $63.47$ & $69.44$ & $63.89$ & $73.61$ & $79.86$ & $72.22$ & $81.04$ & $80.83$ & $83.33$ & $82.64$ \\
 & $(0.001)$ & $(0.00001, 0.00001, 100000)$ & $(0.0001, 203)$ & $(0.00001, 123)$ & $(0.001, 3)$ & $(1000, 23)$ & $(10000, 0.00001, 0.01, 4)$ & $(1, 10, 50)$ & $(5, 5, 0.00001, 0.000001)$ & $(0.001, 1, 0.0001, 3)$ & $(0.001, 1, 0.00001, 3)$ \\
Chimpanzee vs Rat & $77.08$ & $52.78$ & $57.64$ & $75$ & $63.89$ & $71.94$ & $68.06$ & $82.08$ & $83.96$ & $81.25$ & $78.47$ \\
 & $(10)$ & $(0.00001, 0.00001, 0.00001)$ & $(100000, 43)$ & $(0.00001, 183)$ & $(0.001, 163)$ & $(0.001, 3)$ & $(100000, 0.0001, 0.00001, 4)$ & $(1000, 1, 30)$ & $(0.1, 0.1, 0.0001, 0.001)$ & $(0.1, 1, 0.01, 43)$ & $(0.00001, 0.0001, 0.00001, 63)$ \\
Chimpanzee vs Seal & $70.69$ & $53.47$ & $79.17$ & $76.39$ & $79.17$ & $79.81$ & $75.69$ & $78.54$ & $76.46$ & $83.33$ & $78.47$ \\
 & $(0.01)$ & $(0.00001, 0.00001, 0.00001)$ & $(0.0001, 183)$ & $(0.00001, 143)$ & $(0.001, 3)$ & $(0.0001, 3)$ & $(0.001, 0.001, 0.01, 0.25)$ & $(10, 10000, 70)$ & $(1, 0.01, 0.000001, 0.0001)$ & $(10, 10000, 0.001, 23)$ & $(0.01, 10, 0.0001, 3)$ \\
Giant panda vs Leopard & $80.19$ & $54.17$ & $61.11$ & $78.47$ & $72.92$ & $80.11$ & $61.81$ & $83.96$ & $85.79$ & $88.89$ & $83.33$ \\
 & $(0.01)$ & $(0.00001, 0.00001, 0.00001)$ & $(0.0001, 103)$ & $(0.00001, 203)$ & $(0.001, 3)$ & $(0.001, 3)$ & $(100000, 0.00001, 0.01, 0.25)$ & $(10000, 0.0001, 100)$ & $(10, 2.5, 0.001, 0.00001)$ & $(0.001, 1, 0.001, 3)$ & $(100, 0.01, 100, 63)$ \\
Giant panda vs Persian cat & $81.81$ & $52.08$ & $77.78$ & $76.39$ & $64.58$ & $80.42$ & $66.67$ & $80.96$ & $79.5$ & $85.42$ & $81.94$ \\
 & $(0.0001)$ & $(0.00001, 0.00001, 0.00001)$ & $(100, 103)$ & $(0.00001, 203)$ & $(0.001, 23)$ & $(100000, 3)$ & $(0.01, 0.00001, 0.01, 4)$ & $(0.0001, 1000, 50)$ & $(0.01, 0.1, 0.00001, 0.000001)$ & $(0.001, 0.1, 0.00001, 23)$ & $(1000, 0.01, 1000, 23)$ \\
Giant panda vs Pig & $80.56$ & $51.39$ & $63.89$ & $75$ & $65.97$ & $79.81$ & $65.97$ & $79.17$ & $82.21$ & $84.72$ & $83.33$ \\
 & $(0.0001)$ & $(0.00001, 0.00001, 0.00001)$ & $(1, 43)$ & $(0.00001, 203)$ & $(0.001, 3)$ & $(0.001, 3)$ & $(1000, 0.00001, 0.01, 0.25)$ & $(1, 0.0001, 10)$ & $(2.5, 1, 0.0001, 0.001)$ & $(1, 0.1, 0.1, 23)$ & $(0.00001, 0.01, 0.00001, 23)$ \\
Giant panda vs Hippopotamus & $77.78$ & $54.17$ & $74.31$ & $81.25$ & $68.06$ & $71.94$ & $74.31$ & $80.42$ & $81.75$ & $83.33$ & $75$ \\
 & $(0.01)$ & $(0.00001, 0.00001, 0.00001)$ & $(0.00001, 183)$ & $(0.00001, 183)$ & $(0.01, 3)$ & $(0.01, 23)$ & $(100000, 0.001, 100, 2)$ & $(10, 1000, 30)$ & $(5, 10, 0.000001, 0.0001)$ & $(0.01, 1, 0.00001, 23)$ & $(0.001, 0.00001, 0.00001, 43)$ \\
Giant panda vs Humpback whale & $93.06$ & $46.53$ & $93.06$ & $91.67$ & $93.06$ & $93.22$ & $93.75$ & $82.92$ & $90.42$ & $95.83$ & $94.44$ \\
 & $(0.00001)$ & $(0.00001, 0.00001, 0.00001)$ & $(0.0001, 203)$ & $(0.00001, 203)$ & $(0.0001, 203)$ & $(0.001, 3)$ & $(0.001, 0.00001, 0.01, 4)$ & $(1000, 10, 70)$ & $(0.1, 2.5, 0.001, 0.00001)$ & $(0.001, 0.0001, 0.0001, 3)$ & $(0.01, 100, 0.0001, 3)$ \\
Giant panda vs Raccoon & $80.19$ & $52.78$ & $68.06$ & $74.31$ & $68.75$ & $80.19$ & $64.58$ & $76$ & $74.58$ & $84.72$ & $78.47$ \\
 & $(10000)$ & $(0.00001, 0.00001, 0.00001)$ & $(0.0001, 203)$ & $(0.00001, 163)$ & $(0.00001, 203)$ & $(0.001, 3)$ & $(100000, 0.00001, 0.00001, 2)$ & $(10000, 1, 50)$ & $(1, 5, 0.00001, 0.000001)$ & $(0.001, 10, 0.0001, 3)$ & $(0.001, 10, 0.0001, 3)$ \\
Giant panda vs Rat & $83.33$ & $69.31$ & $66.67$ & $76.39$ & $68.06$ & $80.5$ & $70.14$ & $80.83$ & $80$ & $86.81$ & $81.25$ \\
 & $(1)$ & $(0.00001, 0.00001, 100000)$ & $(10, 103)$ & $(0.00001, 123)$ & $(0.001, 43)$ & $(0.001, 3)$ & $(10000, 0.0001, 0.01, 0.25)$ & $(0.0001, 1, 100)$ & $(10, 0.01, 0.0001, 0.001)$ & $(0.001, 1, 0.0001, 3)$ & $(0.001, 1, 0.00001, 3)$ \\
Giant panda vs Seal & $85.89$ & $56.94$ & $80.56$ & $77.08$ & $80.56$ & $80.19$ & $86.81$ & $83.33$ & $76.75$ & $91.67$ & $79.17$ \\
 & $(0.1)$ & $(0.00001, 0.00001, 0.00001)$ & $(0.0001, 203)$ & $(0.00001, 183)$ & $(0.001, 3)$ & $(0.0001, 3)$ & $(100000, 0.001, 0.01, 2)$ & $(1, 1000, 70)$ & $(0.01, 10, 0.000001, 0.00001)$ & $(0.01, 0.01, 0.00001, 3)$ & $(0.001, 0.001, 0.0001, 23)$ \\
Leopard vs Persian cat & $82.19$ & $79.31$ & $70.83$ & $84.72$ & $77.78$ & $88.19$ & $80.56$ & $88.33$ & $84.54$ & $90.97$ & $85.42$ \\
 & $(0.00001)$ & $(0.00001, 0.00001, 0.00001)$ & $(100, 63)$ & $(0.00001, 203)$ & $(0.001, 23)$ & $(0.001, 3)$ & $(0.00001, 0.00001, 100, 4)$ & $(10, 0.0001, 10)$ & $(2.5, 0.1, 0.001, 0.0001)$ & $(0.001, 1, 0.001, 3)$ & $(0.00001, 10, 0.00001, 163)$ \\
Leopard vs Pig & $75$ & $61.39$ & $61.11$ & $75$ & $66.67$ & $72.17$ & $68.75$ & $72.08$ & $75.21$ & $78.47$ & $75$ \\
 & $(0.01)$ & $(0.00001, 0.00001, 0.00001)$ & $(0.0001, 183)$ & $(0.00001, 183)$ & $(0.001, 3)$ & $(0.001, 3)$ & $(0.01, 0.001, 100, 4)$ & $(1000, 10000, 30)$ & $(5, 1, 0.00001, 0.001)$ & $(0.001, 1, 0.0001, 3)$ & $(0.1, 0.00001, 0.1, 23)$ \\
Leopard vs Hippopotamus & $78.17$ & $50.69$ & $73.61$ & $77.08$ & $74.31$ & $75.94$ & $75$ & $80.33$ & $79.79$ & $81.94$ & $76.39$ \\
 & $(10)$ & $(0.00001, 0.00001, 0.00001)$ & $(0.0001, 143)$ & $(0.00001, 143)$ & $(0.0001, 43)$ & $(0.0001, 3)$ & $(10000, 0.0001, 0.001, 4)$ & $(10000, 10, 100)$ & $(0.1, 0.01, 0.0001, 0.000001)$ & $(0.001, 0.00001, 0.00001, 3)$ & $(0.00001, 10, 0.00001, 183)$ \\
Leopard vs Humpback whale & $90.75$ & $79.31$ & $89.58$ & $91.67$ & $90.97$ & $90.83$ & $89.58$ & $89.17$ & $92.79$ & $93.75$ & $88.19$ \\
 & $(0.00001)$ & $(0.00001, 0.00001, 0.00001)$ & $(0.00001, 103)$ & $(0.00001, 183)$ & $(0.0001, 143)$ & $(0.001, 23)$ & $(100, 0.00001, 0.01, 4)$ & $(0.0001, 10000, 70)$ & $(1, 10, 0.001, 0.00001)$ & $(0.1, 100, 0.0001, 43)$ & $(0.0001, 0.001, 0.0001, 63)$ \\
Leopard vs Raccoon & $80.56$ & $55$ & $59.03$ & $57.64$ & $59.03$ & $69.25$ & $56.94$ & $81.67$ & $79.58$ & $84.03$ & $70.14$ \\
 & $(0.0001)$ & $(0.00001, 0.00001, 0.00001)$ & $(0.0001, 183)$ & $(0.0001, 183)$ & $(0.001, 3)$ & $(0.001, 3)$ & $(0.01, 0.00001, 0.00001, 0.25)$ & $(1, 1, 50)$ & $(10, 5, 0.00001, 0.0001)$ & $(0.001, 0.00001, 0.00001, 3)$ & $(0.0001, 0.1, 0.00001, 3)$ \\
Leopard vs Rat & $76.42$ & $68.61$ & $72.22$ & $76.39$ & $68.06$ & $79.86$ & $65.28$ & $78.58$ & $78.17$ & $79.17$ & $80.56$ \\
 & $(0.0001)$ & $(0.00001, 0.00001, 0.00001)$ & $(10000, 43)$ & $(0.00001, 183)$ & $(0.001, 3)$ & $(0.001, 3)$ & $(10000, 0.0001, 0.00001, 0.25)$ & $(10, 10, 100)$ & $(0.01, 2.5, 0.0001, 0.001)$ & $(100, 100000, 1, 3)$ & $(0.00001, 0.01, 0.00001, 23)$ \\
Leopard vs Seal & $80.42$ & $63.47$ & $75.69$ & $79.86$ & $75$ & $83.33$ & $81.25$ & $77.33$ & $78$ & $84.03$ & $78.47$ \\
 & $(10000)$ & $(0.00001, 0.00001, 0.00001)$ & $(0.0001, 123)$ & $(0.00001, 143)$ & $(0.0001, 123)$ & $(0.001, 43)$ & $(10000, 0.0001, 0.0001, 4)$ & $(1000, 0.0001, 10)$ & $(2.5, 0.01, 0.000001, 0.00001)$ & $(0.01, 100, 0.0001, 23)$ & $(0.0001, 0.0001, 0.00001, 3)$ \\
Persian cat vs Pig & $70$ & $69.31$ & $63.89$ & $67.36$ & $70.14$ & $74.31$ & $69.44$ & $72.5$ & $71.92$ & $74.31$ & $71.53$ \\
 & $(0.001)$ & $(0.00001, 0.00001, 0.00001)$ & $(0.0001, 183)$ & $(0.00001, 163)$ & $(0.001, 3)$ & $(10000, 3)$ & $(100, 0.00001, 0.01, 4)$ & $(10000, 1000, 30)$ & $(0.01, 1, 0.001, 0.001)$ & $(0.01, 10, 0.00001, 3)$ & $(0.01, 10, 0.001, 3)$ \\
Persian cat vs Hippopotamus & $76.81$ & $76.53$ & $75.69$ & $79.86$ & $77.08$ & $78.94$ & $75.69$ & $76.04$ & $79.58$ & $83.33$ & $79.86$ \\
 & $(0.01)$ & $(0.00001, 0.00001, 0.00001)$ & $(1, 63)$ & $(0.00001, 203)$ & $(0.01, 3)$ & $(1, 3)$ & $(100000, 0.001, 0.00001, 0.25)$ & $(0.0001, 10, 50)$ & $(2.5, 10, 0.0001, 0.0001)$ & $(0.0001, 0.00001, 0.00001, 3)$ & $(0.01, 0.00001, 0.001, 3)$ \\
Persian cat vs Humpback whale & $71.67$ & $71.39$ & $81.25$ & $88.19$ & $81.94$ & $81.75$ & $85.42$ & $82.5$ & $84.79$ & $94.44$ & $82.64$ \\
 & $(0.00001)$ & $(0.00001, 0.00001, 0.00001)$ & $(0.00001, 143)$ & $(0.00001, 203)$ & $(0.00001, 143)$ & $(1, 23)$ & $(0.01, 0.00001, 0.1, 4)$ & $(1, 10000, 100)$ & $(5, 0.1, 0.000001, 0.00001)$ & $(10000, 1, 10000, 23)$ & $(0.001, 0.1, 0.0001, 23)$ \\
Persian cat vs Raccoon & $82.64$ & $79.31$ & $73.61$ & $81.25$ & $69.44$ & $71.64$ & $65.97$ & $82.42$ & $83$ & $84.03$ & $73.61$ \\
 & $(0.00001)$ & $(0.00001, 0.00001, 0.00001)$ & $(100, 23)$ & $(0.00001, 163)$ & $(0.001, 23)$ & $(0.001, 3)$ & $(100000, 0.00001, 0.001, 2)$ & $(10, 1, 70)$ & $(0.1, 5, 0.00001, 0.001)$ & $(1, 10000, 0.01, 3)$ & $(0.00001, 100000, 0.00001, 3)$ \\
Persian cat vs Rat & $60.44$ & $64.17$ & $54.86$ & $56.25$ & $59.72$ & $60.67$ & $56.94$ & $60.92$ & $62.75$ & $61.11$ & $63.89$ \\
 & $(0.001)$ & $(0.00001, 0.00001, 0.00001)$ & $(0.001, 183)$ & $(100000, 103)$ & $(0.001, 3)$ & $(0.001, 3)$ & $(1000, 0.00001, 10, 0.5)$ & $(1000, 1, 100)$ & $(1, 0.1, 0.001, 0.0001)$ & $(0.0001, 0.0001, 0.00001, 63)$ & $(0.001, 1, 0.00001, 3)$ \\
Persian cat vs Seal & $80.42$ & $73.47$ & $72.22$ & $71.53$ & $65.97$ & $72.94$ & $83.33$ & $83.96$ & $84$ & $84.72$ & $76.39$ \\
 & $(1)$ & $(0.00001, 0.00001, 0.00001)$ & $(0.0001, 183)$ & $(0.00001, 83)$ & $(0.01, 63)$ & $(0.001, 3)$ & $(10000, 0.00001, 0.01, 4)$ & $(10000, 0.0001, 50)$ & $(10, 1, 0.000001, 0.00001)$ & $(0.01, 100, 0.001, 3)$ & $(0.01, 100, 0.00001, 3)$ \\
Pig vs Hippopotamus & $71.53$ & $65.83$ & $70.14$ & $65.97$ & $64.58$ & $67.36$ & $72.22$ & $64.79$ & $69.79$ & $71.53$ & $64.58$ \\
 & $(0.0001)$ & $(0.00001, 0.00001, 0.00001)$ & $(0.0001, 83)$ & $(0.00001, 203)$ & $(0.0001, 123)$ & $(0.01, 63)$ & $(1000, 0.00001, 0.01, 0.25)$ & $(0.0001, 0.0001, 30)$ & $(0.01, 2.5, 0.00001, 0.0001)$ & $(0.01, 100, 0.001, 3)$ & $(0.0001, 10, 0.00001, 163)$ \\
Pig vs Humpback whale & $80.19$ & $77.92$ & $82.64$ & $89.58$ & $82.64$ & $80.19$ & $88.89$ & $87.92$ & $82.96$ & $90.28$ & $84.72$ \\
 & $(0.01)$ & $(0.00001, 0.00001, 0.00001)$ & $(0.00001, 203)$ & $(0.00001, 203)$ & $(0.00001, 203)$ & $(0.001, 3)$ & $(100000, 0.001, 0.0001, 0.25)$ & $(1, 10, 10)$ & $(2.5, 0.01, 0.001, 0.001)$ & $(0.01, 10, 0.01, 3)$ & $(0.00001, 10, 0.00001, 143)$ \\
Pig vs Raccoon & $71.69$ & $69.31$ & $61.11$ & $72.92$ & $64.58$ & $72.22$ & $62.5$ & $70.83$ & $71.29$ & $80.56$ & $72.33$ \\
 & $(0.00001)$ & $(0.00001, 0.00001, 0.00001)$ & $(0.0001, 43)$ & $(0.00001, 183)$ & $(0.0001, 43)$ & $(10000, 3)$ & $(0.0001, 0.00001, 0.0001, 4)$ & $(10, 1000, 70)$ & $(5, 2.5, 0.0001, 0.000001)$ & $(0.001, 1, 0.00001, 3)$ & $(0.0001, 0.01, 0.0001, 23)$ \\
Pig vs Rat & $71.53$ & $68.61$ & $62.5$ & $59.72$ & $57.64$ & $68.06$ & $64.58$ & $63.96$ & $62.25$ & $64.58$ & $69.72$ \\
 & $(0.01)$ & $(0.00001, 0.00001, 0.00001)$ & $(1000, 63)$ & $(0.0001, 203)$ & $(0.001, 43)$ & $(0.001, 3)$ & $(1000, 0.00001, 0.01, 0.25)$ & $(1000, 10, 100)$ & $(0.1, 1, 0.000001, 0.00001)$ & $(100, 10000, 0.001, 43)$ & $(0.00001, 1, 0.00001, 143)$ \\
Pig vs Seal & $72.69$ & $65.56$ & $70.14$ & $68.06$ & $72.92$ & $74.31$ & $72.92$ & $72.08$ & $68.96$ & $78.47$ & $74.31$ \\
 & $(0.01)$ & $(0.00001, 0.00001, 0.00001)$ & $(0.00001, 163)$ & $(0.00001, 183)$ & $(0.001, 3)$ & $(0.0001, 3)$ & $(0.001, 0.001, 0.00001, 4)$ & $(10000, 1, 30)$ & $(1, 10, 0.00001, 0.001)$ & $(0.001, 1, 0.00001, 3)$ & $(0.001, 10, 0.0001, 3)$ \\
Hippopotamus vs Humpback whale & $82.03$ & $80.31$ & $77.78$ & $79.86$ & $81.25$ & $82.81$ & $79.86$ & $82$ & $81.63$ & $88.19$ & $83.33$ \\
 & $(0.01)$ & $(0.00001, 0.00001, 0.00001)$ & $(0.0001, 143)$ & $(0.00001, 143)$ & $(0.01, 3)$ & $(0.001, 23)$ & $(0.001, 0.0001, 10000, 0.25)$ & $(0.0001, 1000, 10)$ & $(10, 0.1, 0.0001, 0.0001)$ & $(0.01, 10, 0.00001, 43)$ & $(0.001, 10, 0.00001, 3)$ \\
Hippopotamus vs Raccoon & $78.47$ & $75.14$ & $72.22$ & $70.14$ & $78.47$ & $80.94$ & $75.69$ & $82.92$ & $81.71$ & $81.94$ & $77.78$ \\
 & $(0.01)$ & $(0.00001, 0.00001, 0.00001)$ & $(0.0001, 203)$ & $(0.00001, 183)$ & $(0.001, 3)$ & $(0.001, 23)$ & $(0.00001, 0.001, 1, 4)$ & $(1, 0.0001, 70)$ & $(0.01, 0.01, 0.0001, 0.000001)$ & $(0.01, 10, 0.00001, 23)$ & $(0.001, 1, 0.0001, 3)$ \\
Hippopotamus vs Rat & $75.33$ & $75.83$ & $65.28$ & $71.53$ & $70.14$ & $70.25$ & $64.58$ & $78.33$ & $73.75$ & $81.25$ & $75$ \\
 & $(100)$ & $(0.00001, 0.00001, 0.00001)$ & $(1000, 43)$ & $(0.00001, 183)$ & $(0.01, 23)$ & $(0.001, 3)$ & $(100000, 0.0001, 0.1, 4)$ & $(10, 0.0001, 50)$ & $(2.5, 5, 0.000001, 0.00001)$ & $(0.01, 10, 0.01, 3)$ & $(0.00001, 0.01, 0.00001, 43)$ \\
Hippopotamus vs Seal & $69.44$ & $49.31$ & $58.33$ & $65.28$ & $61.81$ & $70.83$ & $60.42$ & $54.58$ & $58.96$ & $59.03$ & $61.81$ \\
 & $(0.01)$ & $(0.00001, 0.00001, 0.00001)$ & $(0.1, 63)$ & $(0.00001, 183)$ & $(0.001, 43)$ & $(0.001, 3)$ & $(0.001, 0.01, 0.00001, 4)$ & $(1000, 10000, 100)$ & $(5, 0.01, 0.001, 0.0001)$ & $(1, 1, 1, 23)$ & $(0.0001, 0.1, 0.00001, 63)$ \\
Humpback whale vs Raccoon & $85.67$ & $80.69$ & $84.03$ & $90.97$ & $81.94$ & $82.36$ & $83.33$ & $82.63$ & $82.63$ & $90.28$ & $83.33$ \\
 & $(0.00001)$ & $(0.00001, 0.00001, 10000)$ & $(0.00001, 83)$ & $(0.00001, 123)$ & $(0.0001, 83)$ & $(0.001, 3)$ & $(10000, 0.0001, 0.01, 0.25)$ & $(10000, 10, 70)$ & $(0.1, 2.5, 0.0001, 0.001)$ & $(0.001, 0.001, 0.0001, 43)$ & $(0.0001, 0.0001, 0.00001, 3)$ \\
Humpback whale vs Rat & $82.28$ & $80.31$ & $81.94$ & $86.11$ & $81.94$ & $80.58$ & $77.78$ & $82.38$ & $82.5$ & $91.67$ & $84.72$ \\
 & $(0.01)$ & $(0.00001, 0.00001, 0.00001)$ & $(0.0001, 183)$ & $(0.00001, 203)$ & $(0.0001, 183)$ & $(100, 3)$ & $(100000, 0.001, 10, 0.25)$ & $(1, 1000, 50)$ & $(1, 0.01, 0.000001, 0.0001)$ & $(10, 10000, 0.1, 23)$ & $(0.00001, 0.001, 0.00001, 3)$ \\
Humpback whale vs Seal & $76.39$ & $72.08$ & $77.78$ & $73.61$ & $78.47$ & $79.17$ & $78.47$ & $70$ & $69.58$ & $78.47$ & $71.53$ \\
 & $(0.01)$ & $(0.00001, 0.00001, 0.00001)$ & $(0.0001, 143)$ & $(0.00001, 203)$ & $(0.0001, 143)$ & $(0.0001, 3)$ & $(100000, 0.0001, 0.001, 0.25)$ & $(0.0001, 1, 30)$ & $(10, 2.5, 0.00001, 0.000001)$ & $(0.001, 0.1, 0.0001, 3)$ & $(100, 0.01, 100, 23)$ \\
Raccoon vs Rat & $62.22$ & $61.89$ & $59.72$ & $70.83$ & $61.81$ & $60.22$ & $65.28$ & $65.87$ & $66.63$ & $68.75$ & $66.67$ \\
 & $(100)$ & $(0.00001, 0.00001, 0.00001)$ & $(100, 43)$ & $(0.00001, 163)$ & $(0.001, 163)$ & $(0.001, 3)$ & $(100000, 0.00001, 10, 4)$ & $(10, 10, 10)$ & $(0.01, 0.1, 0.000001, 0.00001)$ & $(0.1, 10, 0.1, 3)$ & $(0.0001, 0.01, 0.0001, 43)$ \\
Raccoon vs Seal & $90.28$ & $75.39$ & $78.47$ & $84.03$ & $82.64$ & $80.28$ & $75.69$ & $76.46$ & $78.75$ & $88.89$ & $79.86$ \\
 & $(100)$ & $(0.00001, 0.00001, 100000)$ & $(0.00001, 183)$ & $(0.00001, 183)$ & $(0.001, 63)$ & $(0.001, 3)$ & $(100000, 0.00001, 0.01, 4)$ & $(1000, 0.0001, 70)$ & $(2.5, 1, 0.001, 0.0001)$ & $(0.001, 0.00001, 0.0001, 23)$ & $(0.01, 100, 0.00001, 3)$ \\
Rat vs Seal & $70.86$ & $65.17$ & $68.75$ & $68.75$ & $68.75$ & $67.83$ & $69.87$ & $66.04$ & $67.08$ & $77.78$ & $68.06$ \\
 & $(0.001)$ & $(0.00001, 0.00001, 0.00001)$ & $(0.00001, 183)$ & $(0.00001, 143)$ & $(0.00001, 183)$ & $(0.01, 3)$ & $(100000, 0.001, 0.01, 0.25)$ & $(1, 1, 100)$ & $(0.1, 5, 0.00001, 0.001)$ & $(0.01, 10, 0.001, 3)$ & $(0.001, 0.00001, 0.00001, 23)$ \\ \hline
Average Acc. & $77.46$ & $64.31$ & $71.74$ & $75.99$ & $72.87$ & $77.46$ & $73.33$ & $77.66$ & $77.69$ & $82.5$ & $77.97$ \\ \hline
Average Rank & $5.1$ & $9.51$ & $8.36$ & $6.08$ & $7.59$ & $5.26$ & $7.1$ & $5.41$ & $5.34$ & $1.77$ & $4.49$ \\ \hline
\multicolumn{10}{l}{$^{\dagger}$ represents the proposed models.}
\end{tabular}}
\end{table*}

\renewcommand{\thetable}{S.III}
\begin{table*}[htp]
\centering
    \caption{Performance comparison of the proposed MvRVFL along with the baseline models based on classification Acc. for Corel5k datasets.}
    \label{Classification performance Corel5k in Linear Case.}
    \resizebox{1.1\linewidth}{!}{
\begin{tabular}{lccccccccccc}
\hline
Dataset & SVM2K \cite{farquhar2005two} & MvTSVM \cite{xie2015multi} & ELM-Vw-$A$ \cite{huang2006extreme} & ELM-Vw-$B$ \cite{huang2006extreme} & RVFL-Vw-$A$ \cite{pao1994learning} & RVFL-Vw-$B$ \cite{pao1994learning} & MVLDM \cite{hu2024multiview} & CDC \cite{kang2024cdc} & MCMC \cite{bian2025multilevel} & MvRVFL-1$^{\dagger}$ & MvRVFL-2$^{\dagger}$ \\ 
 & $(C_1)$ & $(C_1, C_2, D)$ & $(C, N)$ & $(C, N)$ & $(C, N)$ & $(C, N)$ & $(C_1, \nu_1, \nu_2, \sigma)$ & $(\alpha, \beta, m)$ & $(\lambda_1, \lambda_2, \lambda_3, \zeta)$ & $(C_1, C_2, \rho, N)$ & $(C_1, C_2, \rho, N)$ \\ \hline
1000 & $81.67$ & $(50)$ & $(83.33)$ & $(78.33)$ & $(83.33)$ & $(76.67)$ & $(73.33)$ & $82.17$ & $82$ & $(83.33)$ & $(85)$ \\
 & $(0.0625)$ & $(0.00001, 0.00001, 0.00001)$ & $(0.0001, 183)$ & $(0.00001, 43)$ & $(0.0001, 183)$ & $(0.00001, 3)$ & $(0.25, 0.25, 0.25, 4)$ & $(1, 10, 100)$ & $(0.01, 0.1, 0.000001, 0.000001)$ & $(0.001, 0.01, 0.00001, 23)$ & $(0.0001, 100000, 0.00001, 203)$ \\
10000 & $71.67$ & $(48.33)$ & $(56.67)$ & $(63.33)$ & $(71.67)$ & $(71.67)$ & $(55)$ & $70.87$ & $69.42$ & $(71.67)$ & $(70)$ \\
 & $(0.03125)$ & $(0.00001, 0.00001, 0.00001)$ & $(1000, 23)$ & $(1000, 23)$ & $(0.01, 3)$ & $(0.00001, 63)$ & $(0.25, 0.25, 0.25, 4)$ & $(1000, 0.0001, 50)$ & $(0.1, 0.01, 0.00001, 0.000001)$ & $(0.01, 1000, 0.01, 23)$ & $(0.001, 100, 0.0001, 143)$ \\
100000 & $71.67$ & $(55)$ & $(63.33)$ & $(73.33)$ & $(78.33)$ & $(75)$ & $(61.67)$ & $62.45$ & $63.25$ & $(65)$ & $(71.67)$ \\
 & $(0.0625)$ & $(0.00001, 0.00001, 0.00001)$ & $(100000, 183)$ & $(0.00001, 123)$ & $(0.001, 183)$ & $(0.00001, 123)$ & $(0.25, 0.25, 0.25, 4)$ & $(1000, 10, 100)$ & $(0.1, 1, 0.0001, 0.0001)$ & $(0.1, 10000, 0.01, 143)$ & $(1, 100000, 0.001, 143)$ \\
101000 & $66.67$ & $(58.33)$ & $(75)$ & $(73.33)$ & $(68.33)$ & $(76.67)$ & $(55)$ & $72.49$ & $71.47$ & $(81.67)$ & $(73.33)$ \\
 & $(0.0625)$ & $(0.00001, 0.00001, 0.00001)$ & $(0.001, 123)$ & $(0.00001, 143)$ & $(0.001, 3)$ & $(0.00001, 143)$ & $(0.25, 0.25, 0.25, 4)$ & $(10000, 10000, 10)$ & $(0.01, 1, 0.000001, 0.00001)$ & $(0.001, 1000, 0.00001, 103)$ & $(1, 100000, 0.001, 23)$ \\
102000 & $81.67$ & $(51.67)$ & $(81.67)$ & $(75)$ & $(81.67)$ & $(71.67)$ & $(83.33)$ & $75.04$ & $73.5$ & $(78.33)$ & $(76.67)$ \\
 & $(0.03125)$ & $(0.00001, 0.00001, 0.00001)$ & $(0.01, 63)$ & $(0.00001, 143)$ & $(0.01, 3)$ & $(100000, 3)$ & $(0.25, 0.25, 0.25, 4)$ & $(1, 10000, 70)$ & $(2.5, 0.01, 0.000001, 0.00001)$ & $(0.001, 10, 0.00001, 23)$ & $(0.001, 10, 0.00001, 43)$ \\
103000 & $73.33$ & $(45)$ & $(80)$ & $(75)$ & $(81.67)$ & $(73.33)$ & $(65)$ & $80.86$ & $79$ & $(81.67)$ & $(80)$ \\
 & $(0.03125)$ & $(0.00001, 0.00001, 0.00001)$ & $(0.001, 163)$ & $(0.0001, 43)$ & $(0.001, 163)$ & $(0.00001, 43)$ & $(0.25, 0.25, 0.25, 4)$ & $(1, 0.0001, 100)$ & $(5, 0.1, 0.001, 0.0001)$ & $(0.01, 1000, 0.001, 23)$ & $(0.1, 10000, 0.01, 23)$ \\
104000 & $71.67$ & $(46.67)$ & $(55)$ & $(65)$ & $(76.67)$ & $(76.67)$ & $(46.67)$ & $69.58$ & $68.5$ & $(71.67)$ & $(53.33)$ \\
 & $(32)$ & $(0.00001, 0.00001, 0.00001)$ & $(10, 123)$ & $(10000, 23)$ & $(0.0001, 203)$ & $(0.00001, 83)$ & $(0.25, 0.25, 0.25, 4)$ & $(0.0001, 0.0001, 100)$ & $(0.1, 1, 0.00001, 0.001)$ & $(100000, 0.00001, 0.001, 83)$ & $(1000, 100000, 0.00001, 23)$ \\
108000 & $80$ & $(48.33)$ & $(75.65)$ & $(76.67)$ & $(78.33)$ & $(76.67)$ & $(78.33)$ & $79.75$ & $78.5$ & $(83.33)$ & $(80)$ \\
 & $(8)$ & $(0.00001, 0.00001, 0.00001)$ & $(0.001, 63)$ & $(0.00001, 83)$ & $(0.1, 3)$ & $(0.00001, 3)$ & $(0.25, 0.25, 0.25, 4)$ & $(0.0001, 0.0001, 10)$ & $(1, 0.1, 0.0001, 0.000001)$ & $(0.001, 0.00001, 0.00001, 163)$ & $(0.0001, 10, 0.00001, 3)$ \\
109000 & $65$ & $(43.33)$ & $(68.33)$ & $(73.33)$ & $(68.33)$ & $(75)$ & $(73.33)$ & $57.28$ & $55.5$ & $(75)$ & $(58.33)$ \\
 & $(1)$ & $(0.00001, 0.00001, 0.00001)$ & $(0.0001, 203)$ & $(100, 63)$ & $(0.0001, 203)$ & $(0.0001, 43)$ & $(0.25, 0.25, 0.25, 0.25)$ & $(0.0001,0.0001,30)$ & $(10, 1, 0.001, 0.00001)$ & $(0.01, 10000, 0.00001, 203)$ & $(0.01, 100000, 0.0001, 83)$ \\
113000 & $66.67$ & $(45)$ & $(68.33)$ & $(63.33)$ & $(70)$ & $(63.33)$ & $(66.67)$ & $66.63$ & $65.5$ & $(70)$ & $(68.33)$ \\
 & $(0.03125)$ & $(0.00001, 0.00001, 0.00001)$ & $(0.001, 123)$ & $(1000, 203)$ & $(0.001, 123)$ & $(0.0001, 23)$ & $(0.25, 0.25, 0.25, 0.25)$ & $(0.0001,1,10)$ & $(0.01, 5, 0.00001, 0.0001)$ & $(0.01, 0.00001, 0.001, 3)$ & $(0.01, 0.001, 0.00001, 3)$ \\
118000 & $58.33$ & $(48.33)$ & $(68.33)$ & $(61.67)$ & $(68.33)$ & $(56.67)$ & $(58.33)$ & $70.18$ & $69.5$ & $(71.67)$ & $(53.33)$ \\
 & $(8)$ & $(0.00001, 0.00001, 0.00001)$ & $(0.001, 103)$ & $(100000, 43)$ & $(0.001, 103)$ & $(1000, 43)$ & $(0.25, 0.25, 0.25, 4)$ & $(0.0001,1,30)$ & $(2.5, 2.5, 0.0001, 0.001)$ & $(100000, 100000, 1, 63)$ & $(1, 100000, 0.0001, 143)$ \\
119000 & $71.33$ & $(53.33)$ & $(81.67)$ & $(83.33)$ & $(78.33)$ & $(80)$ & $(70)$ & $70.78$ & $69.87$ & $(78.33)$ & $(71.67)$ \\
 & $(16)$ & $(0.00001, 0.00001, 0.00001)$ & $(0.001, 143)$ & $(0.00001, 123)$ & $(0.01, 63)$ & $(0.00001, 83)$ & $(0.25, 0.25, 0.25, 0.25)$ & $(0.0001,10000,10)$ & $(5, 2.5, 0.000001, 0.00001)$ & $(1, 100000, 0.1, 103)$ & $(0.1, 100000, 0.01, 123)$ \\
12000 & $60$ & $(48.33)$ & $(63.33)$ & $(58.33)$ & $(63.33)$ & $(51.67)$ & $(61.67)$ & $60.55$ & $59.23$ & $(61.67)$ & $(56.67)$ \\
 & $(32)$ & $(0.00001, 0.00001, 0.00001)$ & $(0.0001, 83)$ & $(0.00001, 63)$ & $(0.0001, 83)$ & $(0.00001, 43)$ & $(0.25, 0.25, 0.25, 2)$ & $(0.0001,10000,30)$ & $(0.1, 10, 0.001, 0.0001)$ & $(0.0001, 1000, 0.00001, 143)$ & $(0.00001, 100, 0.00001, 203)$ \\
120000 & $73.33$ & $(48.33)$ & $(78.33)$ & $(80)$ & $(78.33)$ & $(85)$ & $(66.67)$ & $74.83$ & $74$ & $(75)$ & $(88.33)$ \\
 & $(0.03125)$ & $(0.00001, 0.00001, 0.00001)$ & $(0.01, 203)$ & $(0.00001, 203)$ & $(0.01, 203)$ & $(0.00001, 103)$ & $(0.25, 0.25, 0.25, 4)$ & $(1,0.0001,10)$ & $(1, 0.01, 0.00001, 0.001)$ & $(0.0001, 10, 100, 203)$ & $(0.1, 1000, 0.01, 23)$ \\
121000 & $63.33$ & $(55)$ & $(68.33)$ & $(75)$ & $(71.67)$ & $(71.67)$ & $(71.67)$ & $72.35$ & $71.5$ & $(73.33)$ & $(83.33)$ \\
 & $(2)$ & $(0.00001, 0.00001, 0.00001)$ & $(0.0001, 163)$ & $(0.00001, 103)$ & $(0.001, 143)$ & $(0.00001, 103)$ & $(0.25, 0.25, 0.25, 2)$ & $(1,0.0001,30)$ & $(10, 0.1, 0.0001, 0.000001)$ & $(0.01, 10000, 0.00001, 83)$ & $(0.001, 1000, 0.00001, 63)$ \\
122000 & $78.33$ & $(55)$ & $(70)$ & $(76.67)$ & $(70)$ & $(71.67)$ & $(61.67)$ & $77.85$ & $75$ & $(80)$ & $(78.33)$ \\
 & $(0.0625)$ & $(0.00001, 0.00001, 0.00001)$ & $(0.01, 163)$ & $(100, 23)$ & $(0.01, 163)$ & $(0.00001, 123)$ & $(0.25, 0.25, 0.25, 0.25)$ & $(1,1,10)$ & $(0.01, 0.01, 0.001, 0.00001)$ & $(0.0001, 100, 0.0001, 163)$ & $(0.01, 10000, 0.00001, 83)$ \\
13000 & $85$ & $(50)$ & $(83.33)$ & $(91.67)$ & $(85)$ & $(90)$ & $(78.33)$ & $86.95$ & $86$ & $(85)$ & $(86.67)$ \\
 & $(32)$ & $(0.00001, 0.00001, 0.00001)$ & $(0.001, 183)$ & $(0.00001, 103)$ & $(0.001, 183)$ & $(0.00001, 43)$ & $(0.25, 0.25, 0.25, 2)$ & $(1,1,30)$ & $(2.5, 5, 0.00001, 0.0001)$ & $(0.1, 1000, 0.00001, 3)$ & $(0.001, 100, 0.0001, 163)$ \\
130000 & $58.96$ & $(48.33)$ & $(58.33)$ & $(65)$ & $(58.33)$ & $(61.67)$ & $(60)$ & $60.51$ & $65.5$ & $(61.67)$ & $(66.67)$ \\
 & $(0.03125)$ & $(0.00001, 0.00001, 0.00001)$ & $(0.001, 103)$ & $(100, 23)$ & $(0.001, 103)$ & $(0.00001, 23)$ & $(0.25, 0.25, 0.25, 4)$ & $(1,10,10)$ & $(5, 0.01, 0.0001, 0.001)$ & $(0.00001, 1000, 0.0001, 103)$ & $(0.01, 100000, 0.0001, 163)$ \\
131000 & $78.33$ & $(63.33)$ & $(73.33)$ & $(76.67)$ & $(83.33)$ & $(76.67)$ & $(81.67)$ & $75.03$ & $77$ & $(80)$ & $(78.33)$ \\
 & $(0.03125)$ & $(0.00001, 0.00001, 0.00001)$ & $(0.001, 163)$ & $(0.00001, 83)$ & $(0.01, 23)$ & $(0.00001, 183)$ & $(0.25, 0.25, 0.25, 4)$ & $(1,10,30)$ & $(0.1, 2.5, 0.000001, 0.00001)$ & $(0.001, 1000, 0.00001, 123)$ & $(10, 100000, 0.1, 203)$ \\
140000 & $71.67$ & $(40)$ & $(58.33)$ & $(76.67)$ & $(56.67)$ & $(78.33)$ & $(45)$ & $66.44$ & $65.5$ & $(78.33)$ & $(66.67)$ \\
 & $(0.0625)$ & $(0.00001, 0.00001, 0.00001)$ & $(0.1, 183)$ & $(0.00001, 143)$ & $(0.1, 3)$ & $(100, 3)$ & $(0.25, 0.25, 0.25, 4)$ & $(1,1000,10)$ & $(1, 1, 0.001, 0.0001)$ & $(1000, 0.001, 10, 3)$ & $(10000, 100, 0.0001, 3)$ \\
142000 & $86.67$ & $(51.67)$ & $(85)$ & $(68.33)$ & $(91.67)$ & $(88.33)$ & $(78.33)$ & $90.44$ & $89.5$ & $(91.67)$ & $(90)$ \\
 & $(0.25)$ & $(0.00001, 0.00001, 0.00001)$ & $(0.001, 183)$ & $(10000, 23)$ & $(0.1, 3)$ & $(1, 3)$ & $(0.25, 0.25, 0.25, 4)$ & $(1,1000,30)$ & $(10, 2.5, 0.00001, 0.001)$ & $(0.01, 0.0001, 0.001, 3)$ & $(0.1, 10, 0.0001, 3)$ \\
143000 & $66.67$ & $(51.67)$ & $(63.33)$ & $(65)$ & $(63.33)$ & $(65)$ & $(71.67)$ & $59.7$ & $57$ & $(60)$ & $(60)$ \\
 & $(0.03125)$ & $(0.00001, 0.00001, 0.00001)$ & $(0.001, 163)$ & $(0.00001, 203)$ & $(0.001, 163)$ & $(0.00001, 203)$ & $(0.25, 0.25, 0.25, 0.25)$ & $(1,10000,10)$ & $(0.01, 0.1, 0.0001, 0.000001)$ & $(0.001, 0.1, 0.0001, 203)$ & $(10, 0.001, 0.001, 3)$ \\
144000 & $73.33$ & $(48.33)$ & $(68.33)$ & $(71.67)$ & $(68.33)$ & $(70)$ & $(66.67)$ & $72.5$ & $72.5$ & $(76.67)$ & $(73.33)$ \\
 & $(16)$ & $(0.00001, 0.00001, 0.00001)$ & $(0.01, 63)$ & $(0.00001, 203)$ & $(0.01, 63)$ & $(0.00001, 203)$ & $(0.25, 0.25, 0.25, 0.25)$ & $(1,10000,30)$ & $(2.5, 10, 0.001, 0.00001)$ & $(0.001, 1000, 0.00001, 143)$ & $(0.001, 1000, 0.00001, 43)$ \\
147000 & $58.33$ & $(55)$ & $(63.33)$ & $(73.33)$ & $(63.33)$ & $(73.33)$ & $(70)$ & $66.84$ & $64.5$ & $(71.67)$ & $(66.67)$ \\
 & $(0.125)$ & $(0.00001, 0.00001, 0.00001)$ & $(0.01, 203)$ & $(0.00001, 123)$ & $(0.01, 203)$ & $(0.00001, 123)$ & $(0.25, 0.25, 0.25, 4)$ & $(10,0.0001,10)$ & $(5, 5, 0.00001, 0.0001)$ & $(0.1, 1, 0.00001, 203)$ & $(0.01, 100000, 0.00001, 103)$ \\
148000 & $86.67$ & $(50)$ & $(90)$ & $(80)$ & $(88.33)$ & $(86.67)$ & $(83.33)$ & $81.83$ & $85.5$ & $(88.33)$ & $(83.33)$ \\
 & $(2)$ & $(0.00001, 0.00001, 0.00001)$ & $(0.001, 103)$ & $(0.00001, 83)$ & $(0.001, 23)$ & $(0.00001, 63)$ & $(0.25, 0.25, 0.25, 4)$ & $(10,0.0001,30)$ & $(0.1, 0.01, 0.0001, 0.001)$ & $(0.001, 100, 0.00001, 123)$ & $(0.001, 1000, 0.00001, 83)$ \\
152000 & $56.67$ & $(48.33)$ & $(65)$ & $(66.67)$ & $(66.67)$ & $(45)$ & $(51.67)$ & $65.37$ & $63$ & $(66.67)$ & $(66.67)$ \\
 & $(0.25)$ & $(0.00001, 0.00001, 0.00001)$ & $(0.0001, 143)$ & $(0.1, 23)$ & $(0.00001, 103)$ & $(1, 103)$ & $(0.25, 0.25, 0.25, 0.25)$ & $(10,1,10)$ & $(1, 2.5, 0.000001, 0.00001)$ & $(0.00001, 1000, 0.00001, 163)$ & $(0.001, 100000, 0.0001, 3)$ \\
153000 & $79.33$ & $(55)$ & $(76.67)$ & $(73.33)$ & $(76.67)$ & $(68.33)$ & $(80)$ & $79.58$ & $77$ & $(81.67)$ & $(80)$ \\
 & $(0.5)$ & $(0.00001, 0.00001, 0.00001)$ & $(0.001, 203)$ & $(0.00001, 43)$ & $(0.001, 203)$ & $(0.00001, 43)$ & $(0.25, 0.25, 0.25, 4)$ & $(10,1,30)$ & $(10, 0.01, 0.001, 0.0001)$ & $(0.01, 0.00001, 0.001, 83)$ & $(0.01, 100, 0.001, 103)$ \\
161000 & $86.67$ & $(46.67)$ & $(91.67)$ & $(91.67)$ & $(88.33)$ & $(90)$ & $(93.33)$ & $90.91$ & $91$ & $(93.33)$ & $(91.67)$ \\
 & $(0.5)$ & $(0.00001, 0.00001, 0.00001)$ & $(0.001, 183)$ & $(0.00001, 143)$ & $(0.1, 23)$ & $(0.00001, 43)$ & $(0.25, 0.25, 0.25, 2)$ & $(10,10,30)$ & $(0.01, 1, 0.00001, 0.001)$ & $(0.1, 0.01, 0.00001, 23)$ & $(0.01, 1000, 0.001, 143)$ \\
163000 & $80$ & $(41.67)$ & $(80)$ & $(78.33)$ & $(80)$ & $(85)$ & $(71.67)$ & $89.41$ & $86$ & $(90)$ & $(90)$ \\
 & $(4)$ & $(0.00001, 0.00001, 0.00001)$ & $(0.0001, 203)$ & $(1, 23)$ & $(0.001, 83)$ & $(0.00001, 3)$ & $(0.25, 0.25, 0.25, 0.5)$ & $(1, 1, 10)$ & $(2.5, 0.1, 0.0001, 0.000001)$ & $(0.01, 1000, 0.00001, 43)$ & $(0.1, 100000, 0.00001, 3)$ \\
17000 & $83.33$ & $(56.67)$ & $(86.67)$ & $(90)$ & $(86.67)$ & $(86.67)$ & $(83.33)$ & $82.36$ & $82.5$ & $(91.67)$ & $(83.33)$ \\
 & $(2)$ & $(0.00001, 0.00001, 0.00001)$ & $(0.001, 163)$ & $(0.00001, 143)$ & $(0.001, 163)$ & $(0.00001, 203)$ & $(0.25, 0.25, 0.25, 4)$ & $(1, 1000, 100)$ & $(5, 1, 0.001, 0.00001)$ & $(0.01, 1000, 0.001, 183)$ & $(0.01, 10000, 0.00001, 63)$ \\
171000 & $88.33$ & $(48.33)$ & $(76.67)$ & $(68.33)$ & $(76.67)$ & $(75)$ & $(63.33)$ & $69.63$ & $68.47$ & $(80)$ & $(70)$ \\
 & $(0.03125)$ & $(0.00001, 0.00001, 0.00001)$ & $(0.001, 123)$ & $(0.00001, 83)$ & $(0.001, 123)$ & $(0.00001, 3)$ & $(0.25, 0.25, 0.25, 0.25)$ & $(0.0001, 0.0001, 10)$ & $(0.1, 5, 0.00001, 0.0001)$ & $(0.0001, 100, 0.0001, 103)$ & $(0.001, 100, 0.0001, 103)$ \\
173000 & $85$ & $(61.67)$ & $(86.67)$ & $(71.67)$ & $(83.33)$ & $(78.33)$ & $(80)$ & $79.61$ & $78.5$ & $(86.67)$ & $(80)$ \\
 & $(0.125)$ & $(0.00001, 0.00001, 0.00001)$ & $(0.001, 203)$ & $(1000, 23)$ & $(0.001, 63)$ & $(1000, 3)$ & $(0.25, 0.25, 0.25, 0.25)$ & $(10,10,10)$ & $(1, 10, 0.0001, 0.001)$ & $(0.001, 10000, 0.0001, 183)$ & $(0.001, 10000, 0.00001, 123)$ \\
174000 & $81.67$ & $(45)$ & $(81.67)$ & $(85)$ & $(83.33)$ & $(88.33)$ & $(88.33)$ & $82.15$ & $80.5$ & $(93.33)$ & $(83.33)$ \\
 & $(0.03125)$ & $(0.00001, 0.00001, 0.00001)$ & $(0.01, 103)$ & $(0.001, 43)$ & $(0.01, 103)$ & $(0.00001, 23)$ & $(0.25, 0.25, 0.25, 4)$ & $(1, 10000, 10)$ & $(10, 5, 0.000001, 0.00001)$ & $(1, 100000, 0.001, 23)$ & $(0.1, 10000, 0.00001, 203)$ \\
182000 & $76.67$ & $(46.67)$ & $(78.33)$ & $(76.67)$ & $(80)$ & $(76.67)$ & $(70)$ & $74.15$ & $75$ & $(81.67)$ & $(75)$ \\
 & $(0.03125)$ & $(0.00001, 0.00001, 0.00001)$ & $(0.001, 183)$ & $(0.00001, 123)$ & $(0.001, 183)$ & $(0.00001, 123)$ & $(0.25, 0.25, 0.25, 0.25)$ & $(1, 0.0001, 10)$ & $(0.01, 2.5, 0.001, 0.0001)$ & $(0.01, 0.00001, 0.001, 3)$ & $(0.01, 100, 0.00001, 3)$ \\
183000 & $70$ & $(48.33)$ & $(80)$ & $(68.33)$ & $(80)$ & $(78.33)$ & $(68.33)$ & $77.63$ & $77$ & $(80)$ & $(78.33)$ \\
 & $(4)$ & $(0.00001, 0.00001, 0.00001)$ & $(0.0001, 203)$ & $(1000, 43)$ & $(0.0001, 203)$ & $(0.00001, 43)$ & $(0.25, 0.25, 0.25, 4)$ & $(0.0001, 1000, 10)$ & $(2.5, 0.01, 0.00001, 0.001)$ & $(0.01, 0.01, 0.001, 23)$ & $(0.001, 100, 0.00001, 143)$ \\
187000 & $83.33$ & $(43.33)$ & $(81.67)$ & $(80)$ & $(83.33)$ & $(85)$ & $(81.67)$ & $75.65$ & $74.56$ & $(78.33)$ & $(81.67)$ \\
 & $(0.03125)$ & $(0.00001, 0.00001, 0.00001)$ & $(0.0001, 103)$ & $(10000, 23)$ & $(0.001, 103)$ & $(0.00001, 3)$ & $(0.25, 0.25, 0.25, 0.25)$ & $(0.0001, 1, 50)$ & $(5, 0.1, 0.0001, 0.000001)$ & $(0.001, 100, 0.001, 123)$ & $(0.001, 1000, 0.0001, 123)$ \\
189000 & $76.67$ & $(58.33)$ & $(80)$ & $(81.67)$ & $(80)$ & $(75)$ & $(73.33)$ & $79.29$ & $82.5$ & $(81.67)$ & $(83.33)$ \\
 & $(0.25)$ & $(0.00001, 0.00001, 0.00001)$ & $(0.001, 203)$ & $(0.00001, 123)$ & $(0.001, 203)$ & $(0.00001, 163)$ & $(0.25, 0.25, 0.25, 4)$ & $(0.0001, 0.0001, 50)$ & $(0.1, 1, 0.001, 0.00001)$ & $(0.1, 0.01, 0.01, 23)$ & $(0.01, 100000, 0.00001, 183)$ \\
20000 & $65$ & $(40)$ & $(63.33)$ & $(73.33)$ & $(70)$ & $(63.33)$ & $(65)$ & $65.53$ & $63.5$ & $(70)$ & $(68.33)$ \\
 & $(0.03125)$ & $(0.00001, 0.00001, 0.00001)$ & $(0.001, 183)$ & $(100000, 23)$ & $(0.01, 3)$ & $(0.00001, 163)$ & $(0.25, 0.25, 0.25, 4)$ & $(0.0001, 10000, 30)$ & $(1, 0.1, 0.00001, 0.0001)$ & $(0.001, 10000, 0.00001, 103)$ & $(0.001, 100000, 0.0001, 183)$ \\
201000 & $86.67$ & $(58.33)$ & $(90)$ & $(88.33)$ & $(91.67)$ & $(86.67)$ & $(80)$ & $82.33$ & $88.5$ & $(88.33)$ & $(83.33)$ \\
 & $(1)$ & $(0.00001, 0.00001, 0.00001)$ & $(1, 23)$ & $(0.00001, 63)$ & $(0.001, 63)$ & $(0.00001, 23)$ & $(0.25, 0.25, 0.25, 4)$ & $(1, 10000, 50)$ & $(10, 1, 0.0001, 0.001)$ & $(0.01, 100, 0.01, 3)$ & $(0.01, 10000, 0.0001, 63)$ \\
21000 & $90$ & $(50)$ & $(93.33)$ & $(83.33)$ & $(88.33)$ & $(86.67)$ & $(86.67)$ & $85.92$ & $81$ & $(86.67)$ & $(90)$ \\
 & $(1)$ & $(0.00001, 0.00001, 0.00001)$ & $(0.001, 183)$ & $(0.00001, 43)$ & $(0.01, 163)$ & $(0.00001, 103)$ & $(0.25, 0.25, 0.25, 0.25)$ & $(1, 10, 30)$ & $(0.01, 10, 0.000001, 0.00001)$ & $(0.1, 10000, 0.00001, 203)$ & $(0.1, 100000, 0.01, 23)$ \\
22000 & $70$ & $(46.67)$ & $(70)$ & $(53.33)$ & $(73.33)$ & $(68.33)$ & $(68.33)$ & $72.64$ & $73$ & $(75)$ & $(75)$ \\
 & $(0.03125)$ & $(0.00001, 0.00001, 0.00001)$ & $(0.001, 143)$ & $(1, 23)$ & $(0.1, 23)$ & $(0.00001, 43)$ & $(0.25, 0.25, 0.25, 0.25)$ & $(0.0001, 10, 30)$ & $(2.5, 2.5, 0.001, 0.0001)$ & $(1, 1, 0.001, 3)$ & $(0.1, 0.0001, 0.0001, 3)$ \\
231000 & $65$ & $(55)$ & $(61.67)$ & $(60)$ & $(56.67)$ & $(60)$ & $(53.33)$ & $59.23$ & $58.47$ & $(60)$ & $(70)$ \\
 & $(0.0625)$ & $(0.00001, 0.00001, 0.00001)$ & $(0.1, 123)$ & $(0.00001, 203)$ & $(0.01, 83)$ & $(0.00001, 203)$ & $(0.25, 0.25, 0.25, 0.25)$ & $(1, 0.0001, 50)$ & $(5, 2.5, 0.00001, 0.001)$ & $(0.001, 10000, 1, 203)$ & $(1, 0.1, 0.00001, 23)$ \\
276000 & $78.67$ & $(56.67)$ & $(78.33)$ & $(76.67)$ & $(80)$ & $(76.67)$ & $(71.67)$ & $75.92$ & $77.5$ & $(80)$ & $(78.33)$ \\
 & $(0.25)$ & $(0.00001, 0.00001, 0.00001)$ & $(0.01, 63)$ & $(0.00001, 183)$ & $(0.01, 63)$ & $(0.00001, 183)$ & $(0.5, 0.5, 0.25, 4)$ & $(0.0001, 1, 10)$ & $(0.1, 0.01, 0.000001, 0.0001)$ & $(0.1, 100, 0.001, 43)$ & $(1, 100000, 0.01, 3)$ \\
296000 & $78.33$ & $(46.67)$ & $(85)$ & $(71.67)$ & $(85)$ & $(78.33)$ & $(68.33)$ & $65.76$ & $67.5$ & $(66.67)$ & $(75)$ \\
 & $(16)$ & $(0.00001, 0.00001, 0.00001)$ & $(0.001, 183)$ & $(0.00001, 103)$ & $(0.001, 183)$ & $(0.00001, 3)$ & $(0.25, 0.25, 0.25, 4)$ & $(1000, 0.0001, 10)$ & $(1, 5, 0.0001, 0.000001)$ & $(10, 100000, 10000, 183)$ & $(0.0001, 10000, 0.0001, 123)$ \\
33000 & $83.33$ & $(43.33)$ & $(86.67)$ & $(81.67)$ & $(88.33)$ & $(81.67)$ & $(68.33)$ & $84.61$ & $89.5$ & $(88.33)$ & $(85)$ \\
 & $(0.03125)$ & $(0.00001, 0.00001, 0.00001)$ & $(0.001, 143)$ & $(0.00001, 43)$ & $(0.01, 3)$ & $(0.00001, 43)$ & $(0.25, 0.25, 0.25, 0.25)$ & $(0.0001, 10000, 30)$ & $(10, 0.1, 0.001, 0.00001)$ & $(0.1, 0.0001, 0.0001, 3)$ & $(0.001, 1000, 0.00001, 23)$ \\
335000 & $78.33$ & $(43.33)$ & $(75)$ & $(78.33)$ & $(73.33)$ & $(75)$ & $(70)$ & $67.58$ & $67.25$ & $(68.33)$ & $(73.33)$ \\
 & $(8)$ & $(0.00001, 0.00001, 0.00001)$ & $(0.001, 203)$ & $(0.00001, 43)$ & $(0.01, 3)$ & $(0.00001, 43)$ & $(0.25, 0.25, 0.25, 4)$ & $(0.0001, 1000, 50)$ & $(0.01, 0.01, 0.00001, 0.001)$ & $(0.01, 10000, 0.00001, 203)$ & $(0.01, 10000, 0.001, 83)$ \\
34000 & $80.33$ & $(53.33)$ & $(76.67)$ & $(86.67)$ & $(83.33)$ & $(86.67)$ & $(73.33)$ & $79.42$ & $78.5$ & $(80)$ & $(81.67)$ \\
 & $(0.25)$ & $(0.00001, 0.00001, 0.00001)$ & $(0.001, 83)$ & $(0.00001, 163)$ & $(0.001, 203)$ & $(0.00001, 163)$ & $(0.25, 0.25, 0.25, 0.5)$ & $(0.0001, 0.0001, 100)$ & $(2.5, 5, 0.0001, 0.000001)$ & $(0.1, 100000, 0.0001, 3)$ & $(10000, 1, 1000, 3)$ \\
384000 & $85$ & $(55)$ & $(78.33)$ & $(86.67)$ & $(78.33)$ & $(85)$ & $(78.33)$ & $75.47$ & $73.5$ & $(86.67)$ & $(76.67)$ \\
 & $(8)$ & $(0.00001, 0.00001, 0.00001)$ & $(0.001, 83)$ & $(1000, 63)$ & $(0.001, 83)$ & $(0.00001, 183)$ & $(0.25, 0.25, 0.25, 2)$ & $(0.0001, 0.0001, 10)$ & $(5, 10, 0.001, 0.0001)$ & $(0.001, 100, 0.00001, 183)$ & $(0.1, 10000, 0.01, 143)$ \\
41000 & $65$ & $(43.33)$ & $(66.67)$ & $(68.33)$ & $(56.67)$ & $(65)$ & $(56.67)$ & $62.36$ & $65.83$ & $(75.67)$ & $(63.33)$ \\
 & $(0.125)$ & $(0.00001, 0.00001, 0.00001)$ & $(0.1, 43)$ & $(0.00001, 183)$ & $(0.1, 43)$ & $(1000, 23)$ & $(0.25, 0.25, 0.25, 4)$ & $(1, 1000, 100)$ & $(0.1, 2.5, 0.00001, 0.001)$ & $(0.01, 0.0001, 1, 183)$ & $(0.01, 1000, 0.001, 203)$ \\
46000 & $70$ & $(46.67)$ & $(78.33)$ & $(81.67)$ & $(80)$ & $(83.33)$ & $(65)$ & $72.45$ & $64.5$ & $(75)$ & $(78.33)$ \\
 & $(0.125)$ & $(0.00001, 0.00001, 0.00001)$ & $(0.001, 163)$ & $(0.00001, 43)$ & $(0.001, 163)$ & $(0.00001, 43)$ & $(0.25, 0.25, 0.25, 4)$ & $(1, 1, 100)$ & $(1, 0.01, 0.0001, 0.000001)$ & $(0.1, 1000, 0.0001, 83)$ & $(0.01, 1000, 0.00001, 83)$ \\ \hline
Average Acc. & $74.87$ & $49.93$ & $74.98$ & $74.83$ & $76.33$ & $75.43$ & $69.87$ & $73.98$ & $73.47$ & $78.01$ & $75.63$ \\ \hline
Average Rank & $5.57$ & $10.93$ & $5.4$ & $5.44$ & $4.47$ & $5.48$ & $7.55$ & $6.49$ & $6.82$ & $3.17$ & $4.68$ \\ \hline
\multicolumn{10}{l}{$^{\dagger}$ represents the proposed models.}
\end{tabular}}
\end{table*}

\end{document}